%% file: main.tex
\newif\ifarxiv
\arxivtrue 
\ifdefined\arxivbuild\arxivtrue\fi

\ifarxiv
  \PassOptionsToClass{nonacm}{aamas}      
\else
  \PassOptionsToClass{anonymous}{aamas}
\fi
\documentclass[sigconf]{aamas}

\usepackage{bbding}
\usepackage{balance}
\usepackage{adjustbox}
\usepackage{enumitem}
\usepackage{multirow}
\usepackage{tikz}
\usetikzlibrary{positioning,arrows.meta,calc,fit,backgrounds}

\makeatletter
\gdef\@copyrightpermission{
  \begin{minipage}{0.2\columnwidth}
   \href{https://creativecommons.org/licenses/by/4.0/}{\includegraphics[width=0.90\textwidth]{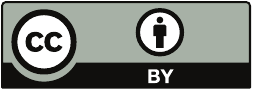}}
  \end{minipage}\hfill
  \begin{minipage}{0.8\columnwidth}
   \href{https://creativecommons.org/licenses/by/4.0/}{This work is licensed under a Creative Commons Attribution International 4.0 License.}
  \end{minipage}
  \vspace{5pt}
}
\makeatother
\ifarxiv
  \setcopyright{none}
  \makeatletter
  \AtBeginDocument{%
    \fancypagestyle{standardpagestyle}{%
      \fancyhf{}\renewcommand{\headrulewidth}{\z@}\renewcommand{\footrulewidth}{\z@}%
      \fancyhead[LO]{\@headfootfont\shorttitle}%
      \fancyhead[RE]{\@headfootfont\@shortauthors}%
      \fancyfoot[C]{\footnotesize\thepage}}%
    \fancypagestyle{firstpagestyle}{%
      \fancyhf{}\renewcommand{\headrulewidth}{\z@}\renewcommand{\footrulewidth}{\z@}%
      \fancyfoot[C]{\footnotesize\thepage}}%
    \pagestyle{standardpagestyle}}
  \makeatother
\else
  \setcopyright{ifaamas}
\fi
\acmConference[AAMAS 2027]{Proc.\@ of the 26th International Conference
on Autonomous Agents and Multiagent Systems (AAMAS 2027)}{3--7 May 2027}{Hanoi, Vietnam}{M.~Baldoni, F.~Fang, W.~Yeoh, N.~Yorke-Smith (eds.)}
\copyrightyear{2027}
\acmYear{2027}
\acmDOI{}
\acmISBN{}
\acmSubmissionID{1536}
\submissionType{Research Paper Track}

\newcommand{\cmark}{\textcolor{teal}{\ding{51}}}
\newcommand{\xmark}{\textcolor{gray!60}{\ding{55}}}
\newcommand{\pmark}{\textcolor{orange}{\textbf{$\bullet$}}}
\newcommand{\pending}[1]{\textcolor{red}{[#1]}}
\newcommand{\Pa}{\mathrm{Pa}}
\newcommand{\dox}{\mathrm{do}}
\newcommand{\nbn}{\texttt{NBN}}
\newcommand{\lib}[1]{\texttt{#1}}
\newcommand{\figfile}[3][\textwidth]{\IfFileExists{#2}{\includegraphics[width=#1]{#2}}{\fbox{\parbox[c][#3][c]{0.97\textwidth}{\centering\pending{figure slot: \texttt{\detokenize{#2}}}}}}}

\title[Neural Bayesian Networks]{Neural Bayesian Networks \\for Modular and Scalable Decision-Making Under Uncertainty}

\ifarxiv
\author{Giovanni Briglia\textsuperscript{\Envelope}}
\affiliation{\institution{University of Modena and Reggio Emilia}\city{Reggio Emilia}\country{Italy}}
\email{giovanni.briglia@unimore.it}

\author{Marco Lippi}
\affiliation{\institution{University of Florence}\city{Florence}\country{Italy}}
\email{marco.lippi@unifi.it}

\author{Stefano Mariani}
\affiliation{\institution{University of Modena and Reggio Emilia}\city{Reggio Emilia}\country{Italy}}
\email{stefano.mariani@unimore.it}

\author{Franco Zambonelli}
\affiliation{\institution{University of Modena and Reggio Emilia}\city{Reggio Emilia}\country{Italy}}
\email{franco.zambonelli@unimore.it}
\else
\author{Anonymous Author(s)}
\affiliation{\institution{Anonymous Institution}\city{}\country{}}
\email{}
\fi

\begin{abstract}
    Bayesian Networks (BNs) are established tools for letting agents reason under uncertainty in sequential decision-making.
    They split a joint distribution into local conditional mechanisms arranged on a directed acyclic graph, which makes them modular, interpretable, and ready to perform probabilistic inference.
    Their classical formulation, however, has limitations in both expressiveness and scale: each conditional is usually stored as a table or a linear-Gaussian form, and inference evaluates the graph one node and one query at a time.
    In this paper we introduce \emph{Neural Bayesian Networks} (NBNs), in which each conditional is a differentiable, tensor-valued mechanism, while the graph and its factorization are unchanged. 
    NBNs retain the modularity and local mechanisms of BNs, can represent any learnable conditional distribution, and run natively on GPUs.
    We also release \nbn, a PyTorch-based library for learning and inference in NBNs, together with a public benchmark.
    Our evaluation shows gains in both expressiveness and scale: NBNs fit distributions that classical BN libraries struggle to represent, and on large networks and large batches of queries they run two orders of magnitude faster thanks to GPU parallelization. 
    \ifarxiv \nbn{} is available at \url{https://github.com/Giovannibriglia/NeuralBayesianNetworks}.\fi
\end{abstract}


\begin{document}
\maketitle




\section{Introduction}
\label{sec:intro}

An agent that acts under uncertainty would benefit from a model of how its environment responds to its actions to reason about probabilistic outcomes reliably. 
Model-based reinforcement learning, planning, and diagnosis all can leverage a transition model that can be learned from data and queried under evidence \cite{kaelbling1998planning,ha2018world,levine2020offlinerl}. 
If the dynamics of the environment are known, at least in the sense of which variables influence which, the decision-making problem can be written as a factored Markov decision process (F-MDP) \cite{boutilier1995exploiting,kearns1999factored,guestrin2003efficient}, whose transition model is a Bayesian network (BN) over the state variables. 
The BN is then the tool that lets the agent reason under uncertainty: it can predict the next state, and it can explain an observation in terms of the state that produced it. 
This matters most when the state is high-dimensional, because the factored model stays compact where a flat one does not. 

Classical BNs, however, were mostly designed for discrete variables and small graphs \cite{pearl1988probabilistic,koller2009probabilistic}. A conditional probability table grows combinatorially with the parents and is undefined when a parent is continuous. 
Additionally, both exact and approximate inference evaluate the conditionals one node and one sample at a time. These limitations stem from representing each conditional as a scalar function and from evaluating the graph sequentially. 
They are thus inherited by mainstream BN software packages (Table~\ref{tab:bn_comparison}).

We propose \emph{Neural Bayesian Network} (NBN) as a new tool for autonomous agents when dealing with decision-making under uncertainty. 
NBNs keep the structural properties of BNs (i.e., the graph) while changing the conditionals to leverage more expressive and scalable representations based on neural networks.
In an NBN each local conditional $p(x_i \mid x_{\Pa(i)})$ is a differentiable, tensor-valued \emph{mechanism} that takes a batch of parent configurations and returns the full conditional distribution for each row (Figure~\ref{fig:overview}). The graph and the factorization are exactly those of a BN. What changes is their computational representation: the mechanisms can be tabular, parametric, neural or non-parametric; every one of them exposes density, sampling, and gradients through one unified interface; and a whole layer of the graph is evaluated in one tensor operation on a GPU, for thousands of samples and evidence rows at once.

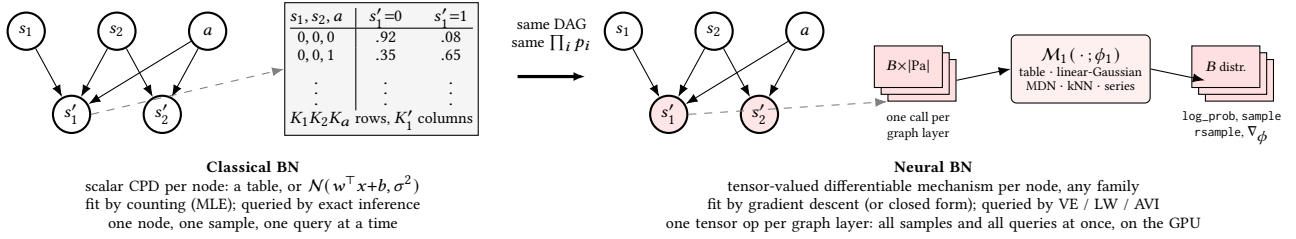
\begin{figure*}[t]
\centering
\input{figures/fig_overview.tex}
\caption{From a classical BN to an NBN. Left: each node stores a scalar conditional, evaluated one node and one configuration at a time. Right: each node stores a tensor-valued differentiable mechanism $\mathcal{M}_i$ that maps a batch of $B$ parent configurations to a batch of conditional distributions. The DAG and the factorization are unchanged.}
\label{fig:overview}
\end{figure*}

This advancement enables BNs to be used efficiently where they were not practical before: on continuous and non-Gaussian variables without reasonable discretization available, on networks of thousands of nodes, and for thousands of queries answered in one call. 
As such, NBNs make BNs seamless to integrate with learning-based agent architectures (e.g., reinforcement learning and cognitive architectures) and, more generally, with any machine learning toolchain: the NBN model is a differentiable module that is trained by gradient descent and runs on the same computational accelerators as a neural network (e.g., a GPU), while keeping the explicit graph, modular representation, and local mechanisms factorisation of the BNs' symbolic model.

Our contributions are:
\begin{enumerate}[leftmargin=*,nosep]
  \item \textbf{Representation.} We define NBNs (Section~\ref{sec:nbn}) and show that they are semantically equivalent to classical BNs, end-to-end differentiable, and whose inference scales with the number of active nodes instead of samples or queries.
  \item \textbf{Library.} We release \nbn, a PyTorch implementation with tabular, linear-Gaussian, neural and non-parametric mechanisms and with exact, sampling-based and amortized inference engines, all batched on the GPU.
  \item \textbf{Benchmark.} We release a benchmark for BN learning and inference and compare \nbn{} with state-of-the-art BN frameworks on discrete, Gaussian, and non-Gaussian networks of different dimensions (Section~\ref{sec:exp}). 
\end{enumerate}
Although NBNs do not yet win all comparisons, the benchmark confirms their potential with regards to scalability and expressiveness, while also pointing to further assessment and development.

The paper is organised as follows. Section~\ref{sec:nbn} introduces NBNs, contrasts them with BNs, and define their properties. Section~\ref{sec:exp} presents the design of the benchmark and Section~\ref{sec:results} its results. Section~\ref{sec:related} reviews related works, and Section~\ref{sec:conc} concludes.

\section{Neural Bayesian Networks}
\label{sec:nbn}


A Neural Bayesian Network (NBN) is a BN whose conditional distributions are learnable modules. The graph acts as the architecture of the model: each node hosts one module, its \emph{mechanism}, which receives the values of the node's parents and returns the distribution of the node (Figure~\ref{fig:overview}). Learning fits the mechanisms to data and inference runs them along the graph, both on batches, as in a neural network.

Similarly to a BN, an NBN over variables $X = (X_1,\dots,X_d)$ is a directed acyclic graph (DAG) $\mathcal{G} = (V,E)$ with $V = \{1,\dots,d\}$ and one conditional distribution $p_i(x_i \mid x_{\Pa(i)})$ per node, where $\Pa(i)$ are the parents of $i$ in $\mathcal{G}$ \cite{pearl1988probabilistic,koller2009probabilistic}. 
The joint distribution is the product of the conditionals,
\begin{equation}
\label{eq:bn-factorization}
p(x_1,\dots,x_d) = \prod_{i=1}^{d} p_i\!\left(x_i \mid x_{\Pa(i)}\right),
\end{equation}
so each variable is independent of its non-descendants given its parents.
%
Differently from a BN, where a conditional is a conditional probability table (CPT) or a linear-Gaussian form, an NBN admits arbitrary representations of a conditional distribution: tabular entries, parametric and non-parametric expressions, and neural network based approximations, possibly mixed in the same graph.

Additionally, in a BN, \emph{parameter learning} estimates the conditionals one at a time using the log-likelihood $\sum_{n}\sum_i \log p_{i}(x^{(n)}_i \mid x^{(n)}_{\Pa(i)})$ based on counting for a CPT and on the least squares method for a linear-Gaussian form \cite{heckerman1998tutorial}. 
In an NBN, instead, every conditional is differentiable, so the same log-likelihood can also be maximized by gradient descent on batches of data---for all the conditionals at once, and for those that admit no closed-form estimate, too.
Also \emph{inference} proceeds one variable at a time in a BN, by computing the posterior $p(x_Q \mid x_E)$ of query variables $Q$ given observed evidence $x_E$ using variable elimination at a cost exponential in the treewidth of the graph \cite{DBLP:journals/ai/Cooper90}. 
In an NBN, instead, each conditional is evaluated once for a whole batch of samples and evidence assignments, so that many queries are answered in a single pass over the graph, by exact elimination where the treewidth allows it and by sampling elsewhere.
As we will show in the following, BNs are special case of NBNs when the closed-form methods are feasible.

\nbn{} is written in PyTorch \cite{paszke2019pytorch}. The network is an \texttt{nn.Module}. A user builds it from a DAG and a per-node mechanism assignment, calls \texttt{fit} once, and then \texttt{query}. 
The benchmark, its configuration, and the complete results are released with the library.\footnote{\ifarxiv\url{https://github.com/Giovannibriglia/NeuralBayesianNetworks}\else Link withheld for review. Repository provided as supplementary material.\fi}

\subsection{Motivation}
\label{ssec:example}

BNs are a common ``world model'' for agents that decide under uncertainty: they are the transition models of factored Markov decision processes in planning and reinforcement learning~\cite{boutilier1995exploiting,kearns1999factored,guestrin2003efficient}, and the diagnostic models used in troubleshooting, risk analysis, and healthcare \cite{DBLP:conf/uai/BreeseH96,weber2012bn,kyrimi2021bnhealth}. 

Consider a decision-making agent that has to deal with a Markov decision process (MDP) formalised as $(\mathcal{S}, \mathcal{A}, P, R, \gamma)$ with state $s \in \mathcal{S}$, action $a \in \mathcal{A}$ and transition kernel $P(s' \mid s, a)$. 
The MDP treats the state as a single object, so $P$ is one distribution over the joint next state for every pair $(s,a)$, and its size grows exponentially with the number of state variables. 
If we know which current variables each next-state variable depends on, we can move to a factored MDP \cite{boutilier1995exploiting,kearns1999factored}, which is much lighter (Figure~\ref{fig:fmdp}):
\begin{equation}
\label{eq:factored-mdp}
P(s' \mid s, a) \;=\; \prod_{j=1}^{n} p\!\left(s'_j \mid \Pa(s'_j)\right),
\end{equation}
where the parents $\Pa(s'_j)$ of each next-state variable are a few of the current state and action variables. This is a two-slice BN over $X = (s, a, s')$. 
The same formulation applies to agents that face non-sequential scenarios, such as while doing diagnosis, where the BN has a single slice. 

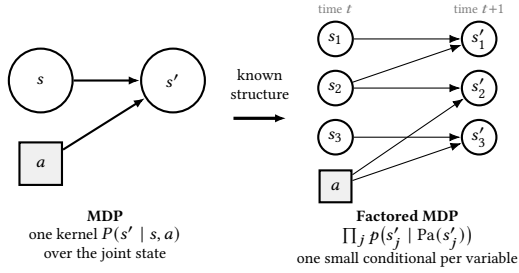
\begin{figure}[t]
\centering
\input{figures/fig_fmdp.tex}
\caption{From an MDP to a factored MDP. When it is known which variables influence which, the transition kernel over the joint state becomes a two-slice BN with one conditional per next-state variable.}
\label{fig:fmdp}
\end{figure}

The agent has to deal with two problems:
\begin{enumerate}[leftmargin=*,nosep]
  \item[\textbf{P1}] \textbf{Learning.} Given $\mathcal{G}$ and a dataset $\mathcal{D}$ of $N$ samples, estimate every conditional of Eq.~\ref{eq:bn-factorization}, for discrete and for continuous variables, \emph{without fixing the shape of the conditional in advance}.
  \item[\textbf{P2}] \textbf{Inference.} Given the learned model, compute $p(x_Q \mid x_E)$ for query variables $Q$ and evidence $x_E$, \emph{for many evidence assignments at once}.
\end{enumerate}
In this work we assume that the structure is given and focus on parameter learning and inference, while leaving structure learning and the sampling of trajectories for planning to future work. 
More precisely, we make the standard assumptions of parameter learning in BNs. (A1)~The graph $\mathcal{G}$ is given and the data distribution factorizes according to it. (A2)~The data are complete, so every variable is observed in every sample, and the samples are independent and identically distributed. (A3)~Each variable is a discrete variable with finitely many values or a real-valued scalar.
All statements below are for a BN over $d$ variables and cover both cases.

Classical BNs solve P1 and P2 only in part, because of three limitations of scalar conditionals evaluated one at a time. 
(L1)~\emph{Representation, which limits P1.} A CPT for $X_i$ has $K_i \prod_{j\in\Pa(i)} K_j$ entries, with $K_j$ the number of values of $X_j$, and is undefined for a continuous parent. Discretizing a parent introduces a binning parameter and loses information, and the linear-Gaussian alternative cannot represent multi-modal or heteroscedastic conditionals. Both fix the shape of the outcome uncertainty before seeing the data: a table needs the continuous outcome discretized, a Gaussian has one mode and a spread that does not depend on the parents. 
(L2)~\emph{Sequential evaluation, which limits P1 and P2.} A likelihood evaluation on one sample visits the $d$ nodes in order, and each visit is a scalar lookup or a scalar density. Nothing is shared across samples or across queries, so $S$ samples for $B$ evidence assignments cost $\Theta(d \cdot S \cdot B)$ scalar operations. 
(L3)~\emph{Exact inference, which limits P2.} Variable elimination on a network of treewidth $w$ costs $O(d\,c^{w})$ with $c$ being the maximum cardinality \cite{DBLP:journals/ai/Cooper90}, and it exists only for discrete and linear-Gaussian networks. Sampling avoids the exponent but inherits (L2).

A solution to P1 therefore needs conditionals that are not restricted to tables and linear-Gaussian forms (L1) and that are fit on batches of data (L2). A solution to P2 needs an evaluation whose cost does not grow with $S \cdot B$ (L2), with exact inference where the treewidth allows it and batched sampling elsewhere (L3). 
As described in Section \ref{subsec:nbn-formulation}, NBNs provide all of these properties.

\subsection{Formal definition and Properties}
\label{subsec:nbn-formulation}

Let $\mathcal{G}$ be a DAG over $X_1,\dots,X_d$ with $X_i \in \mathcal{X}_i$. An NBN replaces each conditional by a parametrised mapping
\begin{equation}
\label{eq:nbn-mechanism}
p_{\phi_i}\!\left(x_i \mid x_{\Pa(i)}\right) = \mathcal{M}_i\!\left(x_{\Pa(i)};\,\phi_i\right),
\end{equation}
where $\mathcal{M}_i$ takes a batch $X_{\Pa(i)} \in \mathbb{R}^{B \times |\Pa(i)|}$ of parent configurations and returns, for every row, a distribution over $\mathcal{X}_i$ with a density $\log p_{\phi_i}$, a sampler, and where it exists a reparametrised sampler. The joint is still Eq.~\ref{eq:bn-factorization}:
\begin{equation}
\label{eq:nbn-joint}
p_\phi(X) = \prod_{i=1}^d p_{\phi_i}\!\left(X_i \mid X_{\Pa(i)}\right).
\end{equation}
Two design choices are worth emphasising here. First, $\mathcal{M}_i$ returns the full conditional, not a sample or a point estimate, so learning and inference reach every mechanism through one interface. Second, $\phi_i$ are tensors on which autograd works, so a table, a linear-Gaussian form and a neural network are instances of the same class and can be mixed in one graph. All nodes at the same depth of $\mathcal{G}$ are evaluated in one batched call.


With this representation, NBNs support the following properties. 

\begin{proposition}[Semantic equivalence]
\label{prop:semantic-equivalence}
Let $p$ be a BN on $\mathcal{G}$ and $p_\phi$ an NBN on $\mathcal{G}$ with $p_{\phi_i}(x_i \mid x_{\Pa(i)}) = p_i(x_i \mid x_{\Pa(i)})$ for all $i$ and all arguments. Then $p_\phi = p$, and $p_\phi$ satisfies the global Markov property of $\mathcal{G}$.
\end{proposition}
\begin{proof}
Both joints are the product in Eq.~\ref{eq:bn-factorization} of the same factors. The Markov property depends only on the factorization \cite{koller2009probabilistic}.
\end{proof}
Every classical BN is an NBN with tabular or linear-Gaussian mechanisms, and everything that follows from the factorization (e.g., conditional independences) holds for NBNs. 

\begin{proposition}[Differentiability]
\label{prop:differentiability}
If every $\log p_{\phi_i}$ is differentiable in $\phi_i$, the log-likelihood $\ell(\phi;x) = \sum_i \log p_{\phi_i}(x_i \mid x_{\Pa(i)})$ is differentiable in $\phi$, $\nabla_{\phi_i}\ell$ depends only on the term of node $i$, and the whole gradient is one backward pass.
\end{proposition}
\begin{proof}
$\ell$ is a sum of terms depending on one $\phi_i$; reverse-mode differentiation of the sum evaluates all terms in a single pass.
\end{proof}
Parameter learning (P1) is therefore gradient descent on $-\ell$, per node since the terms are independent (closed form where the family admits it), or jointly when a downstream differentiable objective couples the nodes.

\begin{proposition}[Cost of inference]
\label{prop:complexity}
Consider a query $p_\phi(x_Q \mid x_E)$ answered by ancestral simulation with $S$ samples and $B$ evidence assignments, and let $A = \mathrm{Anc}(Q \cup E)$ be the active nodes. An NBN issues $O(\mathrm{depth}(\mathcal{G}[A]))$ sequential tensor operations, each of cost $c(\mathcal{M}, B\cdot S)$, so the wall-clock time is $O(\mathrm{depth}(\mathcal{G}[A]) \cdot c(\mathcal{M}, BS))$; a classical implementation issues $\Theta(|A| \cdot B \cdot S)$ scalar operations. For a fixed family, $c(\mathcal{M}, BS)$ is constant in $BS$ until the device saturates.
\end{proposition}
\begin{proof}
Nodes outside $A$ are barren and pruned. Within $A$, all nodes at the same depth have their parents available and are evaluated in one call over the $B\cdot S$ rows; there are $\mathrm{depth}(\mathcal{G}[A])$ such levels. The scalar count is the number of (node, sample, row) triplets.
\end{proof}
The claim is not that NBNs change the dependence on the graph, but that for inference (P2) they remove dependence on $S$ and $B$ up to device saturation, so one or a thousand queries cost the same---on a GPU the saturation point is far, on a CPU it is near.

These gains do not come at the cost of sacrificing the usual BN properties that are behind their adoption: \emph{modularity} is preserved since each mechanism has its own parameters and can be inspected, refit, or replaced without touching the others; the  \emph{explicit structure} stands still (in the DAG), which a designer can read and edit; 
and the \emph{``one model for every query''} workflow is also there, since a fitted NBN answers $p(x_Q \mid x_E)$ for any $Q$ and $E$ just like BNs. 

Other routes to expressive and scalable probabilistic models give up at least one of them:  
see Section \ref{sec:related}. 

\subsection{Learning and Inference Methods}
\label{subsec:nbn-taxonomy}

Learning mechanisms differ in how much of the conditional is fixed in advance. \emph{Parametric} mechanisms fix its form and have a closed-form fit: a probability table for a discrete child, by counting under a Dirichlet prior \cite{heckerman1998tutorial}, and the linear-Gaussian form $\mathcal{N}(Wx{+}b,\sigma^2)$ for a continuous one \cite{lauritzen1989graphical}, by ridge regression \cite{hoerl1970ridge}. The \emph{neural} mechanism fixes the output form but not its dependence on the parents: a Mixture Density Network (MDN) \cite{bishop1994mdn} maps the parents with a multilayer perceptron to the parameters of a Gaussian mixture, fit by gradient descent on $-\ell$. It is still parametric, but linear in the number of parents and multi-modal. \emph{Non-parametric} mechanisms fix no form and grow with the data: KNN is a local mixture over the $k{=}\lceil\sqrt N\rceil$ training rows with the closest parents \cite{stone1977consistent}, and FlexCode \cite{izbicki2017converting} expands the conditional density in a cosine basis, one regression per coefficient.

Three inference engines answer $p_\phi(x_Q \mid x_E)$. \emph{Variable Elimination} (VE) \cite{zhang1994simple,obermeyer2019tensor} is exact and discrete-only: it drops the nodes that cannot affect the query and sums out the others by multiplying their tables in one batched tensor contraction, so $B$ evidence rows cost one pass. Being exponential in the induced width, it is our exact reference, not a scalable method.
\emph{Likelihood Weighting} (LW) \cite{fung1989weighing,shachter1989simulation} simulates $S$ samples with the evidence clamped and averages the query node, each sample weighted by the probability of the evidence; it needs only a sampler and a density per mechanism, so it works with every family. \emph{Amortized Variational Inference} (AVI) \cite{gershman2014amortized,le2017inference} avoids the samples LW wastes when the evidence is unlikely: a recognition network  $q_\psi(x_{-E}\mid x_E)$, trained once per model on the evidence lower bound \cite{kingma2014auto}, proposes samples close to the posterior for every query at constant cost; it falls back to LW when the proposal is worse than LW's own Monte-Carlo noise.


\section{Benchmark Design}
\label{sec:exp}

To the best of our knowledge, a fully reproducible benchmark for BNs has not been provided so far. The \lib{bnlearn} repository \cite{scutari2010bnlearn} offers domains and their structures, and it is the usual source of test networks, but the benchmark configuration is unspecified: which queries to ask, how to obtain the ground truth, which metrics to report, or how to count a method that fails or does not finish is not established. 
In this way, results are hardly comparable thus advancements difficult to identify. Hence, we built one and release it with \nbn. A benchmark is a configuration file that fixes the problems, the queries, the methods and the budgets, and the harness writes one row per problem, seed, method, query and metric, failures included. Its problems come from two sources: (i)~the real domains of \lib{bnlearn}, and (ii)~synthetic domains that we generate.
Synthetic networks are needed to define \textit{whatever} network structure (varying nodes and edges configuration) with \textit{whatever} underlying conditional probabilities (discrete, continuous gaussian and continuous non-gaussian), and evaluate sample efficiency, scalability and accuracy across \textit{whatever} BN.

We use the benchmark on \nbn to answer three research questions:
\begin{enumerate}[label=\textbf{RQ\arabic*}, wide=0pt, leftmargin=*, align=left, itemsep=0pt, parsep=0pt, topsep=0pt]
  \item[\textbf{RQ1.}] Where the classical BN assumptions hold, does \nbn{} match the classical BN libraries?
  \item[\textbf{RQ2.}] Where they do not hold, how much accuracy is gained?
  \item[\textbf{RQ3.}] Does the cost model of Proposition~\ref{prop:complexity} show up as wall-clock time when many queries are answered at once?
\end{enumerate}
The numerical settings that follow are fields of the configuration file. We report here those needed to read the results, and all the others in Appendix~\ref{app:design}.

\paragraph{Problems.}
A problem is a triplet of family, number of nodes $n$, and seed.
The \emph{family} is the type of the conditionals of the true model: \emph{discrete}, where every conditional is a CPT; \emph{linear-Gaussian}; or \emph{non-Gaussian}, where every conditional is a set of components forming a Gaussian mixture (that a linear-Gaussian model cannot represent correctly).
The seed draws a sparse random DAG over the $n$ nodes and then the parameters of every conditional.
From the true model we draw a training set ($20\,480$ rows), a test set that supplies the evidence values of the queries, and a reference set that supplies the ground truth.
Every problem is repeated over 5 seeds.
The real networks are the 13 networks of Section~\ref{subsec:exp-bnlearn}. 
We define a \emph{cell} as the combination of one method on one problem.

\paragraph{Queries.}
A query has two components: a target variable, and an evidence set, which pairs each observed variable with its value. The evidence values are always copied from a row of the test set, so the evidence is a configuration that the true model produces.
Three selectors choose targets and evidence variables, one per purpose:
\begin{enumerate}
    \item \emph{Random} (accuracy): the target and the evidence variables are drawn uniformly for the problem, and the query is asked for the evidence values of $128$ test rows; it measures the typical error over evidence values.
    \item \emph{Topological} (real networks): targets are chosen by their role in the graph, among the most connected nodes (hubs), the nodes whose removal disconnects the graph (cut nodes), and the deep nodes with few descendants (terminals). Each target is asked in the \emph{prediction} direction, with evidence among its ancestors, or in the \emph{diagnosis} direction, with evidence among its descendants.
    \item \emph{Heaviest} (scalability and speed): only the top-ranked hubs, cut nodes, and terminals are queried, each in both directions, once with the evidence observed (\emph{full}) and once with it marginalized (\emph{empty}). 
\end{enumerate}

\paragraph{Reference and metrics.}
For a discrete target the reference posterior $\hat{p}$ is its empirical distribution among the reference samples that agree with the evidence, and the metric is the total variation (TV) between $\hat{p}$ and the answer $q$, $\mathrm{TV}(\hat{p}, q) = \tfrac{1}{2}\sum_{x_Q}\lvert \hat{p}(x_Q) - q(x_Q) \rvert$~\cite{DBLP:books/daglib/0035708}. 
For a continuous target the reference is a sample from $p(x_Q \mid x_E)$ obtained by likelihood weighting on the true model, and the metric is the 1-Wasserstein distance (W1) between the empirical distribution $\hat{\mu}$ of this sample and the empirical distribution $\nu$ of the samples returned by the method, $W_1(\hat{\mu}, \nu) = \inf_{\gamma \in \Pi(\hat{\mu}, \nu)} \mathbb{E}_{(x, x') \sim \gamma}\lvert x - x' \rvert$, where $\Pi(\hat{\mu}, \nu)$ is the set of couplings of $\hat{\mu}$ and $\nu$~\cite{gabriel2019computational}.
In both cases the reference is a finite sample and not the exact posterior, so the measured accuracy has a floor that all methods share. On the discrete networks of the accuracy scenario it is a TV of about $0.06$ at 10 nodes and $0.08$ from 50 nodes on, which are the values exact engines reach in Figure~\ref{fig:inference}. We use a sampled reference because an exact one is not available for every family and size: variable elimination on the true model is exponential in the treewidth, and non-Gaussian conditionals have no closed-form posterior. Computing exact references when tractable is left for future work.
The main text reports this accuracy metric (TV or W1) for the accuracy, scalability, and real-network scenarios, and the time per query for the speed batching scenario, measured as the wall-clock time of one call divided by the number of answered queries. Complete fit and query times are in Appendix~\ref{app:results}.

\paragraph{Budgets and failures.}
Every cell has a budget of 1200\,s for the fit, a separate budget of 1200\,s for its queries, and a memory cap: $80\%$ of RAM ($\sim 100$\,GB on our server, that is a more favourable scenarios for cpu-based libraries e.g. \lib{pgmpy}.), and the 12\,GB of the GPU. It ends with a status: \texttt{ok}, \texttt{timeout}, \texttt{oom} or \texttt{not\_supported}. In the figures a bar is the interquartile mean (IQM) over the seeds a method solved, as \cite{agarwal2021deep} suggests; $f/n$ is printed when $f$ of $n$ seeds failed, and a missing bar indicates the failure status.

\paragraph{Scenarios.}
Each scenario is one configuration. \emph{Accuracy} fixes the query load (128 random queries) and varies the family and the size, $n \in \{10, 50, 100, 500, 1000\}$: on the discrete and the linear-Gaussian family the classical model is correct, which isolates RQ1, and on the non-Gaussian family it is misspecified, which isolates RQ2. \emph{Scalability} pushes the network size to $n \in \{10, 10^2, 10^3, 10^4\}$ on the heaviest queries and records where each method stops. \emph{Batching} fixes the network size ($n{=}50$) and varies only the number $B \in \{1, 128, 512, 1024\}$ of queries answered in one call, which is the variable of Proposition~\ref{prop:complexity} (RQ3). \emph{Real networks} repeat the accuracy scenario on published structures with topological queries. \emph{Learning} scores the fitted conditionals without any query and addresses P1.

\paragraph{Methods and hardware.}
The comparison baselines are \lib{pgmpy} with maximum-likelihood and with Bayesian CPTs, both queried by its variable elimination; \lib{pgmpy} linear-Gaussian, which conditions the joint Gaussian in closed form; \lib{pomegranate}, which is discrete only and uses its built-in loopy belief propagation; and \lib{Pyro}, where the BN is written as a program with smoothed tables or linear-Gaussian conditionals and queried by importance sampling from the prior with $50$ particles (the largest number its sequential sampler affords within the budgets). On continuous networks \lib{pgmpy} and \lib{Pyro} always fit linear-Gaussian conditionals, so on the non-Gaussian family they are the misspecified classical model. 

\nbn{} runs every applicable pair of a mechanism and an engine, with $2048$ particles for the sampling engines ($16$ in the batching scenario, which measures time only). To keep the figures readable, each panel shows all applicable external baselines plus the best \nbn{} methods of the family, selected by accuracy. Appendix~\ref{app:results} reports every method on every scenario. 
All runs use one server with a 20-core CPU, 126\,GB of RAM and one GPU with 12\,GB of memory; \nbn{} and \lib{pomegranate} run on the GPU, \lib{pgmpy} on the CPU, and \lib{Pyro} on the CPU except in the batching scenario.

\section{Experimental Results}
\label{sec:results}


Section~\ref{sec:nbn} makes two claims: an NBN keeps what a BN offers, while enhancing expressiveness of conditionals and scaling up inference performance. 
The results support both, within limits that the benchmark makes visible. 
In terms of our research questions:
\begin{enumerate}[label=\textbf{RQ\arabic*}, wide=0pt, leftmargin=*, align=left, itemsep=0pt, parsep=0pt, topsep=0pt]
  \item[\textbf{RQ1.}] \emph{No loss where the classical model is right} (Proposition~\ref{prop:semantic-equivalence}). On discrete networks, exact inference in \nbn{} returns the same TV as \lib{pgmpy} at every size, and on real networks it matches or improves it on seven out of eight; learning is identical as well. The exception is inference on linear-Gaussian networks, where the closed form of \lib{pgmpy} is more accurate than our sampled posterior (W1 $0.039$ against $0.074$ at 1000 nodes).
  \item[\textbf{RQ2.}] \emph{Expressiveness pays where it is wrong.} On non-Gaussian networks the MDN has between one third and $60\%$ of the W1 of a linear-Gaussian model, and achieves better held-out likelihood: where the outcome has several modes or a state-dependent spread, the linear-Gaussian family misstates the uncertainty while the MDN learns its shape.
  \item[\textbf{RQ3.}] \emph{The cost model holds on the GPU} (Proposition~\ref{prop:complexity}). With the sampling engines the time per query falls by two orders of magnitude between one query and a batch of 1024, down to tens of microseconds.
\end{enumerate}
While \lib{pgmpy} achieves top accuracy wherever one of its two families is the right model, it cannot represent non-Gaussian conditionals and answers one query at a time. 
\lib{pomegranate} leverages the GPU, but its loopy belief propagation returns undefined posteriors or exceeds budget on all networks but the smallest. 
\lib{Pyro} runs a generic sequential importance sampler, which the budget limits to $50$ particles: it has the highest error on all synthetic families and times out from 500 nodes onwards.
%
The performance of \nbn{} can still be improved, and the benchmark shows where. The sampling engines pay a Monte-Carlo error where a closed form exists, the MDN exceeds the fit budget between 100 and 500 nodes, and the nearest-neighbour mechanism exhausts the GPU memory at 500. 

Below, \nbn{} methods are named by their mechanisms and engines: cat+VE is categorical table queried by variable elimination; likewise for linear-Gaussian (LG) and MDN mechanisms, and for likelihood weighting (LW) and amortized variational inference (AVI).

\subsection{Inference accuracy (RQ1, RQ2)}
\label{subsec:exp-accuracy}

Figure~\ref{fig:inference} shows the accuracy of the 128 queries for each family.

\begin{figure*}[t]
\centering
\figfile{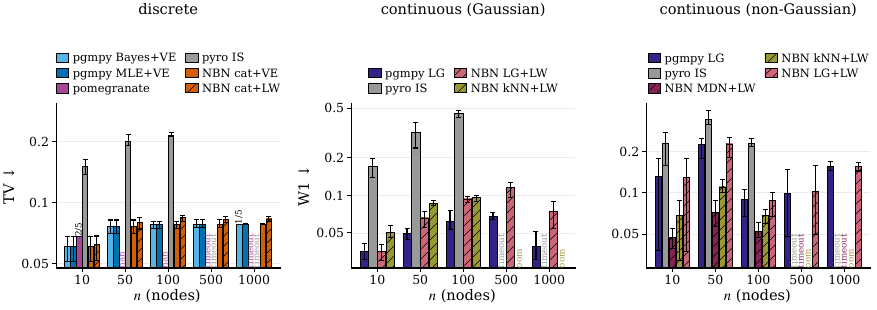}{1.6in}
\caption{Inference accuracy against network size on synthetic networks (TV for discrete, W1 for continuous, lower is better). All applicable external baselines are shown, plus the selected \nbn{} methods per family. A missing bar is a method that failed on every seed, with its status printed; $f/n$ marks $f$ failed seeds out of $n$.}
\label{fig:inference}
\end{figure*}

\paragraph{Discrete.}
cat+VE and \lib{pgmpy} return the same TV at every size 
because they run exact inference on the same maximum-likelihood tables: this is Proposition~\ref{prop:semantic-equivalence}, measured. Their common value is the floor of the reference (Section~\ref{sec:exp}). The sampling engines of \nbn{} stay within $0.01$ of it, so on discrete networks approximation costs little.

\paragraph{Linear-Gaussian.}
This is the one family where a baseline is ahead: 
\lib{pgmpy} conditions the joint Gaussian in closed form, while \nbn{} samples. The gap grows with the size because, in deeper networks, fewer prior samples agree with the evidence: the median effective sample size falls from all $2048$ particles at 10 nodes to $13\%$ of them at 1000. The error is due to the engine, not the model, and a closed-form engine for all-Gaussian networks would remove it.

\paragraph{Non-Gaussian.}
Here the model is decisive: 
every linear-Gaussian method has the same error, in \nbn{} or in \lib{pgmpy} and whatever the engine, because the family cannot represent the data. 
The MDN has one third to $60\%$ of that error and kNN one half to three quarters. 
The price is paid at fit time (Table~\ref{tab:app_inference_continuous_nongauss_fit_time} in Appendix~\ref{app:inference}): the MDN exceeds the fit budget at 500 nodes, where kNN runs out of memory, so beyond 100 nodes only the misspecified models remain.

\subsection{Scalability}
\label{subsec:exp-scalability}

\begin{table}[!b]
\centering\footnotesize\setlength{\tabcolsep}{3pt}
\caption{Accuracy on the heaviest queries with full evidence against network size, up to $10\,000$ nodes (IQM over seeds, lower is better). 
$(f/n)$ marks $f$ failed seeds out of $n$. 
Bold: fewest failed seeds, then best value. Times in Appendix~\ref{app:scalability}.\label{tab:scalability}}
\input{figures/main_tables/scalability.tex}
\end{table}

Table~\ref{tab:scalability} pushes the heaviest queries to $10\,000$ nodes.
Accuracy stays in the same range at every size; what this scenario reveals is why each method stops and how long a query takes. 
Exact inference is bounded by the elimination itself. \lib{pgmpy} completes every seed but needs tens of seconds per query from 1000 nodes; cat+VE answers in $0.2$ to $0.4$\,s, two orders of magnitude faster, and exceeds the 12\,GB of the GPU on two of five seeds at $10\,000$ nodes. The sampling engines have no such wall: cat+LW and LG+LW finish every seed at $10\,000$ nodes, where LG+LW is also faster than the closed form of \lib{pgmpy} ($2.6$\,s against $17.5$\,s per query), being the fastest up to 1000 nodes. The flexible mechanisms stop earlier due to the cost of learning: at 1000 nodes the MDN exceeds the fit budget and kNN the memory. 

\subsection{Batched queries (RQ3)}
\label{subsec:exp-batch}

\begin{figure*}[t]
\centering
\figfile{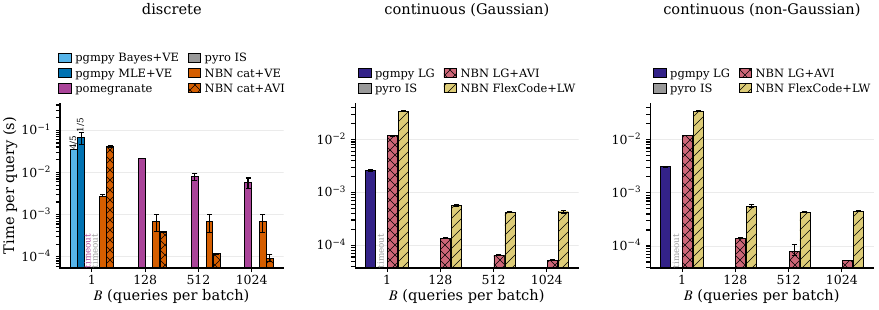}{1.6in}
\caption{Time per query vs. batch size $B$ (log scale) on 50-node networks, full evidence. Non-batching methods are at $B{=}1$ only.}
\label{fig:batch}
\end{figure*}

Proposition~\ref{prop:complexity} postulates that the cost of a batch does not grow with its size until the device saturates, so the time per query should fall as $1/B$. We fix 50-node networks and ask each heaviest query for the evidence values of $B$ different rows (Figure~\ref{fig:batch}). 
The sampling engines follow the prediction closely up to $B{=}128$, where a query costs about one hundredth of a single one, and then flattens as the GPU saturates: at $B{=}1024$ the gain is $456\times$ for cat+AVI and $227\times$ for LG+AVI, down to $0.09$ and $0.05$\,ms per query. Exact cat+VE is the fastest method on one query ($2.8$\,ms) but gains only $4\times$: every factor of the elimination carries the batch axis, so its work grows with $B$ and batching removes only the per-call overhead. Among the baselines only \lib{pomegranate} batches, and at $B{=}1024$ it is $60$ times slower than cat+AVI; \lib{pgmpy} answers one query at a time. The gain is in the batch, not in the latency of one query: on a single continuous query the closed form of \lib{pgmpy} takes $2.6$\,ms and \nbn{} $12$\,ms. An agent that asks one question at a time is better served by the classical code; one that scores many candidate actions or particles at once is the case NBNs are built for.

\subsection{Real networks}
\label{subsec:exp-bnlearn}

Table~\ref{tab:bnlearn} repeats the accuracy comparison on thirteen \lib{bnlearn} networks: eight discrete ones, from 8 to 1041 nodes, two Gaussian ones, and three conditional linear-Gaussian (CLG) ones, which mix discrete and continuous variables.

\begin{table}[!b]
\centering\footnotesize\setlength{\tabcolsep}{3pt}
\caption{Accuracy on \lib{bnlearn} networks, grouped by kind and ordered by node count $n$ (IQM over seeds; TV for discrete, W1 for Gaussian and CLG, lower is better). 
Each block names the method of each library; -- denotes no applicable engine. 
\lib{pgmpy} with Bayesian estimates and the other \nbn{} engines are in Appendix~\ref{app:bnlearn}.\label{tab:bnlearn}}
\input{figures/main_tables/bnlearn.tex}
\end{table}

The conclusions of the synthetic scenarios carry over. On the discrete networks cat+VE has the lowest TV on five, including the two largest, and is within $0.001$ of the best on two more; \lib{pgmpy} is better on \texttt{mildew} only. Amortized inference suffers more here than on synthetic networks (up to four times the TV of exact inference), so the exact engine is the one to use whenever the treewidth allows it. On the Gaussian networks \nbn{} is again close to the closed form of \lib{pgmpy}. The CLG networks are where coverage matters: \lib{pgmpy} and \lib{pomegranate} have no applicable engine, and the best \nbn{} mechanism is below \lib{Pyro} on all three. Which mechanism is best, however, depends on the network: the MDN wins on \texttt{sangiovese} and \texttt{mehra} and has six times the error of \lib{Pyro} on \texttt{healthcare}, where kNN wins. The choice of the mechanism is a modelling decision, and the benchmark is the tool to make it.

\subsection{Parameter learning and sample efficiency}
\label{subsec:exp-learning}

The learning scenarios score the fitted conditionals with no query: by the parameter-recovery TV on discrete networks, the distance between learned and true rows of the CPTs, and by held-out negative log-likelihood (NLL) per node on continuous ones 
(Appendix~\ref{app:design}).

\begin{table}[!b]
\caption{Sample efficiency on 50-node networks against the number of training rows $n_{\mathrm{tr}}$ (IQM over seeds, lower is better): parameter-recovery TV on discrete networks, held-out NLL per node on continuous ones. \emph{timeout}: the 1200\,s fit budget was exceeded on every seed. 
Fit times and the network-size sweep are in Appendices~\ref{app:param_learning} and~\ref{app:learning_curves}.\label{tab:learning}}
\centering\footnotesize\setlength{\tabcolsep}{3pt}
\input{figures/main_tables/learning.tex}
\end{table}

When the model is correct, \nbn{} learns what competitors do, as the estimator is the same: the tabular mechanism equals the max-likelihood tables of \lib{pgmpy} and \lib{pomegranate} at all sizes, its smoothed variant equals the tables of \lib{Pyro}, and the linear-Gaussian mechanism equals the least squares of \lib{pgmpy} (Table~\ref{tab:learning}; Appendix~\ref{app:param_learning}). It is also faster: a 1000-node network is fitted in $2.1$\,s (tables) and $2.6$\,s (linear-Gaussian), against $9.8$\,s and $4.0$\,s for \lib{pgmpy}. The accuracy drop when data is Gaussian is negligible: the NLL of the MDN is within $0.004$ of the exact estimator.

When the model is wrong, more data does not repair it (Table~\ref{tab:learning}). On 50-node non-Gaussian networks the NLL of the linear-Gaussian models stays at $1.26$--$1.27$ from $4096$ to $65\,536$ rows, while the MDN is at $1.20$ with $4096$ rows already and kNN improves steadily to $1.21$. Flexibility is thus sample-efficient but not compute-efficient: the fit time of the MDN grows linearly with the rows and nodes, exceeding the budget at $65\,536$ rows and from 100 nodes (Appendix~\ref{app:learning_curves}).

\section{Related Works}
\label{sec:related}



NBNs extend classical BNs in two directions, the \emph{expressiveness} of the conditionals and the efficiency and \emph{scale} of learning and inference; we review the works that pursue each of them.

\paragraph{Expressiveness.}
Kernel density estimation was used for continuous conditionals in Bayesian classifiers \cite{john1995estimating}, and non-parametric BNs based on kernels or copulas followed \cite{margaritis2005distribution,ickstadt2010nonparametric,hanea2015non}. Semiparametric BNs \cite{atienza2022semiparametric,atienza2022hybrid} mix linear-Gaussian and kernel conditionals in one network; surveys of inference with continuous variables are \cite{mori2016inference,salmeron2018review}. Neural structural causal models replace each structural equation with a neural density, typically a normalizing flow \cite{pawlowski2020dscm,khemakhem2021causal,javaloy2023cnf,xia2021ncm,geffner2022deci}. Closest to ours in name, the Neural Bayesian Network Understudy \cite{rabaey2022neural} trains a single Neural network to approximate the conditional probabilities of a given BN, with its independences encoded through training. These works fix one family of conditionals and one inference procedure. At the other extreme, deep generative models such as normalizing flows \cite{rezende2015flow,papamakarios2017maf} and neural world models \cite{ha2018world} are expressive but hide the dependencies in the weights of one network: there is no local conditional to inspect or replace, and the posterior given arbitrary evidence is not directly available.
An NBN changes the representation rather than adding a family. With tensor-valued mechanisms the BN is kept exactly (Proposition~\ref{prop:semantic-equivalence}); tabular, linear-Gaussian, neural and non-parametric \cite{izbicki2017converting} mechanisms can be mixed in one graph and trained by gradient, also jointly with a downstream objective (Proposition~\ref{prop:differentiability}); and any mechanism that can sample and evaluate a density is queried by the same batched engine, with no family-specific algorithm (Proposition~\ref{prop:complexity}).

\paragraph{Efficiency and scalability.}
Exact inference has been parallelized on multi-core CPUs \cite{jiang2023fast} and on GPUs \cite{zheng2011belief}, and written as tensor contractions \cite{obermeyer2019tensor}, but only for discrete networks and at a cost that stays exponential in the treewidth. Probabilistic circuits \cite{choi2020pc,peharz2020einsum} keep inference tractable and batched, but trade the per-variable conditionals for latent sum and product units. When exact inference is too expensive, a BN is queried by sampling. Likelihood weighting (Section~\ref{subsec:nbn-taxonomy}) is simple and general, but its samples are drawn without looking at the evidence, so most of them are wasted when the evidence is unlikely. Amortized inference \cite{gershman2014amortized,le2017inference} trains a neural network to propose samples that already agree with the evidence and reuses it for every later query.
We do not propose a new estimator: NBNs change how these estimators are executed. All the nodes at the same depth of the graph are evaluated in one tensor operation over every particle and every evidence assignment, so the cost of a query stops growing with the number of samples and of queries (Proposition~\ref{prop:complexity}). 

\paragraph{Libraries.}
Table~\ref{tab:bn_comparison} summarizes NBNs competitors. Classical BN libraries (\lib{pgmpy} \cite{ankan2024pgmpy}, \lib{pyAgrum} \cite{ducamp2020agrum}, \lib{bnlearn} \cite{scutari2010bnlearn}, \lib{CausalNex} \cite{beaumont2021causalnex}) implement tabular or linear-Gaussian conditionals with exact inference; they run on the CPU and accept several queries in one call but execute them one by one, which are the limitations (L1) and (L2) of Section~\ref{ssec:example}. \lib{pomegranate} \cite{schreiber2018pomegranate} is PyTorch-based and runs on the GPU, but its BN inference is limited to discrete networks and to loopy belief propagation, which is not guaranteed to converge on graphs with loops. Probabilistic programming systems (\lib{Pyro} \cite{bingham2019pyro}, \lib{PyMC} \cite{abril2023pymc}) support any distribution and GPU execution, but they do not expose the BN as an object: the graph, the per-node conditionals, and the inference procedure over them have to be written by hand for every model, and a generic sampler cannot eliminate variables exactly, prune the nodes that a query does not need, or evaluate a level of the graph at once. \lib{DoWhy} \cite{sharma2020dowhy,blobaum2024dowhy} estimates the causal effect of a treatment on an outcome and does not answer posterior queries over arbitrary variables.

To the best of our knowledge, no existing library combines an explicit DAG with differentiable conditionals of arbitrary family and GPU-batched learning and inference, which is the gap \nbn{} fills. Our benchmark compares against one representative per class: \lib{pgmpy} for the classical libraries, \lib{pomegranate} as the PyTorch-based one, and \lib{Pyro} as the probabilistic programming system. 

\begin{table}[t]
  \centering\footnotesize\renewcommand{\arraystretch}{0.92}
  \caption{Bayesian-network and probabilistic-modeling libraries. Disc.: discrete; LG: linear-Gaussian; NG:
  continuous non-Gaussian. Batch: the API accepts many queries in one call. GPU: BN queries run in parallel on the
  accelerator. Diff.: differentiable conditionals. $\dox$: native interventions. \cmark{} supported, \pmark{}
  partial, \xmark{} absent.\label{tab:bn_comparison}}
  \setlength{\tabcolsep}{3.5pt}
  \begin{tabular}{l|ccc|cc|cc}
  \toprule
  \multirow{2}{*}{\textbf{Library}} & \multicolumn{3}{c|}{\textbf{Families}} &
  \multicolumn{2}{c|}{\textbf{Compute}} & \multicolumn{2}{c}{\textbf{Capabilities}} \\
   & Disc. & LG & NG & Batch & GPU & Diff. & $\dox$ \\
  \midrule
  pgmpy \cite{ankan2024pgmpy}            & \cmark & \cmark & \xmark & \cmark & \xmark & \xmark & \pmark \\
  pyAgrum \cite{ducamp2020agrum}         & \cmark & \xmark & \xmark & \cmark & \xmark & \xmark & \xmark \\
  bnlearn \cite{scutari2010bnlearn}      & \cmark & \cmark & \xmark & \cmark & \xmark & \xmark & \xmark \\
  CausalNex \cite{beaumont2021causalnex} & \cmark & \xmark & \xmark & \cmark & \xmark & \xmark & \cmark \\
  pomegranate \cite{schreiber2018pomegranate} & \cmark & \cmark & \pmark & \cmark & \pmark & \xmark & \xmark \\
  \midrule
  PyMC \cite{abril2023pymc}              & \cmark & \cmark & \cmark & \pmark & \cmark & \pmark & \xmark \\
  Pyro \cite{bingham2019pyro}            & \cmark & \cmark & \cmark & \cmark & \cmark & \cmark & \pmark \\
  DoWhy \cite{sharma2020dowhy}           & \cmark & \cmark & \pmark & \xmark & \xmark & \xmark & \cmark \\
  \midrule
  \textbf{\nbn{} (ours)}                 & \cmark & \cmark & \cmark & \cmark & \cmark & \cmark & \cmark \\
  \bottomrule
  \end{tabular}
\end{table}

\section{Conclusions and Future Works}
\label{sec:conc}

Neural Bayesian Networks (NBNs) are an extension to Bayesian networks that combine explicit graphical structure with learnable, tensor-valued conditional mechanisms. 
NBNs preserve the probabilistic factorization and modularity of BNs while bringing heterogeneous conditional models into a common framework for learning and GPU-batched inference. 
The accompanying PyTorch library and reproducible benchmark provide a basis for evaluating the trade-offs between expressiveness, accuracy, and computational cost.
%
%
NBNs provide to autonomous agents a tool for decision-making under uncertainty that combines structured world models with the flexibility and computational tools of modern machine learning. 
Explicit dependencies and separately replaceable mechanisms support adaptation of individual model components, while batched inference serve agents that evaluate many candidate actions or hypotheses. 
Evaluation of NBNs within agent decision making tasks (e.g.\ causal reinforcement learning), removing the assumption of having the graph available, evaluating sampling, and considering CPU execution along GPU are amongst the future works we are considering.

\balance
\bibliographystyle{ACM-Reference-Format}
\bibliography{sample}

\clearpage\onecolumn\appendix
\input{appendix_design}
\input{appendix}

\end{document}

%% file: figures/fig_overview.tex
\begin{tikzpicture}[
  font=\scriptsize,
  >={Latex[length=1.6mm]},
  var/.style={circle,draw,thick,minimum size=14pt,inner sep=0pt,fill=white},
  mech/.style={draw,rounded corners=2pt,fill=red!6,inner sep=3pt,align=center},
  tab/.style={draw,fill=gray!8,inner sep=2pt,align=center},
  lbl/.style={align=center},
  ten/.style={draw,fill=red!12,minimum width=9mm,minimum height=5.5mm,inner sep=0pt},
]
\begin{scope}
  \node[var] (s1) at (0,1.5) {$s_1$};
  \node[var] (s2) at (1.2,1.5) {$s_2$};
  \node[var] (a)  at (2.4,1.5) {$a$};
  \node[var] (t1) at (0.6,0.4) {$s'_1$};
  \node[var] (t2) at (1.8,0.4) {$s'_2$};
  \draw[->] (s1)--(t1); \draw[->] (s2)--(t1); \draw[->] (s2)--(t2); \draw[->] (a)--(t2); \draw[->] (a)--(t1);
  \node[tab] (cpt) at (4.7,1.0) {%
    \begin{tabular}{@{}c|cc@{}}
      $s_1,s_2,a$ & $s'_1{=}0$ & $s'_1{=}1$\\ \hline
      $0,0,0$ & $.92$ & $.08$\\
      $0,0,1$ & $.35$ & $.65$\\
      $\vdots$ & $\vdots$ & $\vdots$\\
      \multicolumn{3}{@{}c@{}}{$K_1 K_2 K_a$ rows, $K_{1}'$ columns}
    \end{tabular}};
  \draw[->,dashed,gray] (t1.east) -- (cpt.west);
  \node[lbl,anchor=north,text width=6.2cm] at (3.0,-0.1) {%
    \textbf{Classical BN}\\
    scalar CPD per node: a table, or $\mathcal{N}(w^\top x{+}b,\sigma^2)$\\
    fit by counting (MLE); queried by exact inference\\
    one node, one sample, one query at a time};
\end{scope}
\draw[->,very thick] (6.5,0.9) -- (7.4,0.9);
\node[lbl] at (6.95,1.45) {same DAG\\ same $\prod_i p_i$};
\begin{scope}[xshift=7.9cm]
  \node[var] (s1) at (0,1.5) {$s_1$};
  \node[var] (s2) at (1.2,1.5) {$s_2$};
  \node[var] (a)  at (2.4,1.5) {$a$};
  \node[var,fill=red!10] (t1) at (0.6,0.4) {$s'_1$};
  \node[var,fill=red!10] (t2) at (1.8,0.4) {$s'_2$};
  \draw[->] (s1)--(t1); \draw[->] (s2)--(t1); \draw[->] (s2)--(t2); \draw[->] (a)--(t2); \draw[->] (a)--(t1);
  \node[ten] (in3) at (3.95,0.85) {};
  \node[ten] (in2) at ($(in3)+(-0.09,0.09)$) {};
  \node[ten] (in)  at ($(in2)+(-0.09,0.09)$) {\tiny $B{\times}|\mathrm{Pa}|$};
  \node[mech,minimum height=9mm] (M) at (6.05,1.0) {$\mathcal{M}_1(\,\cdot\,;\phi_1)$\\[-1pt] \tiny table $\cdot$ linear-Gaussian\\[-2pt] \tiny MDN $\cdot$ kNN $\cdot$ series};
  \node[ten] (o3) at (8.15,0.85) {};
  \node[ten] (o2) at ($(o3)+(-0.09,0.09)$) {};
  \node[ten] (o)  at ($(o2)+(-0.09,0.09)$) {\tiny $B$ distr.};
  \draw[->] (in3.east) -- (M.west);
  \draw[->] (M.east) -- (o3.west);
  \draw[->,dashed,gray] (t1.east) -- (in3.south west);
  \node[lbl,anchor=north,font=\tiny] at (8.05,0.55) {\texttt{log\_prob}, \texttt{sample}\\ \texttt{rsample}, $\nabla_{\phi}$};
  \node[lbl,anchor=north,font=\tiny] at (3.9,0.55) {one call per\\ graph layer};
  \node[lbl,anchor=north,text width=8.6cm] at (4.1,-0.1) {%
    \textbf{Neural BN}\\
    tensor-valued differentiable mechanism per node, any family\\
    fit by gradient descent (or closed form); queried by VE / LW / AVI\\
    one tensor op per graph layer: all samples and all queries at once, on the GPU};
\end{scope}
\end{tikzpicture}

%% file: figures/fig_fmdp.tex
\begin{tikzpicture}[
  font=\scriptsize,
  >={Latex[length=1.5mm]},
  var/.style={circle,draw,thick,minimum size=13pt,inner sep=0pt,fill=white},
  act/.style={rectangle,draw,thick,minimum size=12pt,inner sep=0pt,fill=gray!12},
  big/.style={circle,draw,thick,minimum size=24pt,inner sep=0pt,fill=white},
  lbl/.style={align=center},
]
\begin{scope}
  \node[big] (s)  at (0,1.25) {$s$};
  \node[act,minimum size=16pt] (a) at (0,0.15) {$a$};
  \node[big] (sp) at (1.75,1.25) {$s'$};
  \draw[->,thick] (s) -- (sp);
  \draw[->,thick] (a) -- (sp);
  \node[lbl,anchor=north] at (0.85,-0.35) {\textbf{MDP}\\ one kernel $P(s'\mid s,a)$\\ over the joint state};
\end{scope}
\draw[->,very thick] (2.55,0.75) -- (3.25,0.75);
\node[lbl] at (2.9,1.2) {known\\ structure};
\begin{scope}[xshift=3.9cm]
  \node[var] (s1) at (0,1.8)  {$s_1$};
  \node[var] (s2) at (0,1.15) {$s_2$};
  \node[var] (s3) at (0,0.5)  {$s_3$};
  \node[act] (a)  at (0,-0.15) {$a$};
  \node[var] (t1) at (1.9,1.8)  {$s'_1$};
  \node[var] (t2) at (1.9,1.15) {$s'_2$};
  \node[var] (t3) at (1.9,0.5)  {$s'_3$};
  \draw[->] (s1) -- (t1);
  \draw[->] (s2) -- (t1);
  \draw[->] (s2) -- (t2);
  \draw[->] (s3) -- (t3);
  \draw[->] (a)  -- (t2);
  \draw[->] (a)  -- (t3);
  \node[lbl,font=\tiny,gray] at (0,2.2) {time $t$};
  \node[lbl,font=\tiny,gray] at (1.9,2.2) {time $t{+}1$};
  \node[lbl,anchor=north] at (0.95,-0.35) {\textbf{Factored MDP}\\ $\prod_j p\big(s'_j \mid \mathrm{Pa}(s'_j)\big)$\\ one small conditional per variable};
\end{scope}
\end{tikzpicture}

%% file: figures/main_tables/scalability.tex
\begin{tabular}{@{}lrrrr@{}}
\toprule
Method & $n{=}10$ & $n{=}100$ & $n{=}1000$ & $n{=}10000$ \\
\midrule
\multicolumn{5}{@{}l}{\emph{Discrete, TV}} \\
\lib{pgmpy} Bayes+VE & 0.0491 & 0.0578 & 0.0645 & 0.0323\,{\scriptsize (1/5)} \\
\lib{pgmpy} MLE+VE & 0.0438 & \textbf{0.0406} & 0.0582 & \textbf{0.0484} \\
\lib{pomegranate} & 0.0276\,{\scriptsize (2/5)} & \emph{nan} & \emph{nan} & \emph{timeout} \\
\lib{Pyro} IS & 0.120 & 0.188 & 0.151 & 0.141\,{\scriptsize (1/5)} \\
\nbn{} cat+VE & \textbf{0.0435} & 0.0476 & \textbf{0.0516} & 0.0376\,{\scriptsize (2/5)} \\
\nbn{} cat+LW & 0.0495 & 0.0600 & 0.0619 & 0.0586 \\
\midrule
\multicolumn{5}{@{}l}{\emph{Linear-Gaussian, W1}} \\
\lib{pgmpy} LG & 0.0303 & 0.0553 & \textbf{0.0378} & \textbf{0.0466} \\
\lib{Pyro} IS & 0.138 & 0.330 & 0.338 & 0.226\,{\scriptsize (1/5)} \\
\nbn{} LG+LW & \textbf{0.0291} & \textbf{0.0526} & 0.0652 & 0.0560 \\
\nbn{} MDN+LW & 0.0326 & 0.0534 & \emph{timeout} & \emph{timeout} \\
kNN+LW & 0.0557 & 0.102 & \emph{oom} & \emph{oom} \\
\midrule
\multicolumn{5}{@{}l}{\emph{Non-Gaussian, W1}} \\
\lib{pgmpy} LG & 0.115 & 0.142 & \textbf{0.0830} & 0.101 \\
\lib{Pyro} IS & 0.232 & 0.193 & 0.175 & 0.209\,{\scriptsize (1/5)} \\
\nbn{} LG+LW & 0.125 & 0.137 & 0.0917 & \textbf{0.0992} \\
\nbn{} MDN+LW & \textbf{0.0504} & 0.0414\,{\scriptsize (1/5)} & \emph{timeout} & \emph{timeout} \\
\nbn{} kNN+LW & 0.0755 & \textbf{0.0779} & \emph{oom} & \emph{oom} \\
\bottomrule
\end{tabular}

%% file: figures/main_tables/bnlearn.tex
\begin{tabular}{@{}lrrrrr@{}}
\toprule
Network ($n$) & \lib{pgmpy} & \lib{pomegr.} & \lib{Pyro} & \nbn{} & \nbn{} \\
\midrule
\emph{Discrete, TV} & \emph{MLE+VE} & \emph{BP} & \emph{IS} & \emph{cat+VE} & \emph{cat+AVI} \\
\texttt{asia} (8) & \textbf{0.00343} & 0.00380 & 0.0336 & 0.00354 & 0.00478 \\
\texttt{sachs} (11) & 0.0128 & 0.0122 & 0.0818 & \textbf{0.0120} & 0.0162 \\
\texttt{mildew} (35) & \textbf{0.0325} & \emph{nan} & 0.287 & 0.0465 & 0.0703 \\
\texttt{barley} (48) & 0.0275 & \emph{nan} & 0.215 & \textbf{0.0240} & 0.0401 \\
\texttt{hailfinder} (56) & \textbf{0.0203} & 0.0208 & 0.136 & 0.0206 & 0.0266 \\
\texttt{diabetes} (413) & 0.0393\,{\scriptsize (1/5)} & \emph{nan} & \emph{timeout} & \textbf{0.0343} & 0.128 \\
\texttt{link} (724) & 0.00197 & \emph{timeout} & \emph{timeout} & \textbf{0.00174} & 0.00612 \\
\texttt{munin} (1041) & 0.00464 & \emph{timeout} & \emph{timeout} & \textbf{0.00401} & 0.0105 \\
\midrule
\emph{Gaussian, W1} & \emph{LG} & \emph{--} & \emph{IS} & \emph{LG+AVI} & \emph{kNN+LW} \\
\texttt{ecoli70} (46) & \textbf{0.0418} & -- & 0.245 & 0.0541 & 0.111 \\
\texttt{arth150} (107) & \textbf{0.0210} & -- & 0.107 & 0.0212 & 0.0418 \\
\midrule
\emph{CLG, W1} & \emph{--} & \emph{--} & \emph{IS} & \emph{MDN+LW} & \emph{kNN+LW} \\
\texttt{healthcare} (7) & -- & -- & 106 & 625 & \textbf{23.4} \\
\texttt{sangiovese} (15) & -- & -- & 0.0583 & \textbf{0.0242} & 0.0280 \\
\texttt{mehra} (24) & -- & -- & $3.74{\cdot}10^{7}$ & $\mathbf{1.89{\cdot}10^{6}}$ & $2.41{\cdot}10^{6}$ \\
\bottomrule
\end{tabular}

%% file: figures/main_tables/learning.tex
\begin{tabular}{@{}lrrrr@{}}
\toprule
Method & $n_{\mathrm{tr}}{=}4096$ & $n_{\mathrm{tr}}{=}16384$ & $n_{\mathrm{tr}}{=}32768$ & $n_{\mathrm{tr}}{=}65536$ \\
\midrule
\multicolumn{5}{@{}l}{\emph{Discrete, parameter-recovery TV}} \\
\lib{pgmpy} Bayes & 0.0977 & 0.0474 & 0.0335 & 0.0236 \\
\lib{pgmpy} MLE & 0.0979 & 0.0474 & 0.0336 & 0.0236 \\
\lib{pomegranate} & 0.0979 & 0.0474 & 0.0336 & 0.0236 \\
\lib{Pyro} & \textbf{0.0847} & \textbf{0.0453} & \textbf{0.0328} & \textbf{0.0233} \\
\nbn{} cat & 0.0979 & 0.0474 & 0.0336 & 0.0236 \\
\nbn{} cat (Bayes) & \textbf{0.0847} & \textbf{0.0453} & \textbf{0.0328} & \textbf{0.0233} \\
\midrule
\multicolumn{5}{@{}l}{\emph{Gaussian, NLL per node}} \\
\lib{pgmpy} LG & \textbf{0.943} & \textbf{0.942} & \textbf{0.944} & \textbf{0.945} \\
\nbn{} LG & 0.935\,{\scriptsize (1/5)} & \textbf{0.942} & \textbf{0.944} & \textbf{0.945} \\
\nbn{} MDN & 0.953 & 0.946 & 0.958\,{\scriptsize (1/5)} & \emph{timeout} \\
\nbn{} kNN & 1.01 & 0.988 & 0.981 & 0.966\,{\scriptsize (1/5)} \\
\midrule
\multicolumn{5}{@{}l}{\emph{Non-Gaussian, NLL per node}} \\
\lib{pgmpy} LG & 1.26 & 1.27 & 1.27 & 1.27 \\
\nbn{} LG & 1.26 & 1.27 & 1.27 & 1.27 \\
\nbn{} MDN & \textbf{1.20} & \textbf{1.19} & \textbf{1.19} & \emph{timeout} \\
\nbn{} kNN & 1.25 & 1.23 & 1.22 & \textbf{1.21} \\
\bottomrule
\end{tabular}

%% file: appendix_design.tex
\section{Benchmark Design -- Extended Version}\label{app:design}

This appendix gives the settings of the benchmark that Section~\ref{sec:exp} leaves out. All of them are fields of the released configuration files or defaults of the released code, and can be changed.

\paragraph{Synthetic generator.}
The seed draws a DAG over $n$ ordered nodes, where each earlier node becomes a parent of a later one with probability $0.2$ and a node keeps at most four parents. The cap bounds the size of a conditional as $n$ grows: with four values per variable, a CPT has at most $4^5 = 1024$ entries. In the \emph{discrete} family every variable has four values and every row of a CPT is the softmax of standard-normal logits. In the \emph{linear-Gaussian} family every node has weights of variance $1/|\Pa|$, a standard-normal bias and a noise scale drawn uniformly in $[0.3, 1]$. In the \emph{non-Gaussian} family every node is a three-component Gaussian mixture: the component means are linear in the parents, with weights of standard deviation $0.3/\sqrt{|\Pa|}$ and offsets $0.3\,(z_1 + 0.7\,|z_2|)$ for standard-normal $z_1, z_2$, which are skewed; the component scales are uniform in $[0.3, 0.7]$; and the mixing weights are drawn from a flat Dirichlet distribution and do not depend on the parents.

\paragraph{Data.}
The training set has $20\,480$ rows in the accuracy, scalability and learning scenarios and $4096$ rows in the batching scenario; the sample-efficiency study uses $4096$, $16\,384$, $32\,768$ and $65\,536$ rows. On real networks it has $10\,240$, $20\,480$, $40\,960$ or $81\,920$ rows, for networks with up to $10^3$, $10^4$, $10^5$ or more parameters. The test set has $4096$ rows ($1024$ in the batching scenario and on real networks). The five seeds are $0$ to $4$.

\paragraph{Query selectors.}
\emph{Random}: one target and its evidence variables are drawn uniformly for the problem, and the query is asked for the evidence values of $128$ test rows. \emph{Topological} ($128$ queries per cell): targets are allocated by structural role, $30\%$ to hubs, $30\%$ to cut nodes, $30\%$ to terminals and $10\%$ at random. Hubs are ranked by the size of their Markov blanket, that is, the set of their parents, their children and the other parents of their children. Cut nodes are the articulation points of the moral graph: the undirected graph obtained by linking the parents of every node and dropping the direction of the edges, which a cut node disconnects when it is removed. Terminals are ranked by their depth minus their number of descendants. Hubs, cut nodes and random targets are asked in the prediction or in the diagnosis direction with equal probability, and terminals in the prediction direction only. The evidence is taken along the longest path ($40\%$ of the queries), inside the Markov blanket ($40\%$) or at random ($20\%$). \emph{Heaviest}: the top-ranked hub, cut node and terminal, each in the prediction and in the diagnosis direction and with full and with empty evidence, for at most twelve queries per cell (eight when the graph has no cut node); the evidence is taken along the longest path for hubs and terminals and inside the Markov blanket for cut nodes. In the batching scenario each of these queries is asked for the evidence values of $1024$ rows.

\paragraph{Reference.}
For a discrete target the reference set has $5000$ samples on synthetic networks and as many samples as the training set on real ones; a query whose evidence agrees with fewer than ten reference samples is not scored. For a continuous target the reference is a sample of $2000$ draws from $p(x_Q \mid x_E)$, obtained by likelihood weighting on the true model. Both references are finite samples, which gives the floor of Section~\ref{sec:exp}. For a discrete target it comes from the few reference samples that agree with the evidence, so it varies with the network and with the query. The floor is the average error of an exact answer against the sampled reference and not a bound: an approximate method can fall slightly below the exact engines.

\paragraph{Learning metrics.}
The learning scenarios score the fitted conditionals directly, without queries. On discrete networks the metric is the parameter-recovery TV: the total variation between a learned and a true row of a CPT, averaged over the parent configurations with their true probabilities and then over the nodes. Continuous families have no common parameters to compare, since a linear-Gaussian model, an MDN and a kNN estimator do not share a parameter space, so we score the fitted density itself with the held-out negative log-likelihood (NLL) per node: minus the log-density of a test row under the fitted network, averaged over the test rows and divided by $n$. The NLL is a strictly proper scoring rule. Its expectation is the entropy of the true conditionals plus the Kullback--Leibler divergence from the true to the fitted model, so on the same problem the difference between two methods is a difference in divergence from the truth, and it applies to every mechanism that exposes a density. We prefer it to two alternatives. W1, the metric of the inference scenarios, compares samples of a posterior and therefore needs a query and an engine, which would mix the error of learning with the error of inference. The calibration error PIT-KS (the Kolmogorov--Smirnov distance between the probability integral transform of held-out values and the uniform distribution) needs no query, but calibration is necessary and not sufficient: a model that ignores the parents and returns the marginal of each node is perfectly calibrated. Our runs show the difference: on 50-node non-Gaussian networks with $4096$ rows FlexCode has less than half the PIT-KS of the linear-Gaussian model ($0.028$ against $0.065$, close to the $0.023$ of the MDN) and the worst NLL of the three ($1.42$ against $1.26$ and $1.20$). PIT-KS is reported for every method in Appendix~\ref{app:param_learning}.

\paragraph{Batches with empty evidence.}
With empty evidence (Appendix~\ref{app:speed}) the $B$ queries of a batch are the same marginal query, so \nbn{} answers it once and the time per query is the time of one query divided by $B$; we do not count this as a batching gain.

\paragraph{Particles.}
The sampling engines of \nbn{} use $2048$ particles in the accuracy, scalability and real-network scenarios and $16$ in the batching scenario, which measures time only. \lib{Pyro} uses $50$ particles in every scenario.

\paragraph{Budgets and execution.}
The fit budget and the query budget are 1200\,s each; the query budget is 600\,s in the batching scenario. The memory cap of a cell is $80\%$ of the memory that is free on the machine when the cell starts. Each cell runs in its own process, so that a timeout or an out-of-memory failure ends that cell only.

\paragraph{Reporting.}
In the figures the error bar of a method spans the smallest and the largest value that enters its interquartile mean, and the value axes are logarithmic. Each panel has its own legend; a method has the same colour in every figure, and for \nbn{} the colour is the mechanism and the hatch is the engine. Each panel of the main text shows all applicable external baselines plus two \nbn{} methods, the best parametric or neural one and the best non-parametric one (the two best on discrete networks), selected once per family on the accuracy metric of the scenario. On the continuous families of the accuracy and scalability scenarios the linear-Gaussian mechanism is shown as well, so that the \nbn{} method that reaches the largest networks is always visible. The tables of Appendix~\ref{app:results} report the IQM $\pm \frac{1}{2}$ of the interquartile range.

\paragraph{Hardware.}
All runs use one server with an Intel Xeon E5-2640 v4 CPU at 2.40\,GHz (20 cores), 126\,GB of RAM and an NVIDIA Titan X (Pascal) GPU with 12\,GB of memory.

%% file: appendix.tex
\clearpage
\section{Complete results}\label{app:results}

This appendix reports every method on every scenario of Section~\ref{sec:exp}, including the mechanisms and engines that the main text does not discuss (normalizing-flow, kernel-density and neural categorical conditionals, and the amortized importance-sampling engine AIS), and the fit and query times of the scenarios where the main text reports accuracy only. Methods carry the names of the configuration files, \texttt{library-model-engine}; for \nbn{} the model is the mechanism of the conditionals. In the figures a bar is the interquartile mean (IQM) over the seeds a method solved and the error bar spans the smallest and the largest value that enters that mean; $f/n$ above a bar means that $f$ of $n$ seeds failed, and a missing bar carries the failure status. The value axes are logarithmic. Each panel has its own legend; a method has the same colour in every figure, and for \nbn{} the colour is the mechanism and the hatch is the engine. In the tables a cell is IQM $\pm$ half the interquartile range, followed by $(f/n)$ when $f$ of the $n$ seeds failed for any reason (time budget, memory, NaN output, error); a status alone (\texttt{timeout}, \texttt{oom}, \texttt{nan}) means that no seed was solved, and -- that the method was not run in that configuration. In each column the best method is in bold: the one with the fewest failed seeds and, among those, the best IQM. In the learning scenarios the negative log-likelihood is the mean over the held-out rows, divided by the number of nodes.

\subsection{Inference accuracy}\label{app:inference}

\noindent\begin{minipage}{\textwidth}\centering\footnotesize
\includegraphics[width=\textwidth]{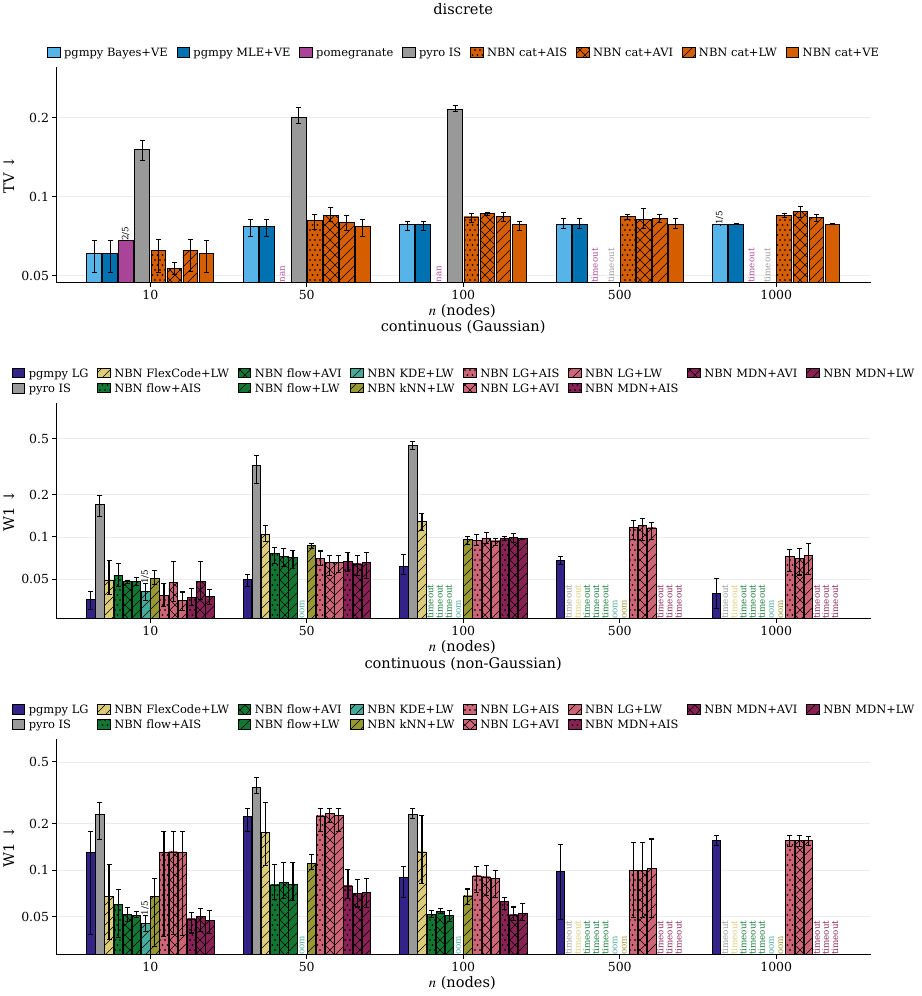}
\captionof{figure}{Inference accuracy, all methods: accuracy (TV on discrete targets, W1 on continuous ones) against network size.}\label{fig:app_inference_accuracy}
\end{minipage}\par\bigskip

\noindent\begin{minipage}{\textwidth}\centering\footnotesize
\includegraphics[width=\textwidth]{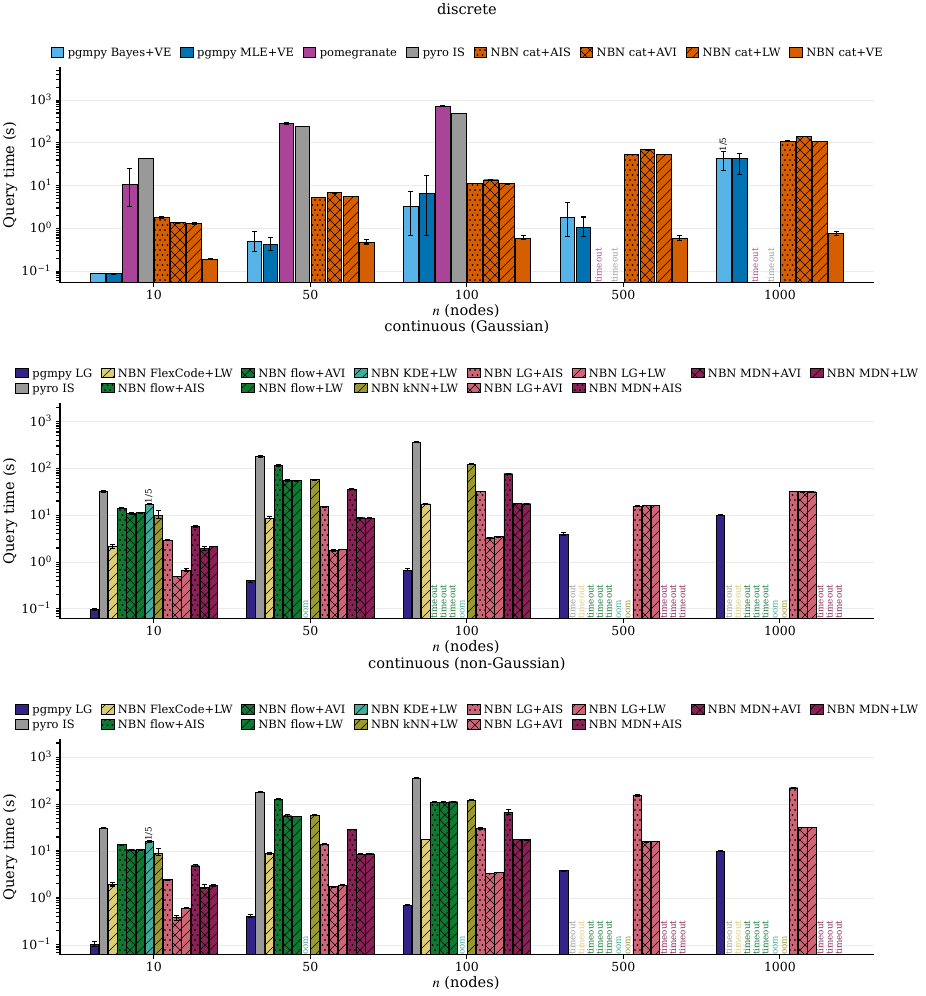}
\captionof{figure}{Inference accuracy, all methods: total query time (s) against network size.}\label{fig:app_inference_total_query_time}
\end{minipage}\par\bigskip

\noindent\begin{minipage}{\textwidth}\centering\footnotesize
\includegraphics[width=\textwidth]{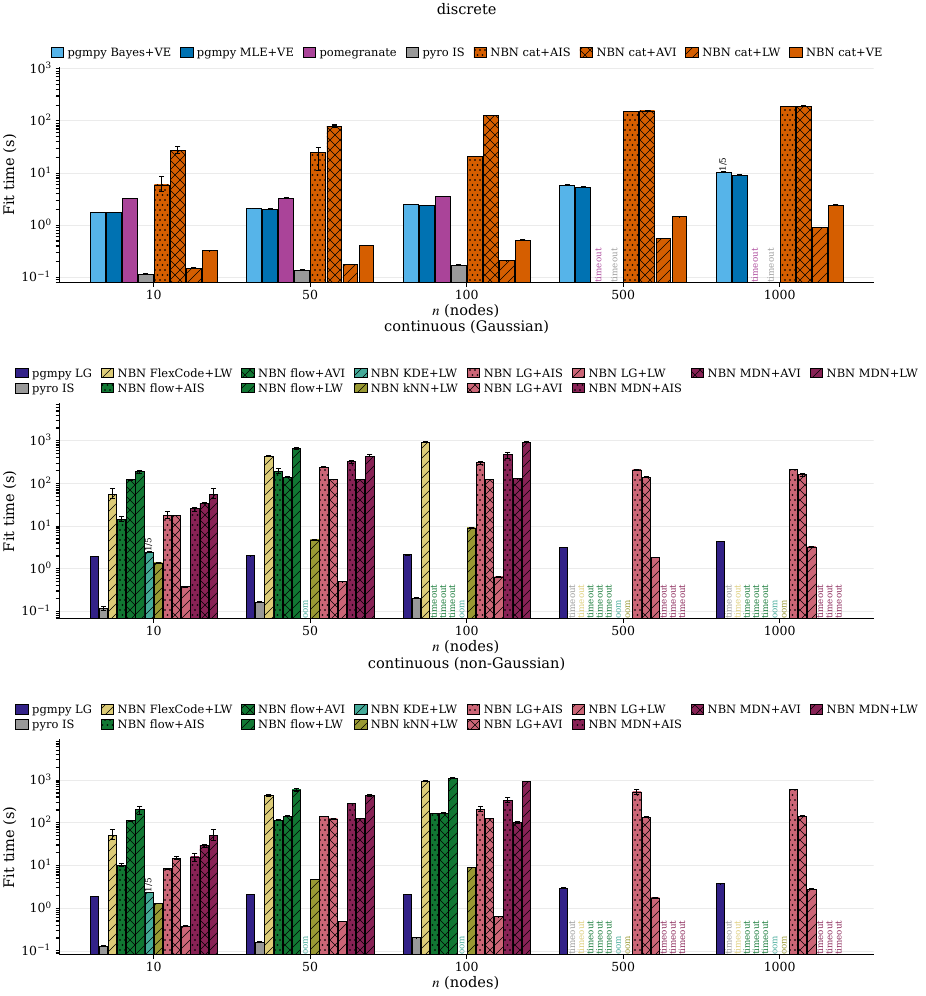}
\captionof{figure}{Inference accuracy, all methods: fit time (s) against network size.}\label{fig:app_inference_fit_time}
\end{minipage}\par\bigskip

\noindent\begin{minipage}{\textwidth}\centering\footnotesize
\captionof{table}{Inference accuracy, discrete networks: TV (lower is better), against network size.}\label{tab:app_inference_discrete_tv_per_node}\par\smallskip
\adjustbox{max width=\textwidth}{\begin{tabular}{lrrrrr}
\toprule
Method & $n=10$ & $n=50$ & $n=100$ & $n=500$ & $n=1000$ \\
\midrule
\texttt{pgmpy-bayes-ve} & 0.061$\pm$0.0083 & \textbf{0.0766$\pm$0.0059} & \textbf{0.0783$\pm$0.0031} & \textbf{0.0783$\pm$0.0035} & 0.078$\pm$0.0014 (1/5) \\
\texttt{pgmpy-mle-ve} & 0.061$\pm$0.0083 & \textbf{0.0766$\pm$0.0059} & \textbf{0.0783$\pm$0.0031} & \textbf{0.0783$\pm$0.0035} & \textbf{0.0782$\pm$0.00035} \\
\texttt{pomegranate-discrete-ve} & 0.0678$\pm$0.0054 (2/5) & nan & nan & timeout & timeout \\
\texttt{pyro-empirical-importance} & 0.151$\pm$0.013 & 0.201$\pm$0.013 & 0.215$\pm$0.0054 & timeout & timeout \\
\texttt{nbn-cat-ais} & 0.0625$\pm$0.0085 & 0.0812$\pm$0.0052 & 0.0836$\pm$0.0032 & 0.0838$\pm$0.0018 & 0.0845$\pm$0.0014 \\
\texttt{nbn-cat-avi} & \textbf{0.0534$\pm$0.0026} & 0.0845$\pm$0.0051 & 0.086$\pm$0.0012 & 0.082$\pm$0.0071 & 0.0875$\pm$0.0043 \\
\texttt{nbn-cat-lw} & 0.0624$\pm$0.0082 & 0.0799$\pm$0.0053 & 0.0842$\pm$0.003 & 0.0825$\pm$0.0027 & 0.0831$\pm$0.0024 \\
\texttt{nbn-cat-ve} & 0.061$\pm$0.0083 & \textbf{0.0766$\pm$0.0059} & \textbf{0.0783$\pm$0.0031} & \textbf{0.0783$\pm$0.0035} & \textbf{0.0782$\pm$0.00035} \\
\bottomrule
\end{tabular}}
\end{minipage}\par\bigskip

\noindent\begin{minipage}{\textwidth}\centering\footnotesize
\captionof{table}{Inference accuracy, discrete networks: total query time in seconds, against network size.}\label{tab:app_inference_discrete_total_query_time}\par\smallskip
\adjustbox{max width=\textwidth}{\begin{tabular}{lrrrrr}
\toprule
Method & $n=10$ & $n=50$ & $n=100$ & $n=500$ & $n=1000$ \\
\midrule
\texttt{pgmpy-bayes-ve} & 0.0893$\pm$0.0018 & 0.504$\pm$0.27 & 3.33$\pm$3.3 & 1.85$\pm$1.7 & 42.9$\pm$39 (1/5) \\
\texttt{pgmpy-mle-ve} & \textbf{0.0882$\pm$0.0028} & \textbf{0.43$\pm$0.16} & 6.63$\pm$8.1 & 1.05$\pm$0.6 & 42.5$\pm$19 \\
\texttt{pomegranate-discrete-ve} & 10.8$\pm$11 & 289$\pm$16 & 711$\pm$24 & timeout & timeout \\
\texttt{pyro-empirical-importance} & 43.1$\pm$0.81 & 243$\pm$0.31 & 490$\pm$3.4 & timeout & timeout \\
\texttt{nbn-cat-ais} & 1.81$\pm$0.081 & 5.27$\pm$0.026 & 11.1$\pm$0.13 & 54.2$\pm$0.54 & 109$\pm$1.7 \\
\texttt{nbn-cat-avi} & 1.36$\pm$0.031 & 6.78$\pm$0.055 & 13.6$\pm$0.075 & 69.3$\pm$1.8 & 140$\pm$0.92 \\
\texttt{nbn-cat-lw} & 1.31$\pm$0.046 & 5.59$\pm$0.066 & 11$\pm$0.18 & 54.2$\pm$0.57 & 108$\pm$0.79 \\
\texttt{nbn-cat-ve} & 0.193$\pm$0.0065 & 0.48$\pm$0.058 & \textbf{0.594$\pm$0.059} & \textbf{0.578$\pm$0.083} & \textbf{0.747$\pm$0.067} \\
\bottomrule
\end{tabular}}
\end{minipage}\par\bigskip

\noindent\begin{minipage}{\textwidth}\centering\footnotesize
\captionof{table}{Inference accuracy, discrete networks: fit time in seconds, against network size.}\label{tab:app_inference_discrete_fit_time}\par\smallskip
\adjustbox{max width=\textwidth}{\begin{tabular}{lrrrrr}
\toprule
Method & $n=10$ & $n=50$ & $n=100$ & $n=500$ & $n=1000$ \\
\midrule
\texttt{pgmpy-bayes-ve} & 1.77$\pm$0.012 & 2.09$\pm$0.0073 & 2.53$\pm$0.02 & 5.91$\pm$0.16 & 10.5$\pm$0.33 (1/5) \\
\texttt{pgmpy-mle-ve} & 1.77$\pm$0.0069 & 2.04$\pm$0.0044 & 2.4$\pm$0.026 & 5.41$\pm$0.052 & 9.17$\pm$0.053 \\
\texttt{pomegranate-discrete-ve} & 3.23$\pm$0.024 & 3.34$\pm$0.026 & 3.56$\pm$0.042 & timeout & timeout \\
\texttt{pyro-empirical-importance} & \textbf{0.116$\pm$0.001} & \textbf{0.137$\pm$0.0011} & \textbf{0.17$\pm$0.0055} & timeout & timeout \\
\texttt{nbn-cat-ais} & 5.93$\pm$2.1 & 24.6$\pm$10 & 20.8$\pm$0.022 & 153$\pm$0.59 & 187$\pm$2.5 \\
\texttt{nbn-cat-avi} & 27.1$\pm$4.3 & 79.1$\pm$5 & 127$\pm$0.29 & 156$\pm$1.6 & 193$\pm$3.1 \\
\texttt{nbn-cat-lw} & 0.151$\pm$0.0015 & 0.177$\pm$0.00027 & 0.209$\pm$0.0014 & \textbf{0.556$\pm$0.0032} & \textbf{0.907$\pm$0.0044} \\
\texttt{nbn-cat-ve} & 0.324$\pm$0.0012 & 0.41$\pm$0.0018 & 0.521$\pm$0.0062 & 1.46$\pm$0.013 & 2.44$\pm$0.017 \\
\bottomrule
\end{tabular}}
\end{minipage}\par\bigskip

\noindent\begin{minipage}{\textwidth}\centering\footnotesize
\captionof{table}{Inference accuracy, linear-Gaussian networks: W1 (lower is better), against network size.}\label{tab:app_inference_continuous_lg_w1_per_node}\par\smallskip
\adjustbox{max width=\textwidth}{\begin{tabular}{lrrrrr}
\toprule
Method & $n=10$ & $n=50$ & $n=100$ & $n=500$ & $n=1000$ \\
\midrule
\texttt{pgmpy-lg-predict} & 0.0359$\pm$0.0053 & \textbf{0.0495$\pm$0.0051} & \textbf{0.0617$\pm$0.011} & \textbf{0.0676$\pm$0.0043} & \textbf{0.0394$\pm$0.01} \\
\texttt{pyro-empirical-importance} & 0.17$\pm$0.029 & 0.321$\pm$0.071 & 0.449$\pm$0.027 & timeout & timeout \\
\texttt{nbn-flexcode-lw} & 0.0489$\pm$0.015 & 0.104$\pm$0.013 & 0.129$\pm$0.018 & timeout & timeout \\
\texttt{nbn-flow-ais} & 0.0534$\pm$0.01 & 0.0763$\pm$0.0093 & timeout & timeout & timeout \\
\texttt{nbn-flow-avi} & 0.048$\pm$0.001 & 0.0729$\pm$0.011 & timeout & timeout & timeout \\
\texttt{nbn-flow-lw} & 0.0485$\pm$0.0033 & 0.0712$\pm$0.011 & timeout & timeout & timeout \\
\texttt{nbn-kde-lw} & 0.0408$\pm$0.0087 (1/5) & oom & oom & oom & oom \\
\texttt{nbn-knn-lw} & 0.0506$\pm$0.0059 & 0.0864$\pm$0.004 & 0.096$\pm$0.0058 & oom & oom \\
\texttt{nbn-lg-ais} & 0.0384$\pm$0.0073 & 0.0699$\pm$0.0082 & 0.0947$\pm$0.0085 & 0.117$\pm$0.018 & 0.0723$\pm$0.012 \\
\texttt{nbn-lg-avi} & 0.0477$\pm$0.016 & 0.0659$\pm$0.01 & 0.0983$\pm$0.0093 & 0.121$\pm$0.021 & 0.0703$\pm$0.015 \\
\texttt{nbn-lg-lw} & \textbf{0.0353$\pm$0.0053} & 0.0659$\pm$0.0092 & 0.0937$\pm$0.0054 & 0.116$\pm$0.015 & 0.0744$\pm$0.018 \\
\texttt{nbn-mdn-ais} & 0.037$\pm$0.0051 & 0.0665$\pm$0.011 & 0.0981$\pm$0.0029 & timeout & timeout \\
\texttt{nbn-mdn-avi} & 0.0481$\pm$0.016 & 0.0643$\pm$0.01 & 0.099$\pm$0.0073 & timeout & timeout \\
\texttt{nbn-mdn-lw} & 0.0376$\pm$0.0046 & 0.0661$\pm$0.013 & 0.0971$\pm$0.00068 & timeout & timeout \\
\bottomrule
\end{tabular}}
\end{minipage}\par\bigskip

\noindent\begin{minipage}{\textwidth}\centering\footnotesize
\captionof{table}{Inference accuracy, linear-Gaussian networks: total query time in seconds, against network size.}\label{tab:app_inference_continuous_lg_total_query_time}\par\smallskip
\adjustbox{max width=\textwidth}{\begin{tabular}{lrrrrr}
\toprule
Method & $n=10$ & $n=50$ & $n=100$ & $n=500$ & $n=1000$ \\
\midrule
\texttt{pgmpy-lg-predict} & \textbf{0.0963$\pm$0.0032} & \textbf{0.399$\pm$0.019} & \textbf{0.676$\pm$0.035} & \textbf{3.98$\pm$0.24} & \textbf{10.1$\pm$0.25} \\
\texttt{pyro-empirical-importance} & 32.1$\pm$0.98 & 179$\pm$5.5 & 367$\pm$2 & timeout & timeout \\
\texttt{nbn-flexcode-lw} & 2.15$\pm$0.22 & 8.69$\pm$0.4 & 17.4$\pm$0.25 & timeout & timeout \\
\texttt{nbn-flow-ais} & 14.2$\pm$0.21 & 117$\pm$4.7 & timeout & timeout & timeout \\
\texttt{nbn-flow-avi} & 11$\pm$0.4 & 55.3$\pm$1.7 & timeout & timeout & timeout \\
\texttt{nbn-flow-lw} & 11.3$\pm$0.32 & 54.3$\pm$1.5 & timeout & timeout & timeout \\
\texttt{nbn-kde-lw} & 17.2$\pm$2.2 (1/5) & oom & oom & oom & oom \\
\texttt{nbn-knn-lw} & 10$\pm$2 & 57.2$\pm$1.1 & 122$\pm$1.8 & oom & oom \\
\texttt{nbn-lg-ais} & 2.92$\pm$0.064 & 15.1$\pm$0.39 & 31.7$\pm$0.31 & 15.6$\pm$0.2 & 32.2$\pm$0.23 \\
\texttt{nbn-lg-avi} & 0.495$\pm$0.0075 & 1.76$\pm$0.057 & 3.26$\pm$0.024 & 16.2$\pm$0.23 & 31.7$\pm$0.36 \\
\texttt{nbn-lg-lw} & 0.675$\pm$0.051 & 1.88$\pm$0.02 & 3.46$\pm$0.063 & 16.2$\pm$0.25 & 31.3$\pm$0.16 \\
\texttt{nbn-mdn-ais} & 5.75$\pm$0.16 & 36.1$\pm$1.3 & 76.1$\pm$0.71 & timeout & timeout \\
\texttt{nbn-mdn-avi} & 1.92$\pm$0.16 & 8.74$\pm$0.1 & 17.5$\pm$0.062 & timeout & timeout \\
\texttt{nbn-mdn-lw} & 2.15$\pm$0.034 & 8.64$\pm$0.11 & 17.4$\pm$0.16 & timeout & timeout \\
\bottomrule
\end{tabular}}
\end{minipage}\par\bigskip

\noindent\begin{minipage}{\textwidth}\centering\footnotesize
\captionof{table}{Inference accuracy, linear-Gaussian networks: fit time in seconds, against network size.}\label{tab:app_inference_continuous_lg_fit_time}\par\smallskip
\adjustbox{max width=\textwidth}{\begin{tabular}{lrrrrr}
\toprule
Method & $n=10$ & $n=50$ & $n=100$ & $n=500$ & $n=1000$ \\
\midrule
\texttt{pgmpy-lg-predict} & 1.95$\pm$0.028 & 2$\pm$0.0094 & 2.12$\pm$0.016 & 3.18$\pm$0.026 & 4.34$\pm$0.04 \\
\texttt{pyro-empirical-importance} & \textbf{0.12$\pm$0.012} & \textbf{0.164$\pm$0.0047} & \textbf{0.205$\pm$0.00058} & timeout & timeout \\
\texttt{nbn-flexcode-lw} & 54.4$\pm$16 & 437$\pm$17 & 928$\pm$18 & timeout & timeout \\
\texttt{nbn-flow-ais} & 14.6$\pm$1.9 & 196$\pm$26 & timeout & timeout & timeout \\
\texttt{nbn-flow-avi} & 122$\pm$3.1 & 141$\pm$1.4 & timeout & timeout & timeout \\
\texttt{nbn-flow-lw} & 189$\pm$15 & 678$\pm$36 & timeout & timeout & timeout \\
\texttt{nbn-kde-lw} & 2.45$\pm$0.12 (1/5) & oom & oom & oom & oom \\
\texttt{nbn-knn-lw} & 1.37$\pm$0.0097 & 4.73$\pm$0.013 & 9.06$\pm$0.022 & oom & oom \\
\texttt{nbn-lg-ais} & 18$\pm$3.1 & 238$\pm$5.1 & 305$\pm$27 & 207$\pm$1.1 & 212$\pm$1.2 \\
\texttt{nbn-lg-avi} & 18$\pm$0.19 & 123$\pm$0.055 & 125$\pm$0.69 & 142$\pm$3.7 & 162$\pm$9.7 \\
\texttt{nbn-lg-lw} & 0.377$\pm$0.0072 & 0.496$\pm$0.0013 & 0.644$\pm$0.0067 & \textbf{1.84$\pm$0.015} & \textbf{3.24$\pm$0.012} \\
\texttt{nbn-mdn-ais} & 25.4$\pm$3.2 & 322$\pm$21 & 479$\pm$73 & timeout & timeout \\
\texttt{nbn-mdn-avi} & 34.6$\pm$0.76 & 125$\pm$0.35 & 129$\pm$0.072 & timeout & timeout \\
\texttt{nbn-mdn-lw} & 55.3$\pm$15 & 441$\pm$28 & 930$\pm$11 & timeout & timeout \\
\bottomrule
\end{tabular}}
\end{minipage}\par\bigskip

\noindent\begin{minipage}{\textwidth}\centering\footnotesize
\captionof{table}{Inference accuracy, non-Gaussian networks: W1 (lower is better), against network size.}\label{tab:app_inference_continuous_nongauss_w1_per_node}\par\smallskip
\adjustbox{max width=\textwidth}{\begin{tabular}{lrrrrr}
\toprule
Method & $n=10$ & $n=50$ & $n=100$ & $n=500$ & $n=1000$ \\
\midrule
\texttt{pgmpy-lg-predict} & 0.131$\pm$0.07 & 0.223$\pm$0.036 & 0.0901$\pm$0.02 & \textbf{0.0982$\pm$0.049} & 0.156$\pm$0.012 \\
\texttt{pyro-empirical-importance} & 0.228$\pm$0.058 & 0.343$\pm$0.04 & 0.228$\pm$0.016 & timeout & timeout \\
\texttt{nbn-flexcode-lw} & 0.068$\pm$0.037 & 0.175$\pm$0.083 & 0.131$\pm$0.071 & timeout & timeout \\
\texttt{nbn-flow-ais} & 0.0599$\pm$0.019 & 0.0801$\pm$0.022 & 0.0516$\pm$0.0028 & timeout & timeout \\
\texttt{nbn-flow-avi} & 0.0517$\pm$0.0056 & 0.083$\pm$0.023 & 0.0539$\pm$0.002 & timeout & timeout \\
\texttt{nbn-flow-lw} & 0.051$\pm$0.0026 & 0.0809$\pm$0.024 & \textbf{0.0508$\pm$0.0043} & timeout & timeout \\
\texttt{nbn-kde-lw} & 0.0454$\pm$0.0079 (1/5) & oom & oom & oom & oom \\
\texttt{nbn-knn-lw} & 0.068$\pm$0.028 & 0.11$\pm$0.012 & 0.0681$\pm$0.0079 & oom & oom \\
\texttt{nbn-lg-ais} & 0.13$\pm$0.07 & 0.224$\pm$0.035 & 0.0917$\pm$0.017 & 0.0998$\pm$0.051 & \textbf{0.155$\pm$0.013} \\
\texttt{nbn-lg-avi} & 0.131$\pm$0.07 & 0.232$\pm$0.024 & 0.0903$\pm$0.019 & 0.0991$\pm$0.05 & \textbf{0.155$\pm$0.012} \\
\texttt{nbn-lg-lw} & 0.129$\pm$0.07 & 0.226$\pm$0.036 & 0.0878$\pm$0.016 & 0.102$\pm$0.055 & \textbf{0.155$\pm$0.011} \\
\texttt{nbn-mdn-ais} & 0.0483$\pm$0.007 & 0.0791$\pm$0.017 & 0.0627$\pm$0.0054 & timeout & timeout \\
\texttt{nbn-mdn-avi} & 0.0499$\pm$0.0084 & \textbf{0.0707$\pm$0.015} & 0.0513$\pm$0.0053 & timeout & timeout \\
\texttt{nbn-mdn-lw} & \textbf{0.0475$\pm$0.008} & 0.0715$\pm$0.016 & 0.0522$\pm$0.007 & timeout & timeout \\
\bottomrule
\end{tabular}}
\end{minipage}\par\bigskip

\noindent\begin{minipage}{\textwidth}\centering\footnotesize
\captionof{table}{Inference accuracy, non-Gaussian networks: total query time in seconds, against network size.}\label{tab:app_inference_continuous_nongauss_total_query_time}\par\smallskip
\adjustbox{max width=\textwidth}{\begin{tabular}{lrrrrr}
\toprule
Method & $n=10$ & $n=50$ & $n=100$ & $n=500$ & $n=1000$ \\
\midrule
\texttt{pgmpy-lg-predict} & \textbf{0.105$\pm$0.012} & \textbf{0.416$\pm$0.031} & \textbf{0.7$\pm$0.013} & \textbf{3.79$\pm$0.061} & \textbf{9.92$\pm$0.14} \\
\texttt{pyro-empirical-importance} & 30.7$\pm$0.55 & 180$\pm$6 & 358$\pm$3.2 & timeout & timeout \\
\texttt{nbn-flexcode-lw} & 1.93$\pm$0.19 & 8.72$\pm$0.37 & 17.4$\pm$0.33 & timeout & timeout \\
\texttt{nbn-flow-ais} & 13.5$\pm$0.14 & 126$\pm$2.6 & 110$\pm$2.1 & timeout & timeout \\
\texttt{nbn-flow-avi} & 10.6$\pm$0.34 & 56$\pm$2.9 & 110$\pm$1.6 & timeout & timeout \\
\texttt{nbn-flow-lw} & 10.6$\pm$0.22 & 54.8$\pm$0.25 & 111$\pm$1.5 & timeout & timeout \\
\texttt{nbn-kde-lw} & 16$\pm$2.1 (1/5) & oom & oom & oom & oom \\
\texttt{nbn-knn-lw} & 9.07$\pm$1.7 & 57.4$\pm$1.4 & 121$\pm$1.5 & oom & oom \\
\texttt{nbn-lg-ais} & 2.44$\pm$0.029 & 14.1$\pm$0.027 & 30.5$\pm$1.4 & 153$\pm$7.7 & 220$\pm$3.1 \\
\texttt{nbn-lg-avi} & 0.391$\pm$0.044 & 1.73$\pm$0.0097 & 3.27$\pm$0.022 & 15.7$\pm$0.2 & 31.5$\pm$0.28 \\
\texttt{nbn-lg-lw} & 0.612$\pm$0.0093 & 1.88$\pm$0.033 & 3.48$\pm$0.063 & 16$\pm$0.14 & 31.1$\pm$0.0079 \\
\texttt{nbn-mdn-ais} & 4.82$\pm$0.24 & 28.6$\pm$0.22 & 67$\pm$8.1 & timeout & timeout \\
\texttt{nbn-mdn-avi} & 1.71$\pm$0.14 & 8.61$\pm$0.032 & 17.5$\pm$0.22 & timeout & timeout \\
\texttt{nbn-mdn-lw} & 1.84$\pm$0.098 & 8.69$\pm$0.099 & 17.4$\pm$0.23 & timeout & timeout \\
\bottomrule
\end{tabular}}
\end{minipage}\par\bigskip

\noindent\begin{minipage}{\textwidth}\centering\footnotesize
\captionof{table}{Inference accuracy, non-Gaussian networks: fit time in seconds, against network size.}\label{tab:app_inference_continuous_nongauss_fit_time}\par\smallskip
\adjustbox{max width=\textwidth}{\begin{tabular}{lrrrrr}
\toprule
Method & $n=10$ & $n=50$ & $n=100$ & $n=500$ & $n=1000$ \\
\midrule
\texttt{pgmpy-lg-predict} & 1.92$\pm$0.0021 & 2.07$\pm$0.027 & 2.11$\pm$0.025 & 2.98$\pm$0.063 & 3.75$\pm$0.021 \\
\texttt{pyro-empirical-importance} & \textbf{0.131$\pm$0.0025} & \textbf{0.162$\pm$0.00045} & \textbf{0.204$\pm$0.0004} & timeout & timeout \\
\texttt{nbn-flexcode-lw} & 51.2$\pm$15 & 433$\pm$15 & 932$\pm$24 & timeout & timeout \\
\texttt{nbn-flow-ais} & 10.2$\pm$0.7 & 115$\pm$2.7 & 166$\pm$3 & timeout & timeout \\
\texttt{nbn-flow-avi} & 111$\pm$1.2 & 142$\pm$2.1 & 166$\pm$2.1 & timeout & timeout \\
\texttt{nbn-flow-lw} & 207$\pm$43 & 588$\pm$44 & 1.11e+03$\pm$15 & timeout & timeout \\
\texttt{nbn-kde-lw} & 2.31$\pm$0.11 (1/5) & oom & oom & oom & oom \\
\texttt{nbn-knn-lw} & 1.32$\pm$0.0052 & 4.74$\pm$0.0095 & 9.02$\pm$0.015 & oom & oom \\
\texttt{nbn-lg-ais} & 8.35$\pm$0.14 & 140$\pm$0.58 & 202$\pm$24 & 525$\pm$73 & 606$\pm$0.84 \\
\texttt{nbn-lg-avi} & 14.7$\pm$1.2 & 123$\pm$0.025 & 125$\pm$0.73 & 136$\pm$1 & 143$\pm$1.3 \\
\texttt{nbn-lg-lw} & 0.383$\pm$0.0055 & 0.495$\pm$0.0022 & 0.627$\pm$0.0054 & \textbf{1.73$\pm$0.032} & \textbf{2.79$\pm$0.068} \\
\texttt{nbn-mdn-ais} & 15.8$\pm$3.2 & 288$\pm$0.45 & 339$\pm$45 & timeout & timeout \\
\texttt{nbn-mdn-avi} & 29.1$\pm$2.6 & 126$\pm$0.64 & 100$\pm$5.4 & timeout & timeout \\
\texttt{nbn-mdn-lw} & 50.1$\pm$16 & 431$\pm$16 & 920$\pm$19 & timeout & timeout \\
\bottomrule
\end{tabular}}
\end{minipage}\par\bigskip

\clearpage

\subsection{Scalability}\label{app:scalability}

\noindent\begin{minipage}{\textwidth}\centering\footnotesize
\includegraphics[width=\textwidth]{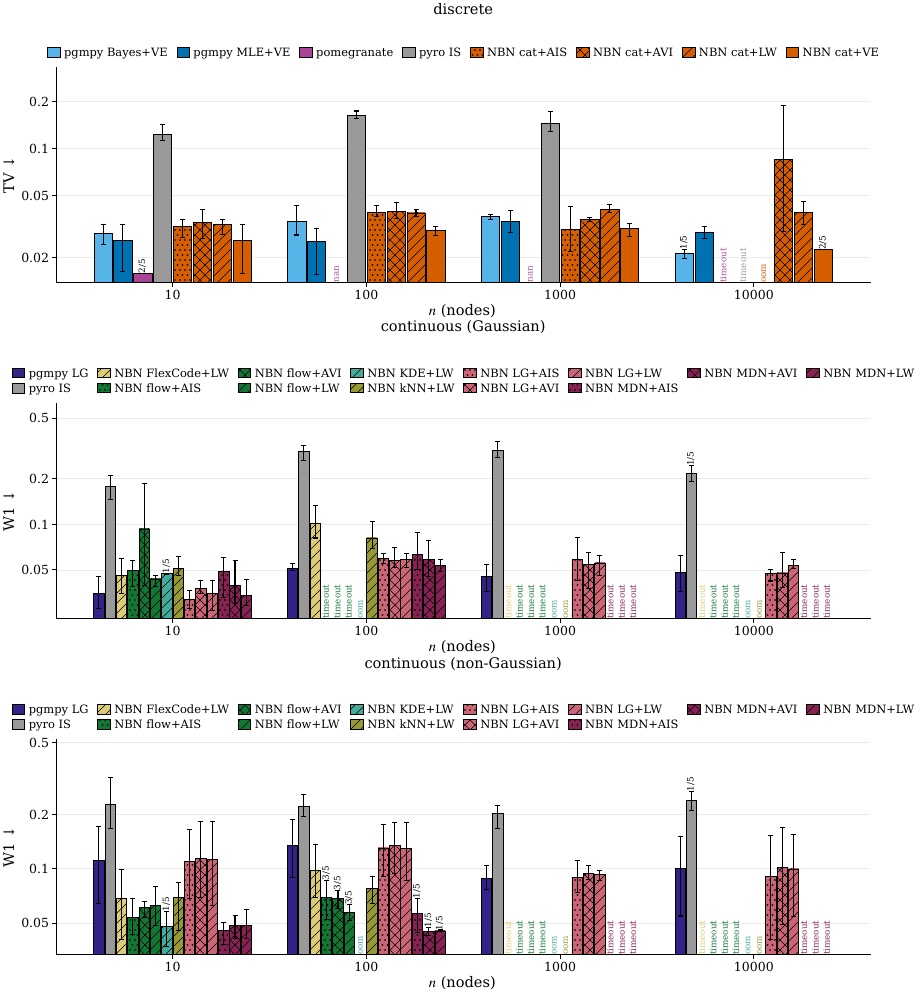}
\captionof{figure}{Scalability, all methods: accuracy (TV on discrete targets, W1 on continuous ones) against network size.}\label{fig:app_scalability_accuracy}
\end{minipage}\par\bigskip

\noindent\begin{minipage}{\textwidth}\centering\footnotesize
\includegraphics[width=\textwidth]{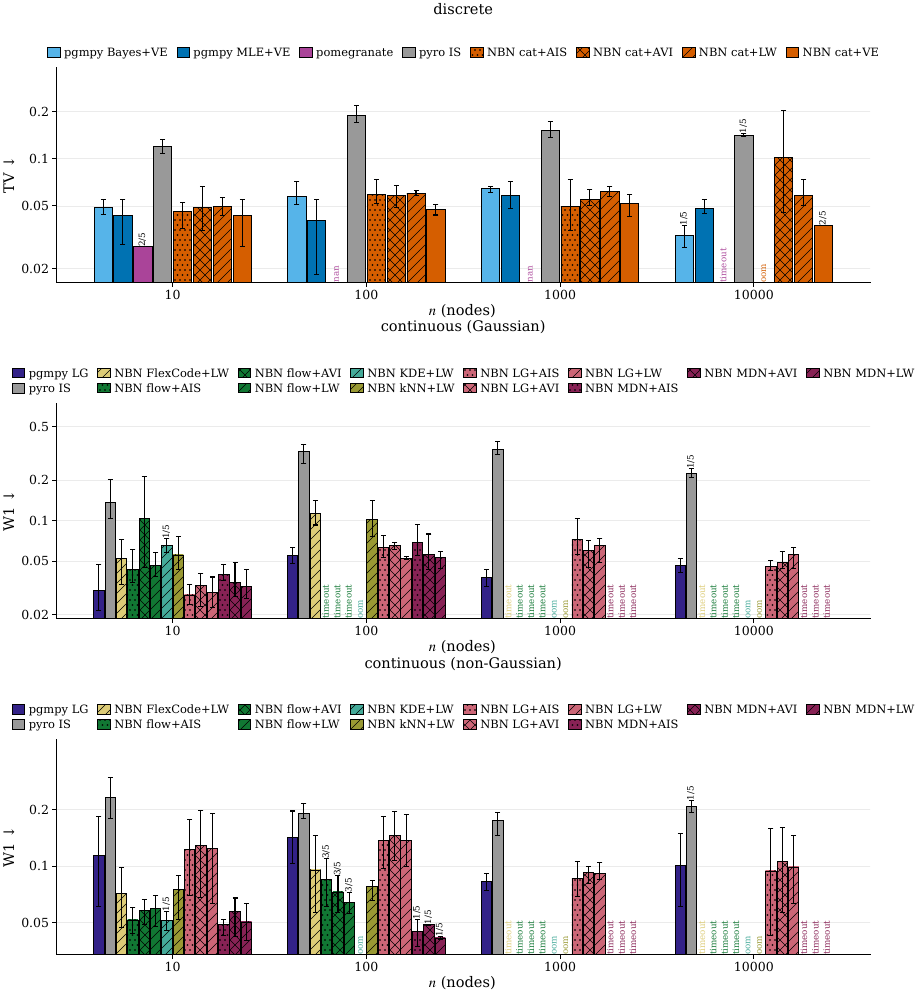}
\captionof{figure}{Scalability, all methods: accuracy with full evidence against network size.}\label{fig:app_scalability_accuracy_full}
\end{minipage}\par\bigskip

\noindent\begin{minipage}{\textwidth}\centering\footnotesize
\includegraphics[width=\textwidth]{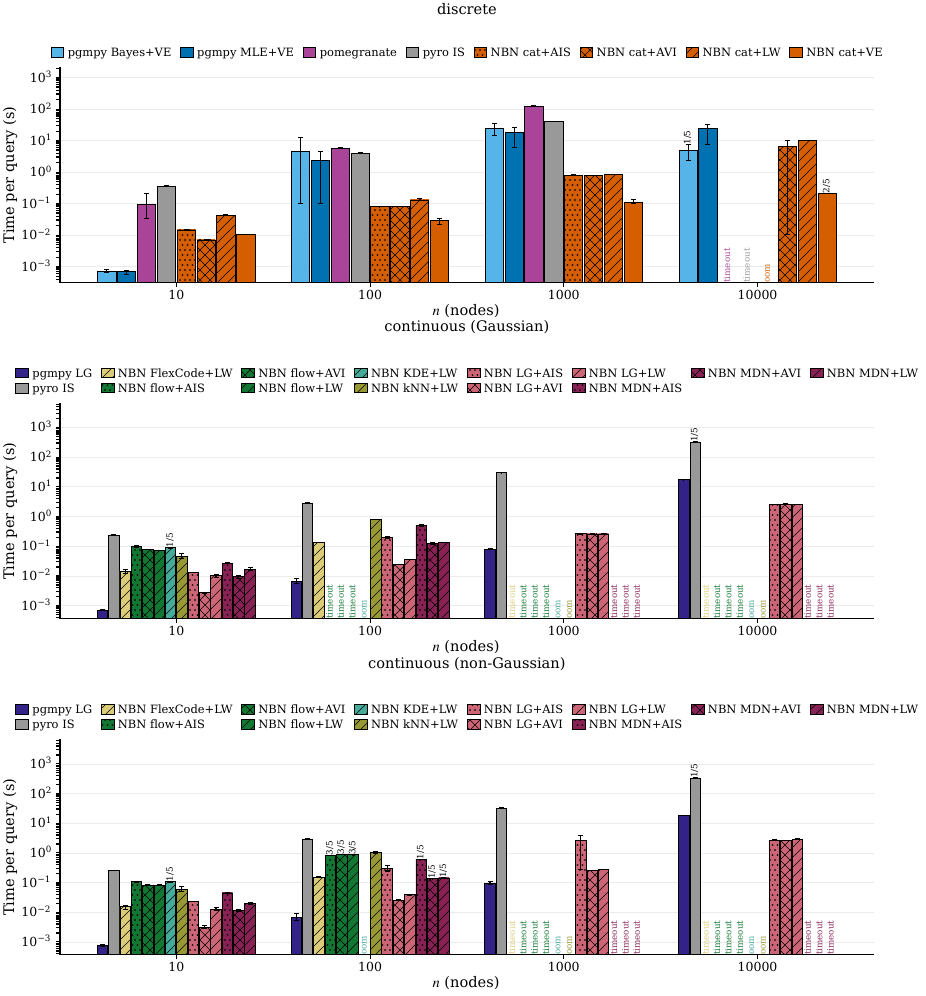}
\captionof{figure}{Scalability, all methods: time per query (s), both evidence modes against network size.}\label{fig:app_scalability_query_time}
\end{minipage}\par\bigskip

\noindent\begin{minipage}{\textwidth}\centering\footnotesize
\includegraphics[width=\textwidth]{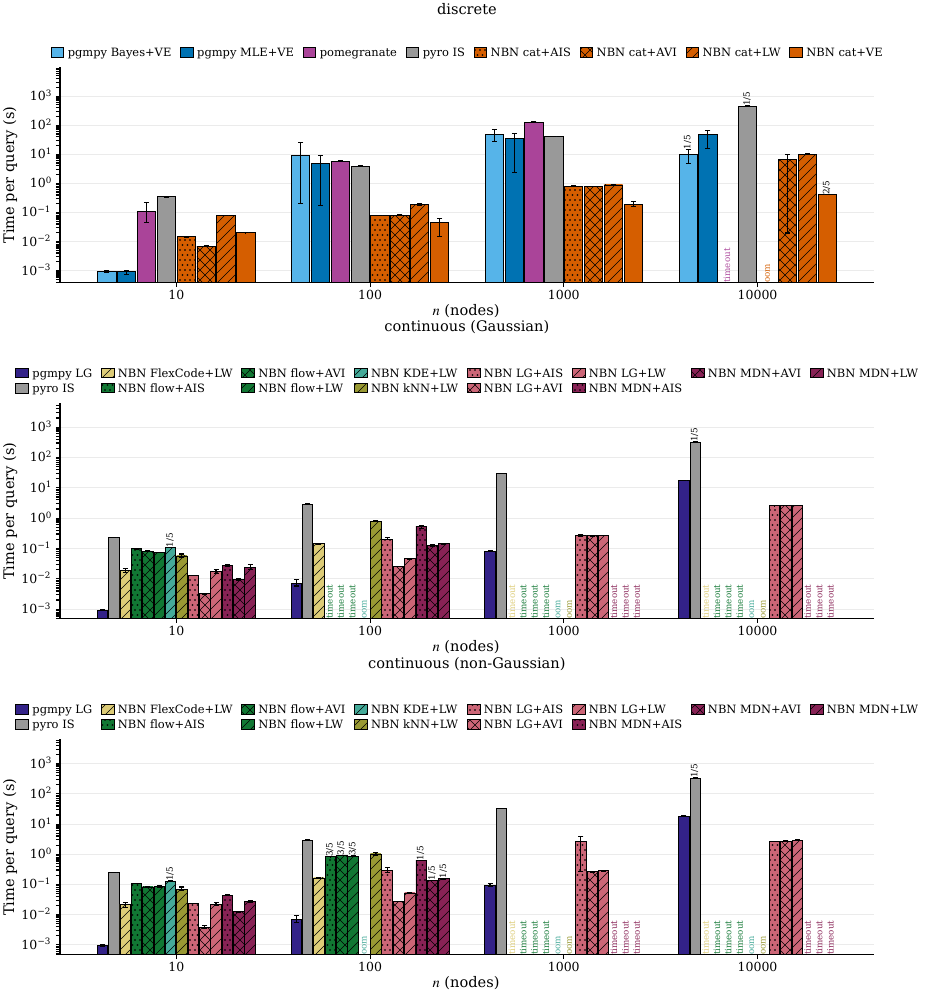}
\captionof{figure}{Scalability, all methods: time per query (s), full evidence against network size.}\label{fig:app_scalability_query_time_full}
\end{minipage}\par\bigskip

\noindent\begin{minipage}{\textwidth}\centering\footnotesize
\captionof{table}{Scalability, discrete networks: TV (lower is better), both evidence modes, against network size.}\label{tab:app_scalability_discrete_tv_per_node}\par\smallskip
\adjustbox{max width=\textwidth}{\begin{tabular}{lrrrr}
\toprule
Method & $n=10$ & $n=100$ & $n=1000$ & $n=10000$ \\
\midrule
\texttt{pgmpy-bayes-ve} & 0.0284$\pm$0.0041 & 0.0343$\pm$0.0076 & 0.037$\pm$0.0013 & 0.0212$\pm$0.0029 (1/5) \\
\texttt{pgmpy-mle-ve} & 0.0258$\pm$0.008 & \textbf{0.0256$\pm$0.0074} & 0.0343$\pm$0.0057 & \textbf{0.0291$\pm$0.0026} \\
\texttt{pomegranate-discrete-ve} & 0.0159$\pm$0.0034 (2/5) & nan & nan & timeout \\
\texttt{pyro-empirical-importance} & 0.123$\pm$0.015 & 0.164$\pm$0.0087 & 0.146$\pm$0.022 & timeout \\
\texttt{nbn-cat-ais} & 0.0317$\pm$0.004 & 0.0392$\pm$0.0034 & \textbf{0.0303$\pm$0.01} & oom \\
\texttt{nbn-cat-avi} & 0.0337$\pm$0.0071 & 0.0394$\pm$0.0046 & 0.0351$\pm$0.001 & 0.0847$\pm$0.079 \\
\texttt{nbn-cat-lw} & 0.0326$\pm$0.0037 & 0.0387$\pm$0.002 & 0.0409$\pm$0.0026 & 0.0389$\pm$0.0067 \\
\texttt{nbn-cat-ve} & \textbf{0.0256$\pm$0.0083} & 0.03$\pm$0.0018 & 0.0306$\pm$0.0029 & 0.0225$\pm$0.0028 (2/5) \\
\bottomrule
\end{tabular}}
\end{minipage}\par\bigskip

\noindent\begin{minipage}{\textwidth}\centering\footnotesize
\captionof{table}{Scalability, discrete networks: TV (lower is better), full evidence, against network size.}\label{tab:app_scalability_discrete_tv_per_node_full}\par\smallskip
\adjustbox{max width=\textwidth}{\begin{tabular}{lrrrr}
\toprule
Method & $n=10$ & $n=100$ & $n=1000$ & $n=10000$ \\
\midrule
\texttt{pgmpy-bayes-ve} & 0.0491$\pm$0.0055 & 0.0578$\pm$0.01 & 0.0645$\pm$0.0026 & 0.0323$\pm$0.0082 (1/5) \\
\texttt{pgmpy-mle-ve} & 0.0438$\pm$0.013 & \textbf{0.0406$\pm$0.018} & 0.0582$\pm$0.011 & \textbf{0.0484$\pm$0.0051} \\
\texttt{pomegranate-discrete-ve} & 0.0276$\pm$0.0087 (2/5) & nan & nan & timeout \\
\texttt{pyro-empirical-importance} & 0.12$\pm$0.012 & 0.188$\pm$0.024 & 0.151$\pm$0.019 & 0.141$\pm$0.046 (1/5) \\
\texttt{nbn-cat-ais} & 0.046$\pm$0.0082 & 0.0595$\pm$0.011 & \textbf{0.0498$\pm$0.019} & oom \\
\texttt{nbn-cat-avi} & 0.049$\pm$0.016 & 0.0583$\pm$0.0093 & 0.0551$\pm$0.0066 & 0.102$\pm$0.078 \\
\texttt{nbn-cat-lw} & 0.0495$\pm$0.0065 & 0.06$\pm$0.0022 & 0.0619$\pm$0.0046 & 0.0586$\pm$0.012 \\
\texttt{nbn-cat-ve} & \textbf{0.0435$\pm$0.014} & 0.0476$\pm$0.0035 & 0.0516$\pm$0.0081 & 0.0376$\pm$0.0067 (2/5) \\
\bottomrule
\end{tabular}}
\end{minipage}\par\bigskip

\noindent\begin{minipage}{\textwidth}\centering\footnotesize
\captionof{table}{Scalability, discrete networks: time per query in seconds, both evidence modes, against network size.}\label{tab:app_scalability_discrete_query_time}\par\smallskip
\adjustbox{max width=\textwidth}{\begin{tabular}{lrrrr}
\toprule
Method & $n=10$ & $n=100$ & $n=1000$ & $n=10000$ \\
\midrule
\texttt{pgmpy-bayes-ve} & 0.000709$\pm$7e-05 & 4.54$\pm$6.3 & 23.6$\pm$10 & 4.96$\pm$20 (1/5) \\
\texttt{pgmpy-mle-ve} & \textbf{0.000689$\pm$9.3e-05} & 2.43$\pm$2.2 & 18.9$\pm$11 & 23.8$\pm$13 \\
\texttt{pomegranate-discrete-ve} & 0.094$\pm$0.087 & 5.87$\pm$0.23 & 124$\pm$3.9 & timeout \\
\texttt{pyro-empirical-importance} & 0.359$\pm$0.0058 & 4.01$\pm$0.12 & 40.4$\pm$0.34 & timeout \\
\texttt{nbn-cat-ais} & 0.0146$\pm$0.0004 & 0.08$\pm$0.00099 & 0.808$\pm$0.019 & oom \\
\texttt{nbn-cat-avi} & 0.00704$\pm$0.00036 & 0.0806$\pm$0.0018 & 0.782$\pm$0.018 & \textbf{6.73$\pm$5.1} \\
\texttt{nbn-cat-lw} & 0.0435$\pm$0.00027 & 0.131$\pm$0.0095 & 0.831$\pm$0.015 & 10.1$\pm$0.035 \\
\texttt{nbn-cat-ve} & 0.0104$\pm$0.00012 & \textbf{0.0296$\pm$0.0059} & \textbf{0.113$\pm$0.016} & 0.212$\pm$0.024 (2/5) \\
\bottomrule
\end{tabular}}
\end{minipage}\par\bigskip

\noindent\begin{minipage}{\textwidth}\centering\footnotesize
\captionof{table}{Scalability, discrete networks: time per query in seconds, full evidence, against network size.}\label{tab:app_scalability_discrete_query_time_full}\par\smallskip
\adjustbox{max width=\textwidth}{\begin{tabular}{lrrrr}
\toprule
Method & $n=10$ & $n=100$ & $n=1000$ & $n=10000$ \\
\midrule
\texttt{pgmpy-bayes-ve} & 0.000942$\pm$7e-05 & 9.05$\pm$13 & 46.8$\pm$21 & 9.78$\pm$41 (1/5) \\
\texttt{pgmpy-mle-ve} & \textbf{0.000912$\pm$0.00013} & 4.83$\pm$4.3 & 34.5$\pm$26 & 47.5$\pm$26 \\
\texttt{pomegranate-discrete-ve} & 0.107$\pm$0.087 & 5.88$\pm$0.21 & 125$\pm$4.3 & timeout \\
\texttt{pyro-empirical-importance} & 0.344$\pm$0.0056 & 3.98$\pm$0.15 & 40.4$\pm$0.4 & 461$\pm$10 (1/5) \\
\texttt{nbn-cat-ais} & 0.0145$\pm$0.00059 & 0.0802$\pm$0.0013 & 0.807$\pm$0.022 & oom \\
\texttt{nbn-cat-avi} & 0.00697$\pm$0.00031 & 0.0805$\pm$0.0016 & 0.782$\pm$0.022 & \textbf{6.72$\pm$5} \\
\texttt{nbn-cat-lw} & 0.0796$\pm$0.0011 & 0.183$\pm$0.015 & 0.884$\pm$0.027 & 10.1$\pm$0.052 \\
\texttt{nbn-cat-ve} & 0.0201$\pm$0.00021 & \textbf{0.047$\pm$0.025} & \textbf{0.189$\pm$0.038} & 0.404$\pm$0.051 (2/5) \\
\bottomrule
\end{tabular}}
\end{minipage}\par\bigskip

\noindent\begin{minipage}{\textwidth}\centering\footnotesize
\captionof{table}{Scalability, linear-Gaussian networks: W1 (lower is better), both evidence modes, against network size.}\label{tab:app_scalability_continuous_lg_w1_per_node}\par\smallskip
\adjustbox{max width=\textwidth}{\begin{tabular}{lrrrr}
\toprule
Method & $n=10$ & $n=100$ & $n=1000$ & $n=10000$ \\
\midrule
\texttt{pgmpy-lg-predict} & 0.0353$\pm$0.0087 & \textbf{0.0515$\pm$0.0028} & \textbf{0.0455$\pm$0.0089} & 0.0482$\pm$0.013 \\
\texttt{pyro-empirical-importance} & 0.178$\pm$0.032 & 0.3$\pm$0.034 & 0.309$\pm$0.037 & 0.218$\pm$0.031 (1/5) \\
\texttt{nbn-flexcode-lw} & 0.0462$\pm$0.012 & 0.101$\pm$0.026 & timeout & timeout \\
\texttt{nbn-flow-ais} & 0.0494$\pm$0.0095 & timeout & timeout & timeout \\
\texttt{nbn-flow-avi} & 0.0932$\pm$0.073 & timeout & timeout & timeout \\
\texttt{nbn-flow-lw} & 0.0436$\pm$0.0035 & timeout & timeout & timeout \\
\texttt{nbn-kde-lw} & 0.0473$\pm$0.0058 (1/5) & oom & oom & oom \\
\texttt{nbn-knn-lw} & 0.0513$\pm$0.008 & 0.0813$\pm$0.018 & oom & oom \\
\texttt{nbn-lg-ais} & \textbf{0.032$\pm$0.0045} & 0.0594$\pm$0.0045 & 0.0584$\pm$0.019 & \textbf{0.0475$\pm$0.004} \\
\texttt{nbn-lg-avi} & 0.0378$\pm$0.0038 & 0.0578$\pm$0.0091 & 0.0546$\pm$0.014 & 0.0478$\pm$0.013 \\
\texttt{nbn-lg-lw} & 0.0348$\pm$0.0079 & 0.059$\pm$0.0065 & 0.0556$\pm$0.0081 & 0.0537$\pm$0.0037 \\
\texttt{nbn-mdn-ais} & 0.0488$\pm$0.014 & 0.0636$\pm$0.019 & timeout & timeout \\
\texttt{nbn-mdn-avi} & 0.0395$\pm$0.014 & 0.0585$\pm$0.016 & timeout & timeout \\
\texttt{nbn-mdn-lw} & 0.034$\pm$0.0072 & 0.0539$\pm$0.0047 & timeout & timeout \\
\bottomrule
\end{tabular}}
\end{minipage}\par\bigskip

\noindent\begin{minipage}{\textwidth}\centering\footnotesize
\captionof{table}{Scalability, linear-Gaussian networks: W1 (lower is better), full evidence, against network size.}\label{tab:app_scalability_continuous_lg_w1_per_node_full}\par\smallskip
\adjustbox{max width=\textwidth}{\begin{tabular}{lrrrr}
\toprule
Method & $n=10$ & $n=100$ & $n=1000$ & $n=10000$ \\
\midrule
\texttt{pgmpy-lg-predict} & 0.0303$\pm$0.013 & 0.0553$\pm$0.0076 & \textbf{0.0378$\pm$0.0054} & 0.0466$\pm$0.0056 \\
\texttt{pyro-empirical-importance} & 0.138$\pm$0.05 & 0.33$\pm$0.051 & 0.338$\pm$0.038 & 0.226$\pm$0.022 (1/5) \\
\texttt{nbn-flexcode-lw} & 0.0524$\pm$0.019 & 0.113$\pm$0.024 & timeout & timeout \\
\texttt{nbn-flow-ais} & 0.0435$\pm$0.013 & timeout & timeout & timeout \\
\texttt{nbn-flow-avi} & 0.104$\pm$0.084 & timeout & timeout & timeout \\
\texttt{nbn-flow-lw} & 0.0467$\pm$0.0098 & timeout & timeout & timeout \\
\texttt{nbn-kde-lw} & 0.0656$\pm$0.011 (1/5) & oom & oom & oom \\
\texttt{nbn-knn-lw} & 0.0557$\pm$0.017 & 0.102$\pm$0.033 & oom & oom \\
\texttt{nbn-lg-ais} & \textbf{0.028$\pm$0.0047} & 0.0637$\pm$0.012 & 0.0726$\pm$0.024 & \textbf{0.0454$\pm$0.0038} \\
\texttt{nbn-lg-avi} & 0.0331$\pm$0.0087 & 0.0654$\pm$0.0037 & 0.0603$\pm$0.013 & 0.0493$\pm$0.0077 \\
\texttt{nbn-lg-lw} & 0.0291$\pm$0.0077 & \textbf{0.0526$\pm$0.0014} & 0.0652$\pm$0.012 & 0.056$\pm$0.0094 \\
\texttt{nbn-mdn-ais} & 0.0399$\pm$0.0057 & 0.0694$\pm$0.019 & timeout & timeout \\
\texttt{nbn-mdn-avi} & 0.035$\pm$0.011 & 0.0558$\pm$0.018 & timeout & timeout \\
\texttt{nbn-mdn-lw} & 0.0326$\pm$0.0084 & 0.0534$\pm$0.0077 & timeout & timeout \\
\bottomrule
\end{tabular}}
\end{minipage}\par\bigskip

\noindent\begin{minipage}{\textwidth}\centering\footnotesize
\captionof{table}{Scalability, linear-Gaussian networks: time per query in seconds, both evidence modes, against network size.}\label{tab:app_scalability_continuous_lg_query_time}\par\smallskip
\adjustbox{max width=\textwidth}{\begin{tabular}{lrrrr}
\toprule
Method & $n=10$ & $n=100$ & $n=1000$ & $n=10000$ \\
\midrule
\texttt{pgmpy-lg-predict} & \textbf{0.000719$\pm$1.3e-05} & \textbf{0.00648$\pm$0.0014} & \textbf{0.0803$\pm$0.0019} & 17.3$\pm$0.16 \\
\texttt{pyro-empirical-importance} & 0.237$\pm$0.0036 & 2.86$\pm$0.11 & 30$\pm$0.29 & 324$\pm$5.4 (1/5) \\
\texttt{nbn-flexcode-lw} & 0.0146$\pm$0.0024 & 0.135$\pm$0.002 & timeout & timeout \\
\texttt{nbn-flow-ais} & 0.0989$\pm$0.0055 & timeout & timeout & timeout \\
\texttt{nbn-flow-avi} & 0.0793$\pm$0.00092 & timeout & timeout & timeout \\
\texttt{nbn-flow-lw} & 0.0728$\pm$0.00074 & timeout & timeout & timeout \\
\texttt{nbn-kde-lw} & 0.0893$\pm$0.011 (1/5) & oom & oom & oom \\
\texttt{nbn-knn-lw} & 0.0457$\pm$0.0091 & 0.788$\pm$0.011 & oom & oom \\
\texttt{nbn-lg-ais} & 0.0132$\pm$0.00029 & 0.202$\pm$0.015 & 0.264$\pm$0.01 & \textbf{2.56$\pm$0.0042} \\
\texttt{nbn-lg-avi} & 0.00272$\pm$4.5e-05 & 0.025$\pm$0.00018 & 0.26$\pm$0.0051 & 2.64$\pm$0.04 \\
\texttt{nbn-lg-lw} & 0.0103$\pm$0.0012 & 0.0357$\pm$0.00072 & 0.26$\pm$0.0022 & 2.57$\pm$0.014 \\
\texttt{nbn-mdn-ais} & 0.0272$\pm$0.00096 & 0.519$\pm$0.039 & timeout & timeout \\
\texttt{nbn-mdn-avi} & 0.00938$\pm$0.00079 & 0.126$\pm$0.0073 & timeout & timeout \\
\texttt{nbn-mdn-lw} & 0.0172$\pm$0.0024 & 0.134$\pm$0.0027 & timeout & timeout \\
\bottomrule
\end{tabular}}
\end{minipage}\par\bigskip

\noindent\begin{minipage}{\textwidth}\centering\footnotesize
\captionof{table}{Scalability, linear-Gaussian networks: time per query in seconds, full evidence, against network size.}\label{tab:app_scalability_continuous_lg_query_time_full}\par\smallskip
\adjustbox{max width=\textwidth}{\begin{tabular}{lrrrr}
\toprule
Method & $n=10$ & $n=100$ & $n=1000$ & $n=10000$ \\
\midrule
\texttt{pgmpy-lg-predict} & \textbf{0.000936$\pm$1.7e-05} & \textbf{0.0071$\pm$0.0018} & \textbf{0.0812$\pm$0.002} & 17.5$\pm$0.27 \\
\texttt{pyro-empirical-importance} & 0.228$\pm$0.0047 & 2.88$\pm$0.11 & 30$\pm$0.26 & 324$\pm$5.5 (1/5) \\
\texttt{nbn-flexcode-lw} & 0.0187$\pm$0.0026 & 0.143$\pm$0.0028 & timeout & timeout \\
\texttt{nbn-flow-ais} & 0.0987$\pm$0.0057 & timeout & timeout & timeout \\
\texttt{nbn-flow-avi} & 0.0816$\pm$0.0016 & timeout & timeout & timeout \\
\texttt{nbn-flow-lw} & 0.075$\pm$0.0012 & timeout & timeout & timeout \\
\texttt{nbn-kde-lw} & 0.109$\pm$0.0095 (1/5) & oom & oom & oom \\
\texttt{nbn-knn-lw} & 0.0568$\pm$0.0085 & 0.798$\pm$0.0098 & oom & oom \\
\texttt{nbn-lg-ais} & 0.013$\pm$0.00013 & 0.204$\pm$0.016 & 0.268$\pm$0.012 & \textbf{2.57$\pm$0.0098} \\
\texttt{nbn-lg-avi} & 0.00323$\pm$6.9e-05 & 0.0259$\pm$0.00027 & 0.262$\pm$0.005 & 2.61$\pm$0.059 \\
\texttt{nbn-lg-lw} & 0.018$\pm$0.0025 & 0.0457$\pm$0.0015 & 0.271$\pm$0.0031 & 2.59$\pm$0.025 \\
\texttt{nbn-mdn-ais} & 0.0279$\pm$0.0013 & 0.519$\pm$0.04 & timeout & timeout \\
\texttt{nbn-mdn-avi} & 0.00972$\pm$0.00074 & 0.127$\pm$0.0072 & timeout & timeout \\
\texttt{nbn-mdn-lw} & 0.024$\pm$0.0045 & 0.143$\pm$0.0048 & timeout & timeout \\
\bottomrule
\end{tabular}}
\end{minipage}\par\bigskip

\noindent\begin{minipage}{\textwidth}\centering\footnotesize
\captionof{table}{Scalability, non-Gaussian networks: W1 (lower is better), both evidence modes, against network size.}\label{tab:app_scalability_continuous_nongauss_w1_per_node}\par\smallskip
\adjustbox{max width=\textwidth}{\begin{tabular}{lrrrr}
\toprule
Method & $n=10$ & $n=100$ & $n=1000$ & $n=10000$ \\
\midrule
\texttt{pgmpy-lg-predict} & 0.111$\pm$0.054 & 0.136$\pm$0.049 & \textbf{0.0891$\pm$0.014} & 0.101$\pm$0.048 \\
\texttt{pyro-empirical-importance} & 0.228$\pm$0.076 & 0.22$\pm$0.032 & 0.203$\pm$0.029 & 0.24$\pm$0.046 (1/5) \\
\texttt{nbn-flexcode-lw} & 0.0687$\pm$0.029 & 0.0981$\pm$0.034 & timeout & timeout \\
\texttt{nbn-flow-ais} & 0.054$\pm$0.013 & 0.0694$\pm$0.0084 (3/5) & timeout & timeout \\
\texttt{nbn-flow-avi} & 0.0613$\pm$0.006 & 0.0682$\pm$0.0039 (3/5) & timeout & timeout \\
\texttt{nbn-flow-lw} & 0.0627$\pm$0.016 & 0.0576$\pm$0.0028 (3/5) & timeout & timeout \\
\texttt{nbn-kde-lw} & 0.0478$\pm$0.012 (1/5) & oom & oom & oom \\
\texttt{nbn-knn-lw} & 0.0691$\pm$0.019 & \textbf{0.0783$\pm$0.013} & oom & oom \\
\texttt{nbn-lg-ais} & 0.11$\pm$0.048 & 0.131$\pm$0.043 & 0.0899$\pm$0.018 & \textbf{0.0908$\pm$0.056} \\
\texttt{nbn-lg-avi} & 0.115$\pm$0.057 & 0.135$\pm$0.044 & 0.0942$\pm$0.0087 & 0.102$\pm$0.06 \\
\texttt{nbn-lg-lw} & 0.113$\pm$0.06 & 0.13$\pm$0.047 & 0.0934$\pm$0.0059 & 0.0999$\pm$0.05 \\
\texttt{nbn-mdn-ais} & \textbf{0.0458$\pm$0.006} & 0.0564$\pm$0.016 (1/5) & timeout & timeout \\
\texttt{nbn-mdn-avi} & 0.0487$\pm$0.0071 & 0.0453$\pm$0.006 (1/5) & timeout & timeout \\
\texttt{nbn-mdn-lw} & 0.0484$\pm$0.0092 & 0.0453$\pm$0.004 (1/5) & timeout & timeout \\
\bottomrule
\end{tabular}}
\end{minipage}\par\bigskip

\noindent\begin{minipage}{\textwidth}\centering\footnotesize
\captionof{table}{Scalability, non-Gaussian networks: W1 (lower is better), full evidence, against network size.}\label{tab:app_scalability_continuous_nongauss_w1_per_node_full}\par\smallskip
\adjustbox{max width=\textwidth}{\begin{tabular}{lrrrr}
\toprule
Method & $n=10$ & $n=100$ & $n=1000$ & $n=10000$ \\
\midrule
\texttt{pgmpy-lg-predict} & 0.115$\pm$0.061 & 0.142$\pm$0.047 & \textbf{0.083$\pm$0.0088} & 0.101$\pm$0.044 \\
\texttt{pyro-empirical-importance} & 0.232$\pm$0.059 & 0.193$\pm$0.017 & 0.175$\pm$0.024 & 0.209$\pm$0.036 (1/5) \\
\texttt{nbn-flexcode-lw} & 0.0714$\pm$0.026 & 0.0956$\pm$0.044 & timeout & timeout \\
\texttt{nbn-flow-ais} & 0.0517$\pm$0.0082 & 0.0853$\pm$0.012 (3/5) & timeout & timeout \\
\texttt{nbn-flow-avi} & 0.0578$\pm$0.009 & 0.0729$\pm$0.0081 (3/5) & timeout & timeout \\
\texttt{nbn-flow-lw} & 0.0596$\pm$0.011 & 0.0641$\pm$0.0042 (3/5) & timeout & timeout \\
\texttt{nbn-kde-lw} & 0.0516$\pm$0.0076 (1/5) & oom & oom & oom \\
\texttt{nbn-knn-lw} & 0.0755$\pm$0.019 & \textbf{0.0779$\pm$0.009} & oom & oom \\
\texttt{nbn-lg-ais} & 0.123$\pm$0.054 & 0.138$\pm$0.043 & 0.0864$\pm$0.018 & \textbf{0.0944$\pm$0.058} \\
\texttt{nbn-lg-avi} & 0.129$\pm$0.065 & 0.146$\pm$0.044 & 0.0923$\pm$0.0097 & 0.106$\pm$0.053 \\
\texttt{nbn-lg-lw} & 0.125$\pm$0.064 & 0.137$\pm$0.045 & 0.0917$\pm$0.0099 & 0.0992$\pm$0.041 \\
\texttt{nbn-mdn-ais} & \textbf{0.049$\pm$0.0048} & 0.0446$\pm$0.014 (1/5) & timeout & timeout \\
\texttt{nbn-mdn-avi} & 0.0575$\pm$0.013 & 0.0487$\pm$0.0053 (1/5) & timeout & timeout \\
\texttt{nbn-mdn-lw} & 0.0504$\pm$0.012 & 0.0414$\pm$0.0046 (1/5) & timeout & timeout \\
\bottomrule
\end{tabular}}
\end{minipage}\par\bigskip

\noindent\begin{minipage}{\textwidth}\centering\footnotesize
\captionof{table}{Scalability, non-Gaussian networks: time per query in seconds, both evidence modes, against network size.}\label{tab:app_scalability_continuous_nongauss_query_time}\par\smallskip
\adjustbox{max width=\textwidth}{\begin{tabular}{lrrrr}
\toprule
Method & $n=10$ & $n=100$ & $n=1000$ & $n=10000$ \\
\midrule
\texttt{pgmpy-lg-predict} & \textbf{0.00076$\pm$4.5e-05} & \textbf{0.00659$\pm$0.0018} & \textbf{0.0954$\pm$0.01} & 18.6$\pm$0.13 \\
\texttt{pyro-empirical-importance} & 0.258$\pm$0.0039 & 2.97$\pm$0.031 & 33.1$\pm$0.39 & 334$\pm$5.7 (1/5) \\
\texttt{nbn-flexcode-lw} & 0.0158$\pm$0.0023 & 0.154$\pm$0.0053 & timeout & timeout \\
\texttt{nbn-flow-ais} & 0.106$\pm$0.00086 & 0.835$\pm$0.0039 (3/5) & timeout & timeout \\
\texttt{nbn-flow-avi} & 0.0806$\pm$0.0033 & 0.895$\pm$0.0087 (3/5) & timeout & timeout \\
\texttt{nbn-flow-lw} & 0.0818$\pm$0.0044 & 0.87$\pm$0.0025 (3/5) & timeout & timeout \\
\texttt{nbn-kde-lw} & 0.107$\pm$0.014 (1/5) & oom & oom & oom \\
\texttt{nbn-knn-lw} & 0.0584$\pm$0.011 & 1.03$\pm$0.091 & oom & oom \\
\texttt{nbn-lg-ais} & 0.0232$\pm$0.00029 & 0.292$\pm$0.064 & 2.66$\pm$1.8 & 2.69$\pm$0.097 \\
\texttt{nbn-lg-avi} & 0.00319$\pm$0.00029 & 0.0257$\pm$0.00028 & 0.26$\pm$0.0053 & \textbf{2.67$\pm$0.033} \\
\texttt{nbn-lg-lw} & 0.0126$\pm$0.0018 & 0.0387$\pm$0.0016 & 0.277$\pm$0.0039 & 2.87$\pm$0.11 \\
\texttt{nbn-mdn-ais} & 0.045$\pm$0.0018 & 0.602$\pm$0.057 (1/5) & timeout & timeout \\
\texttt{nbn-mdn-avi} & 0.0117$\pm$0.00058 & 0.135$\pm$0.013 (1/5) & timeout & timeout \\
\texttt{nbn-mdn-lw} & 0.0198$\pm$0.0012 & 0.144$\pm$0.0046 (1/5) & timeout & timeout \\
\bottomrule
\end{tabular}}
\end{minipage}\par\bigskip

\noindent\begin{minipage}{\textwidth}\centering\footnotesize
\captionof{table}{Scalability, non-Gaussian networks: time per query in seconds, full evidence, against network size.}\label{tab:app_scalability_continuous_nongauss_query_time_full}\par\smallskip
\adjustbox{max width=\textwidth}{\begin{tabular}{lrrrr}
\toprule
Method & $n=10$ & $n=100$ & $n=1000$ & $n=10000$ \\
\midrule
\texttt{pgmpy-lg-predict} & \textbf{0.000981$\pm$8.2e-05} & \textbf{0.00693$\pm$0.0019} & \textbf{0.0935$\pm$0.0078} & 18.5$\pm$0.2 \\
\texttt{pyro-empirical-importance} & 0.248$\pm$0.0022 & 2.93$\pm$0.063 & 32.8$\pm$0.026 & 334$\pm$5.7 (1/5) \\
\texttt{nbn-flexcode-lw} & 0.022$\pm$0.0039 & 0.164$\pm$0.0052 & timeout & timeout \\
\texttt{nbn-flow-ais} & 0.107$\pm$0.00053 & 0.831$\pm$0.0018 (3/5) & timeout & timeout \\
\texttt{nbn-flow-avi} & 0.083$\pm$0.0041 & 0.901$\pm$0.011 (3/5) & timeout & timeout \\
\texttt{nbn-flow-lw} & 0.0847$\pm$0.0047 & 0.873$\pm$0.0048 (3/5) & timeout & timeout \\
\texttt{nbn-kde-lw} & 0.127$\pm$0.014 (1/5) & oom & oom & oom \\
\texttt{nbn-knn-lw} & 0.0707$\pm$0.0091 & 1.02$\pm$0.089 & oom & oom \\
\texttt{nbn-lg-ais} & 0.0236$\pm$0.00015 & 0.291$\pm$0.064 & 2.67$\pm$1.8 & \textbf{2.68$\pm$0.024} \\
\texttt{nbn-lg-avi} & 0.00389$\pm$0.00034 & 0.027$\pm$0.00029 & 0.264$\pm$0.011 & 2.7$\pm$0.059 \\
\texttt{nbn-lg-lw} & 0.0218$\pm$0.0031 & 0.0521$\pm$0.0023 & 0.287$\pm$0.0095 & 2.94$\pm$0.053 \\
\texttt{nbn-mdn-ais} & 0.0437$\pm$0.0018 & 0.613$\pm$0.066 (1/5) & timeout & timeout \\
\texttt{nbn-mdn-avi} & 0.0126$\pm$0.00052 & 0.137$\pm$0.014 (1/5) & timeout & timeout \\
\texttt{nbn-mdn-lw} & 0.0269$\pm$0.0014 & 0.154$\pm$0.0071 (1/5) & timeout & timeout \\
\bottomrule
\end{tabular}}
\end{minipage}\par\bigskip

\clearpage

\subsection{Batched queries}\label{app:speed}

\noindent\begin{minipage}{\textwidth}\centering\footnotesize
\includegraphics[width=\textwidth]{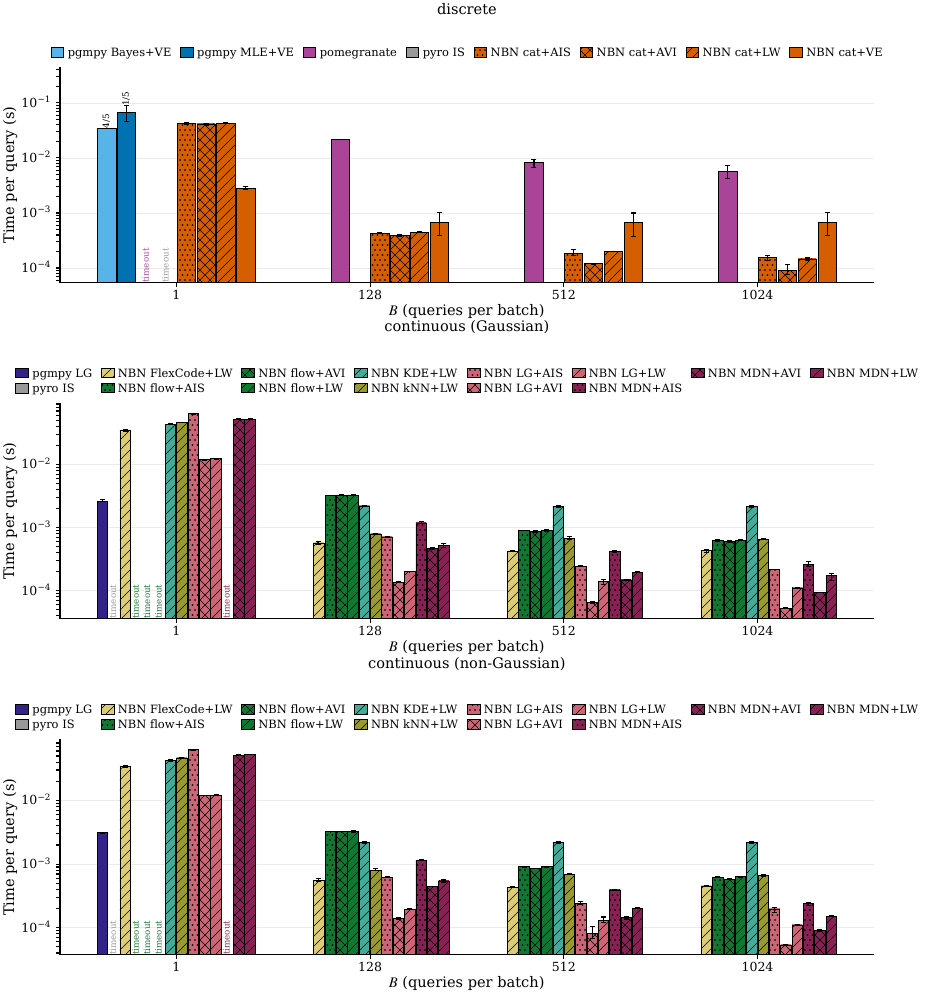}
\captionof{figure}{Batched queries, all methods: time per query (s), full evidence against query batch size $B$.}\label{fig:app_speed_query_time_full}
\end{minipage}\par\bigskip

\noindent\begin{minipage}{\textwidth}\centering\footnotesize
\includegraphics[width=\textwidth]{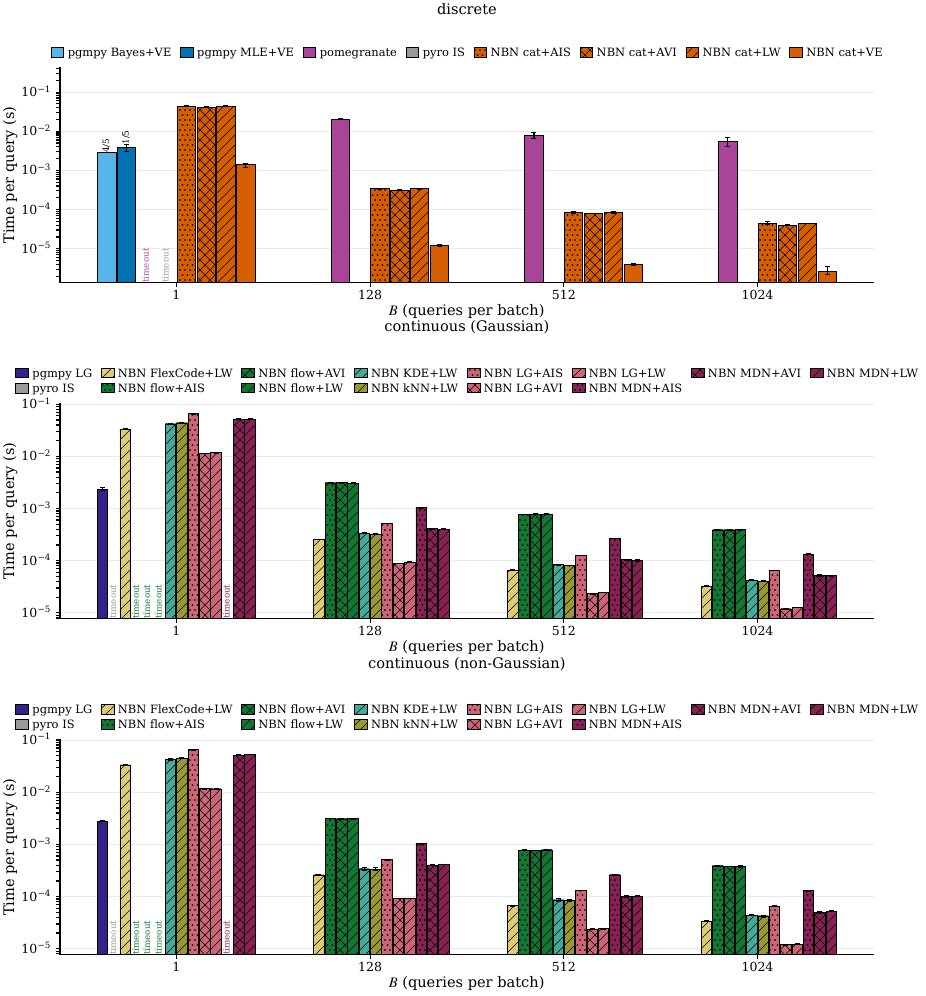}
\captionof{figure}{Batched queries, all methods: time per query (s), empty evidence against query batch size $B$.}\label{fig:app_speed_query_time_empty}
\end{minipage}\par\bigskip

\noindent\begin{minipage}{\textwidth}\centering\footnotesize
\captionof{table}{Batched queries, discrete networks: time per query in seconds, full evidence, against query batch size $B$.}\label{tab:app_speed_discrete_query_time_full}\par\smallskip
\adjustbox{max width=\textwidth}{\begin{tabular}{lrrrr}
\toprule
Method & $B=1$ & $B=128$ & $B=512$ & $B=1024$ \\
\midrule
\texttt{pgmpy-bayes-ve} & 0.0346$\pm$0 (4/5) & -- & -- & -- \\
\texttt{pgmpy-mle-ve} & 0.0671$\pm$0.025 (1/5) & -- & -- & -- \\
\texttt{pomegranate-discrete-ve} & timeout & 0.0216$\pm$0.00012 & 0.00829$\pm$0.0014 & 0.0058$\pm$0.0016 \\
\texttt{pyro-empirical-importance} & timeout & -- & -- & -- \\
\texttt{nbn-cat-ais} & 0.0425$\pm$0.0019 & 0.000431$\pm$4.3e-06 & 0.000187$\pm$2.4e-05 & 0.000154$\pm$1.6e-05 \\
\texttt{nbn-cat-avi} & 0.0417$\pm$0.0017 & \textbf{0.000394$\pm$1.2e-05} & \textbf{0.00012$\pm$4.5e-06} & \textbf{9.14e-05$\pm$1.9e-05} \\
\texttt{nbn-cat-lw} & 0.0425$\pm$0.0007 & 0.000446$\pm$9.4e-06 & 0.0002$\pm$1.2e-06 & 0.000146$\pm$6.4e-06 \\
\texttt{nbn-cat-ve} & \textbf{0.00278$\pm$0.00016} & 0.000688$\pm$0.00031 & 0.000681$\pm$0.00031 & 0.000679$\pm$0.00031 \\
\bottomrule
\end{tabular}}
\end{minipage}\par\bigskip

\noindent\begin{minipage}{\textwidth}\centering\footnotesize
\captionof{table}{Batched queries, discrete networks: time per query in seconds, empty evidence, against query batch size $B$.}\label{tab:app_speed_discrete_query_time_empty}\par\smallskip
\adjustbox{max width=\textwidth}{\begin{tabular}{lrrrr}
\toprule
Method & $B=1$ & $B=128$ & $B=512$ & $B=1024$ \\
\midrule
\texttt{pgmpy-bayes-ve} & 0.00281$\pm$0 (4/5) & -- & -- & -- \\
\texttt{pgmpy-mle-ve} & 0.00381$\pm$0.0012 (1/5) & -- & -- & -- \\
\texttt{pomegranate-discrete-ve} & timeout & 0.0206$\pm$0.00026 & 0.00802$\pm$0.0015 & 0.00535$\pm$0.0014 \\
\texttt{pyro-empirical-importance} & timeout & -- & -- & -- \\
\texttt{nbn-cat-ais} & 0.0437$\pm$0.0012 & 0.000341$\pm$9.6e-06 & 8.5e-05$\pm$4.7e-06 & 4.49e-05$\pm$5e-06 \\
\texttt{nbn-cat-avi} & 0.0413$\pm$0.0007 & 0.000316$\pm$9e-06 & 7.87e-05$\pm$1.8e-06 & 4.02e-05$\pm$7.7e-07 \\
\texttt{nbn-cat-lw} & 0.0436$\pm$0.0011 & 0.000343$\pm$9.5e-06 & 8.5e-05$\pm$4.6e-06 & 4.42e-05$\pm$6.7e-07 \\
\texttt{nbn-cat-ve} & \textbf{0.00138$\pm$0.00013} & \textbf{1.25e-05$\pm$5.3e-07} & \textbf{3.94e-06$\pm$1.8e-07} & \textbf{2.66e-06$\pm$6e-07} \\
\bottomrule
\end{tabular}}
\end{minipage}\par\bigskip

\noindent\begin{minipage}{\textwidth}\centering\footnotesize
\captionof{table}{Batched queries, linear-Gaussian networks: time per query in seconds, full evidence, against query batch size $B$.}\label{tab:app_speed_continuous_lg_query_time_full}\par\smallskip
\adjustbox{max width=\textwidth}{\begin{tabular}{lrrrr}
\toprule
Method & $B=1$ & $B=128$ & $B=512$ & $B=1024$ \\
\midrule
\texttt{pgmpy-lg-predict} & \textbf{0.00262$\pm$0.00011} & -- & -- & -- \\
\texttt{pyro-empirical-importance} & timeout & -- & -- & -- \\
\texttt{nbn-flexcode-lw} & 0.0341$\pm$0.00083 & 0.000569$\pm$3.4e-05 & 0.000425$\pm$8.9e-06 & 0.000428$\pm$2.1e-05 \\
\texttt{nbn-flow-ais} & timeout & 0.00324$\pm$2.5e-05 & 0.000894$\pm$9.4e-06 & 0.000625$\pm$1.5e-05 \\
\texttt{nbn-flow-avi} & timeout & 0.00325$\pm$5.4e-05 & 0.000867$\pm$3e-05 & 0.000593$\pm$2e-05 \\
\texttt{nbn-flow-lw} & timeout & 0.00326$\pm$6.9e-05 & 0.000896$\pm$2.2e-05 & 0.000627$\pm$1e-05 \\
\texttt{nbn-kde-lw} & 0.0428$\pm$0.00068 & 0.0022$\pm$4.9e-05 & 0.00217$\pm$6.7e-05 & 0.00217$\pm$5.4e-05 \\
\texttt{nbn-knn-lw} & 0.0459$\pm$0.00036 & 0.000792$\pm$2.1e-05 & 0.000682$\pm$2.6e-05 & 0.000649$\pm$1.7e-05 \\
\texttt{nbn-lg-ais} & 0.0631$\pm$0.00075 & 0.00071$\pm$6.8e-06 & 0.000244$\pm$3.4e-06 & 0.000217$\pm$2.6e-06 \\
\texttt{nbn-lg-avi} & 0.0119$\pm$0.00025 & \textbf{0.000136$\pm$2.1e-06} & \textbf{6.51e-05$\pm$1.8e-06} & \textbf{5.25e-05$\pm$1.2e-06} \\
\texttt{nbn-lg-lw} & 0.0123$\pm$0.0003 & 0.000199$\pm$7.8e-07 & 0.000142$\pm$1.5e-05 & 0.000111$\pm$4.6e-07 \\
\texttt{nbn-mdn-ais} & timeout & 0.00118$\pm$3.7e-05 & 0.000416$\pm$1.6e-05 & 0.000258$\pm$2.3e-05 \\
\texttt{nbn-mdn-avi} & 0.052$\pm$0.0015 & 0.000467$\pm$1.6e-05 & 0.000148$\pm$1.7e-06 & 9.37e-05$\pm$1.5e-06 \\
\texttt{nbn-mdn-lw} & 0.052$\pm$0.0015 & 0.000515$\pm$3.6e-05 & 0.000197$\pm$4.5e-06 & 0.000172$\pm$2.1e-05 \\
\bottomrule
\end{tabular}}
\end{minipage}\par\bigskip

\noindent\begin{minipage}{\textwidth}\centering\footnotesize
\captionof{table}{Batched queries, linear-Gaussian networks: time per query in seconds, empty evidence, against query batch size $B$.}\label{tab:app_speed_continuous_lg_query_time_empty}\par\smallskip
\adjustbox{max width=\textwidth}{\begin{tabular}{lrrrr}
\toprule
Method & $B=1$ & $B=128$ & $B=512$ & $B=1024$ \\
\midrule
\texttt{pgmpy-lg-predict} & \textbf{0.00236$\pm$0.00013} & -- & -- & -- \\
\texttt{pyro-empirical-importance} & timeout & -- & -- & -- \\
\texttt{nbn-flexcode-lw} & 0.033$\pm$0.00071 & 0.000257$\pm$1.5e-06 & 6.6e-05$\pm$5.6e-07 & 3.27e-05$\pm$4.3e-07 \\
\texttt{nbn-flow-ais} & timeout & 0.00309$\pm$3.8e-05 & 0.000773$\pm$1.3e-05 & 0.000388$\pm$6.6e-06 \\
\texttt{nbn-flow-avi} & timeout & 0.00313$\pm$4.1e-05 & 0.000784$\pm$1.6e-05 & 0.00039$\pm$3.6e-06 \\
\texttt{nbn-flow-lw} & timeout & 0.00308$\pm$6.7e-05 & 0.000773$\pm$2e-05 & 0.000393$\pm$1.2e-05 \\
\texttt{nbn-kde-lw} & 0.0419$\pm$0.00065 & 0.000336$\pm$9.2e-06 & 8.42e-05$\pm$1.7e-06 & 4.21e-05$\pm$1.1e-06 \\
\texttt{nbn-knn-lw} & 0.0439$\pm$0.00045 & 0.000324$\pm$4.1e-06 & 8.01e-05$\pm$5.9e-07 & 4.03e-05$\pm$1.1e-06 \\
\texttt{nbn-lg-ais} & 0.0654$\pm$0.0011 & 0.000515$\pm$6.7e-06 & 0.000126$\pm$8.5e-07 & 6.42e-05$\pm$6.5e-07 \\
\texttt{nbn-lg-avi} & 0.0114$\pm$0.00021 & \textbf{8.98e-05$\pm$5.8e-07} & \textbf{2.33e-05$\pm$5.4e-07} & \textbf{1.19e-05$\pm$4.7e-08} \\
\texttt{nbn-lg-lw} & 0.0118$\pm$0.00029 & 9.4e-05$\pm$1.1e-06 & 2.46e-05$\pm$2.8e-07 & 1.26e-05$\pm$1.5e-07 \\
\texttt{nbn-mdn-ais} & timeout & 0.00105$\pm$2.3e-05 & 0.000266$\pm$3.7e-06 & 0.000133$\pm$2.3e-06 \\
\texttt{nbn-mdn-avi} & 0.052$\pm$0.0014 & 0.000414$\pm$1.2e-05 & 0.000104$\pm$1.2e-06 & 5.25e-05$\pm$1.6e-06 \\
\texttt{nbn-mdn-lw} & 0.0517$\pm$0.0019 & 0.000402$\pm$1.3e-05 & 0.000101$\pm$3.3e-06 & 5.12e-05$\pm$1.6e-06 \\
\bottomrule
\end{tabular}}
\end{minipage}\par\bigskip

\noindent\begin{minipage}{\textwidth}\centering\footnotesize
\captionof{table}{Batched queries, non-Gaussian networks: time per query in seconds, full evidence, against query batch size $B$.}\label{tab:app_speed_continuous_nongauss_query_time_full}\par\smallskip
\adjustbox{max width=\textwidth}{\begin{tabular}{lrrrr}
\toprule
Method & $B=1$ & $B=128$ & $B=512$ & $B=1024$ \\
\midrule
\texttt{pgmpy-lg-predict} & \textbf{0.0031$\pm$6.8e-05} & -- & -- & -- \\
\texttt{pyro-empirical-importance} & timeout & -- & -- & -- \\
\texttt{nbn-flexcode-lw} & 0.0344$\pm$0.00079 & 0.000556$\pm$3.4e-05 & 0.000432$\pm$9.1e-06 & 0.000452$\pm$9.4e-06 \\
\texttt{nbn-flow-ais} & timeout & 0.00324$\pm$6.6e-06 & 0.000911$\pm$2.3e-05 & 0.00062$\pm$9e-06 \\
\texttt{nbn-flow-avi} & timeout & 0.00324$\pm$4.9e-06 & 0.000846$\pm$1.4e-05 & 0.000579$\pm$9.9e-06 \\
\texttt{nbn-flow-lw} & timeout & 0.00328$\pm$9e-05 & 0.000908$\pm$1.3e-05 & 0.000629$\pm$8.9e-06 \\
\texttt{nbn-kde-lw} & 0.0426$\pm$0.0015 & 0.00222$\pm$5.9e-05 & 0.00219$\pm$6.6e-05 & 0.0022$\pm$7.7e-05 \\
\texttt{nbn-knn-lw} & 0.0466$\pm$0.0011 & 0.0008$\pm$3e-05 & 0.000684$\pm$1.3e-05 & 0.000653$\pm$1.9e-05 \\
\texttt{nbn-lg-ais} & 0.0625$\pm$0.00057 & 0.000625$\pm$1.1e-05 & 0.00024$\pm$1.6e-05 & 0.000197$\pm$1.7e-05 \\
\texttt{nbn-lg-avi} & 0.012$\pm$0.00012 & \textbf{0.000141$\pm$3.7e-06} & \textbf{7.99e-05$\pm$1.9e-05} & \textbf{5.34e-05$\pm$2e-07} \\
\texttt{nbn-lg-lw} & 0.0121$\pm$0.00026 & 0.000196$\pm$1.2e-06 & 0.000129$\pm$1.4e-05 & 0.00011$\pm$1.6e-06 \\
\texttt{nbn-mdn-ais} & timeout & 0.00115$\pm$1.3e-05 & 0.000392$\pm$9.5e-06 & 0.000239$\pm$1.1e-05 \\
\texttt{nbn-mdn-avi} & 0.0515$\pm$0.00044 & 0.000444$\pm$5.3e-06 & 0.000144$\pm$5.8e-06 & 9.06e-05$\pm$1.9e-06 \\
\texttt{nbn-mdn-lw} & 0.0523$\pm$0.00056 & 0.000541$\pm$2.6e-05 & 0.000203$\pm$1.3e-06 & 0.000151$\pm$2.3e-06 \\
\bottomrule
\end{tabular}}
\end{minipage}\par\bigskip

\noindent\begin{minipage}{\textwidth}\centering\footnotesize
\captionof{table}{Batched queries, non-Gaussian networks: time per query in seconds, empty evidence, against query batch size $B$.}\label{tab:app_speed_continuous_nongauss_query_time_empty}\par\smallskip
\adjustbox{max width=\textwidth}{\begin{tabular}{lrrrr}
\toprule
Method & $B=1$ & $B=128$ & $B=512$ & $B=1024$ \\
\midrule
\texttt{pgmpy-lg-predict} & \textbf{0.00278$\pm$4.3e-05} & -- & -- & -- \\
\texttt{pyro-empirical-importance} & timeout & -- & -- & -- \\
\texttt{nbn-flexcode-lw} & 0.0332$\pm$0.00053 & 0.000261$\pm$4.1e-06 & 6.67e-05$\pm$2.2e-06 & 3.38e-05$\pm$6.7e-07 \\
\texttt{nbn-flow-ais} & timeout & 0.0031$\pm$9.2e-06 & 0.000775$\pm$9.2e-06 & 0.000387$\pm$4.8e-07 \\
\texttt{nbn-flow-avi} & timeout & 0.00307$\pm$8.2e-06 & 0.000775$\pm$4e-06 & 0.000383$\pm$1e-06 \\
\texttt{nbn-flow-lw} & timeout & 0.00311$\pm$8.4e-05 & 0.000785$\pm$1.7e-05 & 0.000385$\pm$1.5e-05 \\
\texttt{nbn-kde-lw} & 0.0421$\pm$0.0015 & 0.000338$\pm$2.4e-05 & 8.64e-05$\pm$5.3e-06 & 4.41e-05$\pm$1.2e-06 \\
\texttt{nbn-knn-lw} & 0.0446$\pm$0.00078 & 0.000336$\pm$2.2e-05 & 8.36e-05$\pm$4.3e-06 & 4.26e-05$\pm$1.9e-06 \\
\texttt{nbn-lg-ais} & 0.0647$\pm$0.00064 & 0.000512$\pm$1e-05 & 0.00013$\pm$5e-07 & 6.52e-05$\pm$1.1e-06 \\
\texttt{nbn-lg-avi} & 0.0116$\pm$0.00025 & 9.3e-05$\pm$1.1e-06 & \textbf{2.38e-05$\pm$6.9e-07} & \textbf{1.2e-05$\pm$1.8e-07} \\
\texttt{nbn-lg-lw} & 0.0116$\pm$0.0002 & \textbf{9.22e-05$\pm$1.7e-06} & 2.41e-05$\pm$5e-07 & 1.23e-05$\pm$2.8e-07 \\
\texttt{nbn-mdn-ais} & timeout & 0.00104$\pm$2.1e-05 & 0.00026$\pm$9.3e-07 & 0.00013$\pm$2.7e-06 \\
\texttt{nbn-mdn-avi} & 0.0515$\pm$0.00046 & 0.000397$\pm$5.5e-06 & 0.000102$\pm$3.6e-06 & 5.05e-05$\pm$2e-06 \\
\texttt{nbn-mdn-lw} & 0.0521$\pm$0.00017 & 0.000409$\pm$4.6e-06 & 0.000103$\pm$2.6e-06 & 5.22e-05$\pm$5.6e-07 \\
\bottomrule
\end{tabular}}
\end{minipage}\par\bigskip

\clearpage

\subsection{Real networks}\label{app:bnlearn}

\noindent\begin{minipage}{\textwidth}\centering\footnotesize
\includegraphics[width=\textwidth]{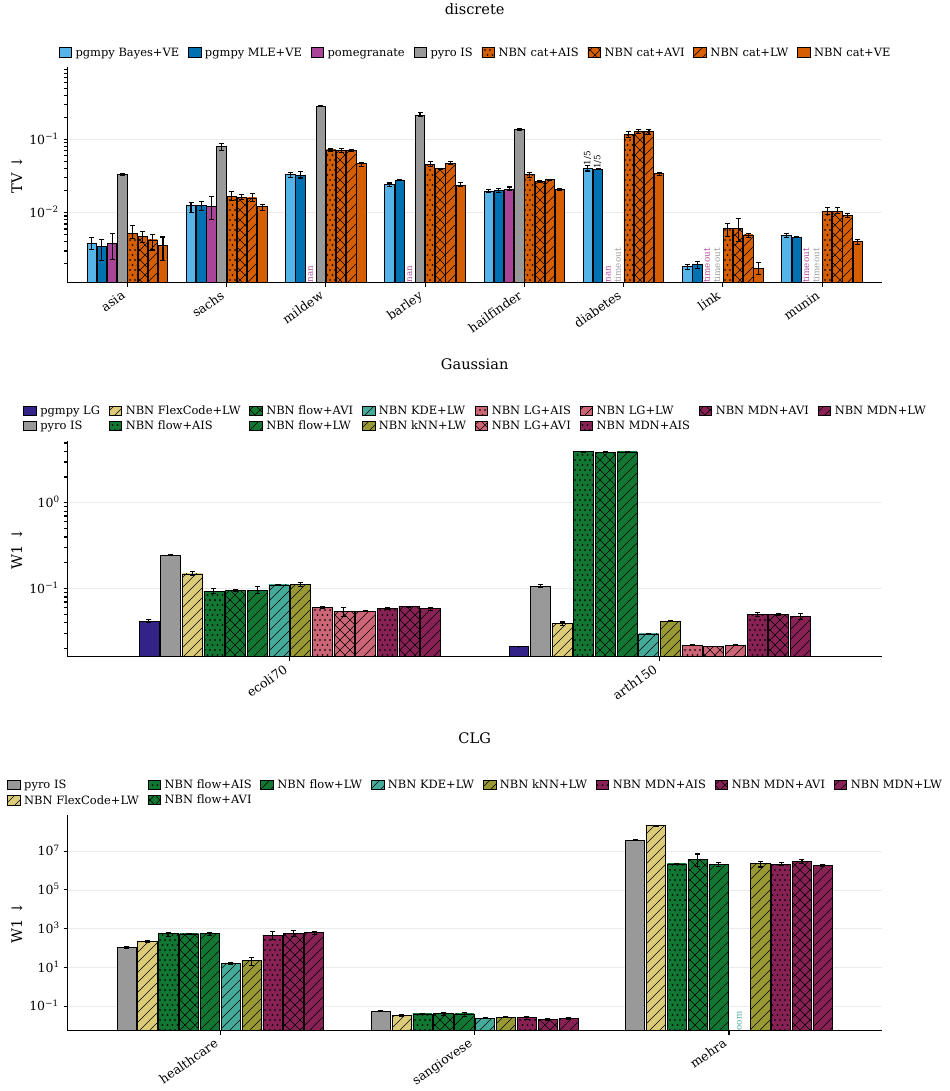}
\captionof{figure}{Real networks, all methods: accuracy (TV on discrete targets, W1 on continuous ones) against network.}\label{fig:app_bnlearn_accuracy}
\end{minipage}\par\bigskip

\noindent\begin{minipage}{\textwidth}\centering\footnotesize
\includegraphics[width=\textwidth]{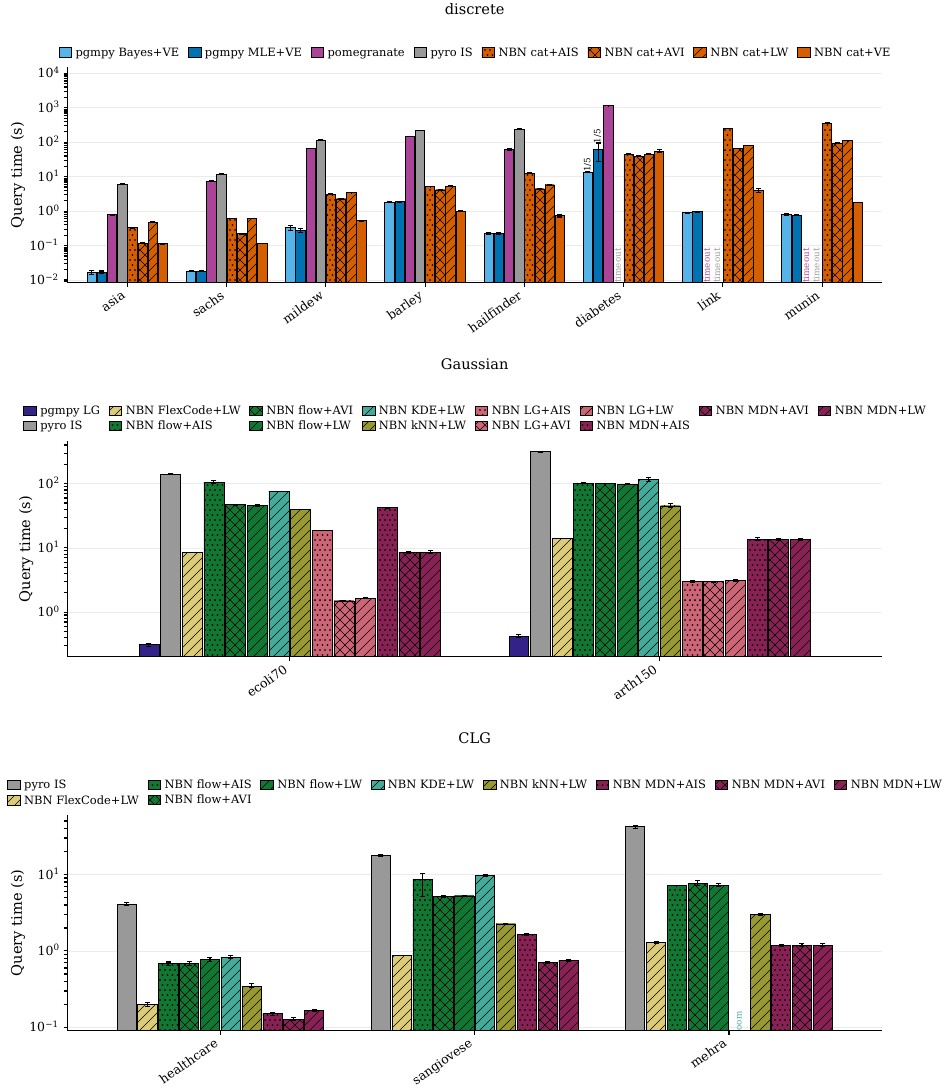}
\captionof{figure}{Real networks, all methods: total query time (s) against network.}\label{fig:app_bnlearn_total_query_time}
\end{minipage}\par\bigskip

\noindent\begin{minipage}{\textwidth}\centering\footnotesize
\captionof{table}{Real networks, discrete networks: TV (lower is better), against network.}\label{tab:app_bnlearn_discrete_tv_per_node}\par\smallskip
\adjustbox{max width=\textwidth}{\begin{tabular}{lrrrrrrrr}
\toprule
Method & asia & sachs & mildew & barley & hailfinder & diabetes & link & munin \\
\midrule
\texttt{pgmpy-bayes-ve} & 0.00383$\pm$0.00074 & 0.0124$\pm$0.0018 & 0.033$\pm$0.0024 & 0.0241$\pm$0.0013 & \textbf{0.0198$\pm$0.00092} & 0.0407$\pm$0.0049 (1/5) & 0.00183$\pm$0.00014 & 0.00488$\pm$0.00033 \\
\texttt{pgmpy-mle-ve} & \textbf{0.00343$\pm$0.0011} & 0.0128$\pm$0.0016 & \textbf{0.0325$\pm$0.0032} & 0.0275$\pm$0.00051 & 0.0203$\pm$0.0015 & 0.0393$\pm$0.0011 (1/5) & 0.00197$\pm$0.00024 & 0.00464$\pm$9.7e-05 \\
\texttt{pomegranate-discrete-ve} & 0.0038$\pm$0.0014 & 0.0122$\pm$0.0044 & nan & nan & 0.0208$\pm$0.0012 & nan & timeout & timeout \\
\texttt{pyro-empirical-importance} & 0.0336$\pm$0.00069 & 0.0818$\pm$0.0091 & 0.287$\pm$0.0041 & 0.215$\pm$0.014 & 0.136$\pm$0.0055 & timeout & timeout & timeout \\
\texttt{nbn-cat-ais} & 0.00524$\pm$0.0011 & 0.0168$\pm$0.0024 & 0.072$\pm$0.0027 & 0.0459$\pm$0.0036 & 0.0333$\pm$0.0025 & 0.118$\pm$0.011 & 0.00604$\pm$0.0011 & 0.0105$\pm$0.0012 \\
\texttt{nbn-cat-avi} & 0.00478$\pm$0.00076 & 0.0162$\pm$0.0015 & 0.0703$\pm$0.0048 & 0.0401$\pm$0.00064 & 0.0266$\pm$0.00054 & 0.128$\pm$0.008 & 0.00612$\pm$0.0021 & 0.0105$\pm$0.00093 \\
\texttt{nbn-cat-lw} & 0.00424$\pm$0.00096 & 0.0159$\pm$0.0022 & 0.0709$\pm$0.0016 & 0.0474$\pm$0.0019 & 0.0281$\pm$0.00052 & 0.129$\pm$0.0099 & 0.00496$\pm$0.00032 & 0.00929$\pm$0.00056 \\
\texttt{nbn-cat-ve} & 0.00354$\pm$0.0012 & \textbf{0.012$\pm$0.0011} & 0.0465$\pm$0.0027 & \textbf{0.024$\pm$0.0013} & 0.0206$\pm$0.00087 & \textbf{0.0343$\pm$0.0016} & \textbf{0.00174$\pm$0.00033} & \textbf{0.00401$\pm$0.0003} \\
\bottomrule
\end{tabular}}
\end{minipage}\par\bigskip

\noindent\begin{minipage}{\textwidth}\centering\footnotesize
\captionof{table}{Real networks, discrete networks: total query time in seconds, against network.}\label{tab:app_bnlearn_discrete_total_query_time}\par\smallskip
\adjustbox{max width=\textwidth}{\begin{tabular}{lrrrrrrrr}
\toprule
Method & asia & sachs & mildew & barley & hailfinder & diabetes & link & munin \\
\midrule
\texttt{pgmpy-bayes-ve} & \textbf{0.0165$\pm$0.0018} & \textbf{0.0183$\pm$8.6e-05} & 0.331$\pm$0.05 & 1.82$\pm$0.027 & \textbf{0.229$\pm$0.015} & 13.6$\pm$1.3 (1/5) & \textbf{0.9$\pm$0.05} & 0.804$\pm$0.056 \\
\texttt{pgmpy-mle-ve} & 0.0171$\pm$0.0014 & \textbf{0.0183$\pm$0.00051} & \textbf{0.269$\pm$0.041} & 1.88$\pm$0.1 & 0.232$\pm$0.016 & 60.8$\pm$1.1e+02 (1/5) & 0.965$\pm$0.04 & \textbf{0.763$\pm$0.02} \\
\texttt{pomegranate-discrete-ve} & 0.785$\pm$0.022 & 7.47$\pm$0.036 & 67.5$\pm$0.32 & 146$\pm$0.85 & 60.9$\pm$2.3 & 1.14e+03$\pm$8.4 & timeout & timeout \\
\texttt{pyro-empirical-importance} & 6.15$\pm$0.18 & 11.6$\pm$0.41 & 117$\pm$0.47 & 217$\pm$4.8 & 240$\pm$5 & timeout & timeout & timeout \\
\texttt{nbn-cat-ais} & 0.328$\pm$0.012 & 0.602$\pm$0.011 & 3.12$\pm$0.03 & 5.13$\pm$0.033 & 12.5$\pm$0.58 & 45.2$\pm$0.68 & 255$\pm$0.64 & 360$\pm$1.6 \\
\texttt{nbn-cat-avi} & 0.116$\pm$0.0033 & 0.219$\pm$0.0025 & 2.25$\pm$0.015 & 4.13$\pm$0.02 & 4.39$\pm$0.065 & \textbf{38.7$\pm$1.7} & 65.6$\pm$1.1 & 94.7$\pm$0.55 \\
\texttt{nbn-cat-lw} & 0.482$\pm$0.014 & 0.614$\pm$0.0071 & 3.39$\pm$0.02 & 5.35$\pm$0.015 & 5.74$\pm$0.1 & 44.8$\pm$1.1 & 79.8$\pm$1.1 & 111$\pm$0.29 \\
\texttt{nbn-cat-ve} & 0.113$\pm$0.0021 & 0.115$\pm$0.0013 & 0.531$\pm$0.023 & \textbf{1.01$\pm$0.02} & 0.733$\pm$0.058 & 55.9$\pm$5.2 & 4.09$\pm$0.6 & 1.8$\pm$0.05 \\
\bottomrule
\end{tabular}}
\end{minipage}\par\bigskip

\noindent\begin{minipage}{\textwidth}\centering\footnotesize
\captionof{table}{Real networks, Gaussian networks: W1 (lower is better), against network.}\label{tab:app_bnlearn_continuous_gauss_w1_per_node}\par\smallskip
\adjustbox{max width=\textwidth}{\begin{tabular}{lrr}
\toprule
Method & ecoli70 & arth150 \\
\midrule
\texttt{pgmpy-lg-predict} & \textbf{0.0418$\pm$0.002} & \textbf{0.021$\pm$7.5e-05} \\
\texttt{pyro-empirical-importance} & 0.245$\pm$0.0029 & 0.107$\pm$0.0031 \\
\texttt{nbn-flexcode-lw} & 0.148$\pm$0.0077 & 0.0389$\pm$0.0019 \\
\texttt{nbn-flow-ais} & 0.0923$\pm$0.0073 & 3.94$\pm$0.062 \\
\texttt{nbn-flow-avi} & 0.0942$\pm$0.0025 & 3.91$\pm$0.073 \\
\texttt{nbn-flow-lw} & 0.0942$\pm$0.0087 & 3.92$\pm$0.075 \\
\texttt{nbn-kde-lw} & 0.11$\pm$0.0018 & 0.0297$\pm$0.0002 \\
\texttt{nbn-knn-lw} & 0.111$\pm$0.0052 & 0.0418$\pm$0.00083 \\
\texttt{nbn-lg-ais} & 0.0603$\pm$0.0012 & 0.0218$\pm$0.00034 \\
\texttt{nbn-lg-avi} & 0.0541$\pm$0.0062 & 0.0212$\pm$1.1e-05 \\
\texttt{nbn-lg-lw} & 0.0547$\pm$0.0012 & 0.0219$\pm$0.0003 \\
\texttt{nbn-mdn-ais} & 0.059$\pm$0.002 & 0.0498$\pm$0.0021 \\
\texttt{nbn-mdn-avi} & 0.0615$\pm$0.00074 & 0.0503$\pm$0.0011 \\
\texttt{nbn-mdn-lw} & 0.0586$\pm$0.0029 & 0.0471$\pm$0.004 \\
\bottomrule
\end{tabular}}
\end{minipage}\par\bigskip

\noindent\begin{minipage}{\textwidth}\centering\footnotesize
\captionof{table}{Real networks, Gaussian networks: total query time in seconds, against network.}\label{tab:app_bnlearn_continuous_gauss_total_query_time}\par\smallskip
\adjustbox{max width=\textwidth}{\begin{tabular}{lrr}
\toprule
Method & ecoli70 & arth150 \\
\midrule
\texttt{pgmpy-lg-predict} & \textbf{0.314$\pm$0.02} & \textbf{0.425$\pm$0.021} \\
\texttt{pyro-empirical-importance} & 141$\pm$0.82 & 317$\pm$4.6 \\
\texttt{nbn-flexcode-lw} & 8.58$\pm$0.047 & 13.9$\pm$0.11 \\
\texttt{nbn-flow-ais} & 106$\pm$4.2 & 102$\pm$1.9 \\
\texttt{nbn-flow-avi} & 47$\pm$0.58 & 101$\pm$0.25 \\
\texttt{nbn-flow-lw} & 45.9$\pm$1.1 & 99$\pm$0.95 \\
\texttt{nbn-kde-lw} & 76$\pm$0.66 & 115$\pm$8.3 \\
\texttt{nbn-knn-lw} & 39.7$\pm$0.22 & 45.2$\pm$3.3 \\
\texttt{nbn-lg-ais} & 18.6$\pm$0.26 & 2.98$\pm$0.14 \\
\texttt{nbn-lg-avi} & 1.49$\pm$0.045 & 2.97$\pm$0.05 \\
\texttt{nbn-lg-lw} & 1.64$\pm$0.018 & 3.11$\pm$0.1 \\
\texttt{nbn-mdn-ais} & 42.4$\pm$0.4 & 13.5$\pm$0.85 \\
\texttt{nbn-mdn-avi} & 8.6$\pm$0.23 & 13.5$\pm$0.41 \\
\texttt{nbn-mdn-lw} & 8.55$\pm$0.37 & 13.6$\pm$0.6 \\
\bottomrule
\end{tabular}}
\end{minipage}\par\bigskip

\noindent\begin{minipage}{\textwidth}\centering\footnotesize
\captionof{table}{Real networks, CLG networks: W1 (lower is better), against network.}\label{tab:app_bnlearn_clg_w1_per_node}\par\smallskip
\adjustbox{max width=\textwidth}{\begin{tabular}{lrrr}
\toprule
Method & healthcare & sangiovese & mehra \\
\midrule
\texttt{pyro-empirical-importance} & 106$\pm$15 & 0.0583$\pm$0.0027 & 3.74e+07$\pm$1.9e+06 \\
\texttt{nbn-flexcode-lw} & 221$\pm$21 & 0.0344$\pm$0.0027 & 2.03e+08$\pm$1.9e+07 \\
\texttt{nbn-flow-ais} & 572$\pm$1.3e+02 & 0.0412$\pm$0.004 & 2.21e+06$\pm$1.6e+05 \\
\texttt{nbn-flow-avi} & 543$\pm$39 & 0.042$\pm$0.005 & 3.56e+06$\pm$2.8e+06 \\
\texttt{nbn-flow-lw} & 581$\pm$1.1e+02 & 0.0412$\pm$0.0084 & 1.97e+06$\pm$5.7e+05 \\
\texttt{nbn-kde-lw} & \textbf{16.6$\pm$1} & 0.0252$\pm$0.0017 & oom \\
\texttt{nbn-knn-lw} & 23.4$\pm$11 & 0.028$\pm$0.0019 & 2.41e+06$\pm$7.1e+05 \\
\texttt{nbn-mdn-ais} & 461$\pm$2.3e+02 & 0.0271$\pm$0.0044 & 2.17e+06$\pm$4.1e+05 \\
\texttt{nbn-mdn-avi} & 576$\pm$1.9e+02 & \textbf{0.022$\pm$0.0029} & 2.97e+06$\pm$6.8e+05 \\
\texttt{nbn-mdn-lw} & 625$\pm$90 & 0.0242$\pm$0.0032 & \textbf{1.89e+06$\pm$2e+05} \\
\bottomrule
\end{tabular}}
\end{minipage}\par\bigskip

\noindent\begin{minipage}{\textwidth}\centering\footnotesize
\captionof{table}{Real networks, CLG networks: total query time in seconds, against network.}\label{tab:app_bnlearn_clg_total_query_time}\par\smallskip
\adjustbox{max width=\textwidth}{\begin{tabular}{lrrr}
\toprule
Method & healthcare & sangiovese & mehra \\
\midrule
\texttt{pyro-empirical-importance} & 4.03$\pm$0.19 & 17.8$\pm$0.6 & 42.2$\pm$1.9 \\
\texttt{nbn-flexcode-lw} & 0.199$\pm$0.011 & 0.873$\pm$0.0092 & 1.3$\pm$0.037 \\
\texttt{nbn-flow-ais} & 0.688$\pm$0.032 & 8.52$\pm$2.6 & 7.12$\pm$0.023 \\
\texttt{nbn-flow-avi} & 0.688$\pm$0.046 & 5.16$\pm$0.15 & 7.71$\pm$0.51 \\
\texttt{nbn-flow-lw} & 0.771$\pm$0.048 & 5.24$\pm$0.11 & 7.24$\pm$0.26 \\
\texttt{nbn-kde-lw} & 0.827$\pm$0.037 & 9.78$\pm$0.23 & oom \\
\texttt{nbn-knn-lw} & 0.345$\pm$0.021 & 2.25$\pm$0.052 & 2.98$\pm$0.098 \\
\texttt{nbn-mdn-ais} & 0.152$\pm$0.0077 & 1.67$\pm$0.047 & 1.2$\pm$0.041 \\
\texttt{nbn-mdn-avi} & \textbf{0.126$\pm$0.0056} & \textbf{0.705$\pm$0.02} & \textbf{1.18$\pm$0.055} \\
\texttt{nbn-mdn-lw} & 0.166$\pm$0.0064 & 0.747$\pm$0.023 & 1.19$\pm$0.042 \\
\bottomrule
\end{tabular}}
\end{minipage}\par\bigskip

\clearpage

\subsection{Parameter learning}\label{app:param_learning}

\noindent\begin{minipage}{\textwidth}\centering\footnotesize
\includegraphics[width=\textwidth]{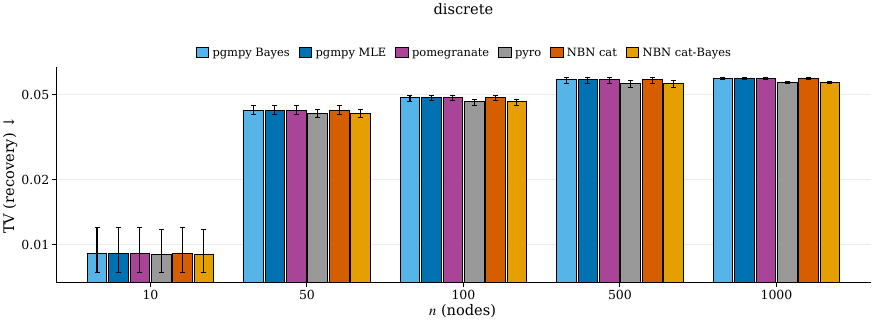}
\captionof{figure}{Parameter learning, all methods: parameter-recovery TV against network size.}\label{fig:app_param_learning_param_recovery_tv}
\end{minipage}\par\bigskip

\noindent\begin{minipage}{\textwidth}\centering\footnotesize
\includegraphics[width=\textwidth]{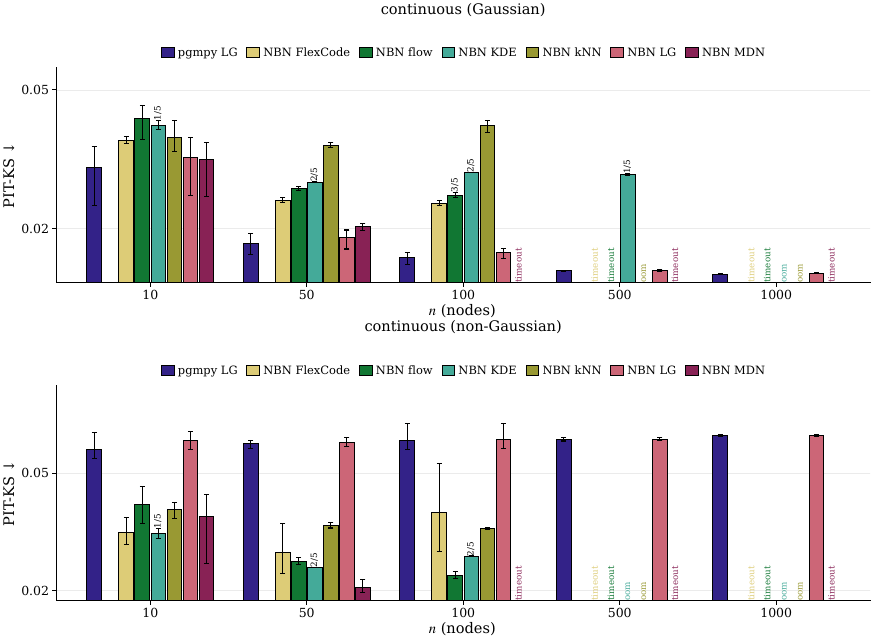}
\captionof{figure}{Parameter learning, all methods: calibration error (PIT-KS) against network size.}\label{fig:app_param_learning_calibration_pit_ks}
\end{minipage}\par\bigskip

\noindent\begin{minipage}{\textwidth}\centering\footnotesize
\includegraphics[width=\textwidth]{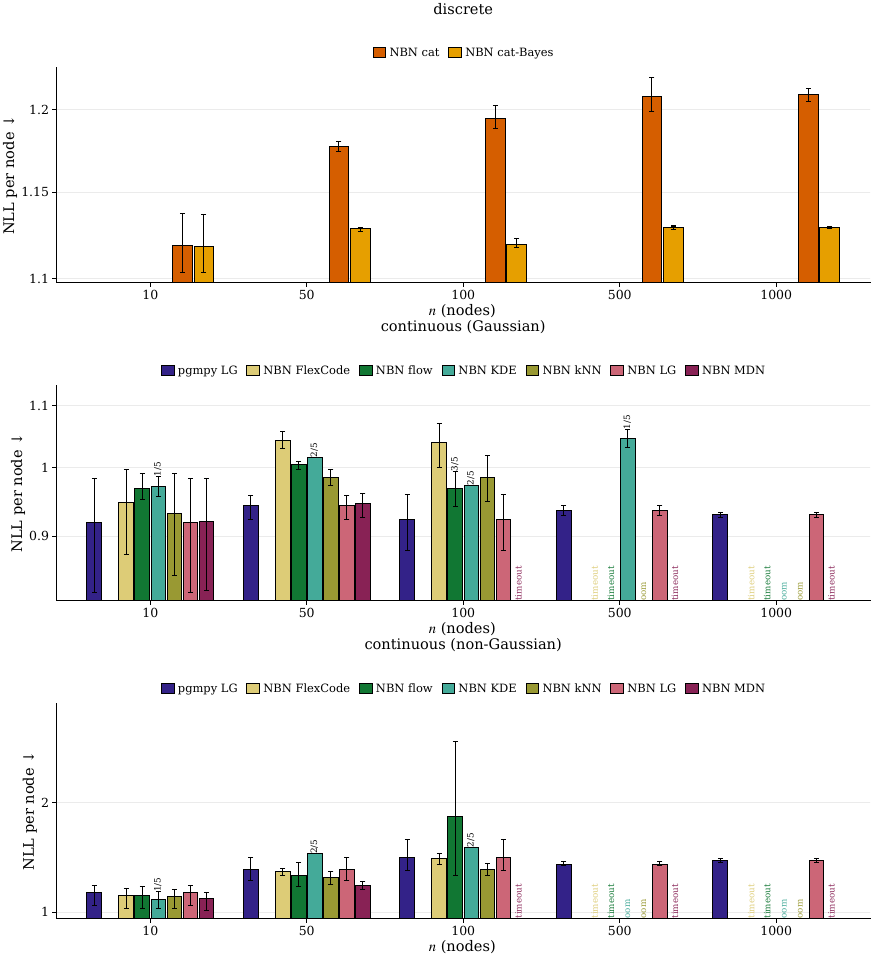}
\captionof{figure}{Parameter learning, all methods: held-out negative log-likelihood per node against network size.}\label{fig:app_param_learning_nll_per_node}
\end{minipage}\par\bigskip

\noindent\begin{minipage}{\textwidth}\centering\footnotesize
\includegraphics[width=\textwidth]{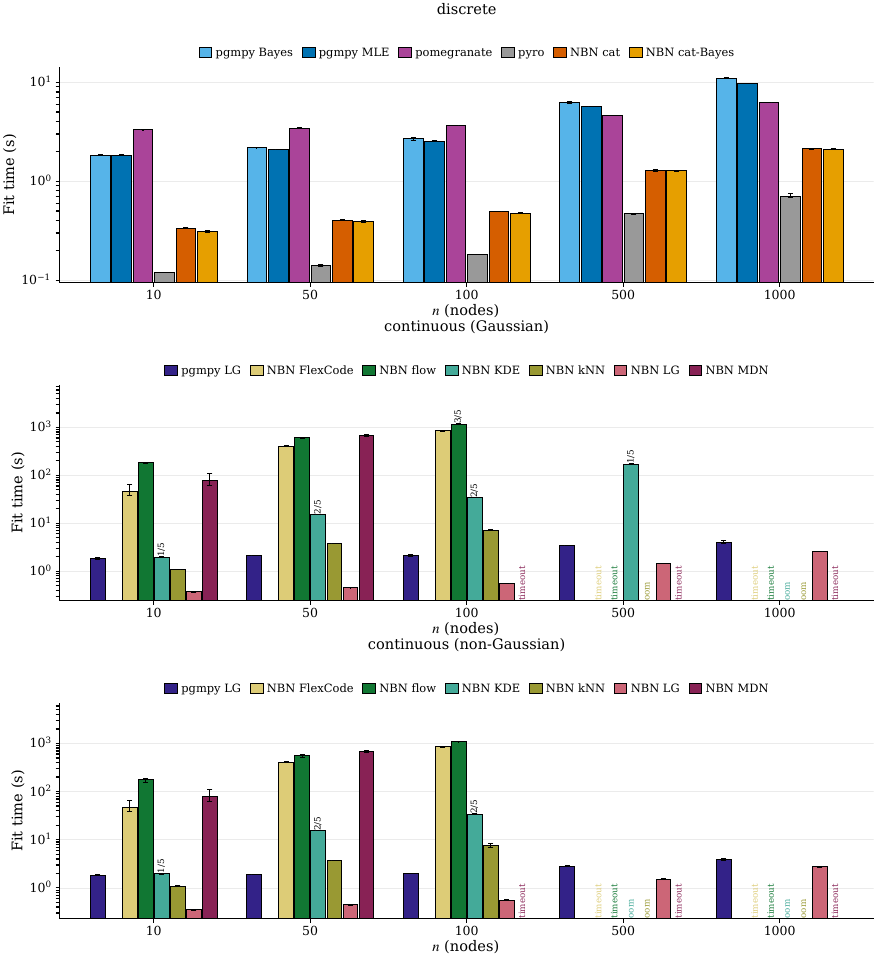}
\captionof{figure}{Parameter learning, all methods: fit time (s) against network size.}\label{fig:app_param_learning_fit_time}
\end{minipage}\par\bigskip

\noindent\begin{minipage}{\textwidth}\centering\footnotesize
\captionof{table}{Parameter learning, discrete networks: parameter-recovery TV (lower is better), against network size.}\label{tab:app_param_learning_discrete_param_recovery_tv}\par\smallskip
\adjustbox{max width=\textwidth}{\begin{tabular}{lrrrrr}
\toprule
Method & $n=10$ & $n=50$ & $n=100$ & $n=500$ & $n=1000$ \\
\midrule
\texttt{pgmpy-bayes} & 0.00907$\pm$0.0023 & 0.0423$\pm$0.002 & 0.0484$\pm$0.0015 & 0.0585$\pm$0.0018 & 0.0592$\pm$0.00063 \\
\texttt{pgmpy-mle} & 0.00908$\pm$0.0023 & 0.0423$\pm$0.002 & 0.0484$\pm$0.0015 & 0.0585$\pm$0.0018 & 0.0592$\pm$0.00063 \\
\texttt{pomegranate-discrete-ve} & 0.00908$\pm$0.0023 & 0.0423$\pm$0.002 & 0.0484$\pm$0.0015 & 0.0585$\pm$0.0018 & 0.0592$\pm$0.00063 \\
\texttt{pyro-empirical} & \textbf{0.00897$\pm$0.0021} & \textbf{0.0408$\pm$0.0017} & \textbf{0.0463$\pm$0.0016} & \textbf{0.0561$\pm$0.002} & \textbf{0.0567$\pm$0.00066} \\
\texttt{nbn-cat} & 0.00908$\pm$0.0023 & 0.0423$\pm$0.002 & 0.0484$\pm$0.0015 & 0.0585$\pm$0.0018 & 0.0592$\pm$0.00063 \\
\texttt{nbn-cat-bayes} & \textbf{0.00897$\pm$0.0021} & \textbf{0.0408$\pm$0.0017} & \textbf{0.0463$\pm$0.0016} & \textbf{0.0561$\pm$0.002} & \textbf{0.0567$\pm$0.00066} \\
\bottomrule
\end{tabular}}
\end{minipage}\par\bigskip

\noindent\begin{minipage}{\textwidth}\centering\footnotesize
\captionof{table}{Parameter learning, discrete networks: held-out negative log-likelihood per node (lower is better), against network size.}\label{tab:app_param_learning_discrete_nll_per_node}\par\smallskip
\adjustbox{max width=\textwidth}{\begin{tabular}{lrrrrr}
\toprule
Method & $n=10$ & $n=50$ & $n=100$ & $n=500$ & $n=1000$ \\
\midrule
\texttt{nbn-cat} & \textbf{1.12$\pm$0.017} & 1.18$\pm$0.003 & 1.19$\pm$0.0069 & 1.21$\pm$0.011 & 1.21$\pm$0.0042 \\
\texttt{nbn-cat-bayes} & \textbf{1.12$\pm$0.017} & \textbf{1.13$\pm$0.0012} & \textbf{1.12$\pm$0.0026} & \textbf{1.13$\pm$0.00093} & \textbf{1.13$\pm$0.00059} \\
\bottomrule
\end{tabular}}
\end{minipage}\par\bigskip

\noindent\begin{minipage}{\textwidth}\centering\footnotesize
\captionof{table}{Parameter learning, discrete networks: fit time in seconds, against network size.}\label{tab:app_param_learning_discrete_fit_time}\par\smallskip
\adjustbox{max width=\textwidth}{\begin{tabular}{lrrrrr}
\toprule
Method & $n=10$ & $n=50$ & $n=100$ & $n=500$ & $n=1000$ \\
\midrule
\texttt{pgmpy-bayes} & 1.83$\pm$0.015 & 2.18$\pm$0.0091 & 2.72$\pm$0.092 & 6.26$\pm$0.18 & 11$\pm$0.16 \\
\texttt{pgmpy-mle} & 1.85$\pm$0.037 & 2.09$\pm$0.0058 & 2.54$\pm$0.016 & 5.73$\pm$0.023 & 9.75$\pm$0.079 \\
\texttt{pomegranate-discrete-ve} & 3.31$\pm$0.0011 & 3.44$\pm$0.025 & 3.66$\pm$0.0033 & 4.63$\pm$0.016 & 6.26$\pm$0.011 \\
\texttt{pyro-empirical} & \textbf{0.12$\pm$0.0011} & \textbf{0.14$\pm$0.0024} & \textbf{0.182$\pm$0.0012} & \textbf{0.469$\pm$0.0074} & \textbf{0.711$\pm$0.032} \\
\texttt{nbn-cat} & 0.336$\pm$0.0024 & 0.401$\pm$0.0046 & 0.495$\pm$0.0039 & 1.3$\pm$0.03 & 2.13$\pm$0.02 \\
\texttt{nbn-cat-bayes} & 0.311$\pm$0.0056 & 0.391$\pm$0.0055 & 0.473$\pm$0.0041 & 1.27$\pm$0.0083 & 2.11$\pm$0.011 \\
\bottomrule
\end{tabular}}
\end{minipage}\par\bigskip

\noindent\begin{minipage}{\textwidth}\centering\footnotesize
\captionof{table}{Parameter learning, linear-Gaussian networks: PIT-KS (lower is better), against network size.}\label{tab:app_param_learning_continuous_lg_calibration_pit_ks}\par\smallskip
\adjustbox{max width=\textwidth}{\begin{tabular}{lrrrrr}
\toprule
Method & $n=10$ & $n=50$ & $n=100$ & $n=500$ & $n=1000$ \\
\midrule
\texttt{pgmpy-lg} & \textbf{0.0299$\pm$0.0056} & \textbf{0.0181$\pm$0.0013} & \textbf{0.0165$\pm$0.00067} & \textbf{0.0151$\pm$2.8e-05} & \textbf{0.0148$\pm$2.9e-05} \\
\texttt{nbn-flexcode} & 0.0358$\pm$0.00089 & 0.0241$\pm$0.00034 & 0.0237$\pm$0.0004 & timeout & timeout \\
\texttt{nbn-flow} & 0.0413$\pm$0.0045 & 0.026$\pm$0.00034 & 0.0249$\pm$0.00022 (3/5) & timeout & timeout \\
\texttt{nbn-kde} & 0.0397$\pm$0.0018 (1/5) & 0.0272$\pm$0.00022 (2/5) & 0.0289$\pm$0.0019 (2/5) & 0.0285$\pm$0.0043 (1/5) & oom \\
\texttt{nbn-knn} & 0.0366$\pm$0.0038 & 0.0346$\pm$0.00049 & 0.0394$\pm$0.0017 & oom & oom \\
\texttt{nbn-lg} & 0.032$\pm$0.0058 & 0.0189$\pm$0.0012 & 0.0171$\pm$0.00056 & 0.0152$\pm$7.5e-05 & 0.0149$\pm$6.4e-05 \\
\texttt{nbn-mdn} & 0.0315$\pm$0.0054 & 0.0203$\pm$0.00048 & timeout & timeout & timeout \\
\bottomrule
\end{tabular}}
\end{minipage}\par\bigskip

\noindent\begin{minipage}{\textwidth}\centering\footnotesize
\captionof{table}{Parameter learning, linear-Gaussian networks: held-out negative log-likelihood per node (lower is better), against network size.}\label{tab:app_param_learning_continuous_lg_nll_per_node}\par\smallskip
\adjustbox{max width=\textwidth}{\begin{tabular}{lrrrrr}
\toprule
Method & $n=10$ & $n=50$ & $n=100$ & $n=500$ & $n=1000$ \\
\midrule
\texttt{pgmpy-lg} & \textbf{0.919$\pm$0.078} & \textbf{0.943$\pm$0.018} & \textbf{0.924$\pm$0.04} & \textbf{0.937$\pm$0.0076} & \textbf{0.931$\pm$0.0041} \\
\texttt{nbn-flexcode} & 0.948$\pm$0.061 & 1.04$\pm$0.014 & 1.04$\pm$0.035 & timeout & timeout \\
\texttt{nbn-flow} & 0.968$\pm$0.02 & 1$\pm$0.006 & 0.968$\pm$0.013 (3/5) & timeout & timeout \\
\texttt{nbn-kde} & 0.971$\pm$0.039 (1/5) & 1.02$\pm$0.018 (2/5) & 0.973$\pm$0.013 (2/5) & 1.05$\pm$0.024 (1/5) & oom \\
\texttt{nbn-knn} & 0.933$\pm$0.072 & 0.985$\pm$0.012 & 0.984$\pm$0.034 & oom & oom \\
\texttt{nbn-lg} & \textbf{0.919$\pm$0.078} & \textbf{0.943$\pm$0.018} & \textbf{0.924$\pm$0.04} & \textbf{0.937$\pm$0.0076} & \textbf{0.931$\pm$0.0041} \\
\texttt{nbn-mdn} & 0.921$\pm$0.078 & 0.947$\pm$0.017 & timeout & timeout & timeout \\
\bottomrule
\end{tabular}}
\end{minipage}\par\bigskip

\noindent\begin{minipage}{\textwidth}\centering\footnotesize
\captionof{table}{Parameter learning, linear-Gaussian networks: fit time in seconds, against network size.}\label{tab:app_param_learning_continuous_lg_fit_time}\par\smallskip
\adjustbox{max width=\textwidth}{\begin{tabular}{lrrrrr}
\toprule
Method & $n=10$ & $n=50$ & $n=100$ & $n=500$ & $n=1000$ \\
\midrule
\texttt{pgmpy-lg} & 1.83$\pm$0.06 & 2.14$\pm$0.025 & 2.16$\pm$0.095 & 3.44$\pm$0.04 & 4.02$\pm$0.25 \\
\texttt{nbn-flexcode} & 47$\pm$14 & 400$\pm$10 & 852$\pm$19 & timeout & timeout \\
\texttt{nbn-flow} & 182$\pm$5.5 & 605$\pm$16 & 1.16e+03$\pm$16 (3/5) & timeout & timeout \\
\texttt{nbn-kde} & 1.96$\pm$0.095 (1/5) & 15.2$\pm$0.23 (2/5) & 34$\pm$0.77 (2/5) & 172$\pm$1.3 (1/5) & oom \\
\texttt{nbn-knn} & 1.08$\pm$0.013 & 3.76$\pm$0.028 & 7.11$\pm$0.16 & oom & oom \\
\texttt{nbn-lg} & \textbf{0.37$\pm$0.0049} & \textbf{0.449$\pm$0.0022} & \textbf{0.547$\pm$0.0073} & \textbf{1.45$\pm$0.0014} & \textbf{2.57$\pm$0.018} \\
\texttt{nbn-mdn} & 77.8$\pm$23 & 674$\pm$27 & timeout & timeout & timeout \\
\bottomrule
\end{tabular}}
\end{minipage}\par\bigskip

\noindent\begin{minipage}{\textwidth}\centering\footnotesize
\captionof{table}{Parameter learning, non-Gaussian networks: PIT-KS (lower is better), against network size.}\label{tab:app_param_learning_continuous_nongauss_calibration_pit_ks}\par\smallskip
\adjustbox{max width=\textwidth}{\begin{tabular}{lrrrrr}
\toprule
Method & $n=10$ & $n=50$ & $n=100$ & $n=500$ & $n=1000$ \\
\midrule
\texttt{pgmpy-lg} & 0.0603$\pm$0.0062 & 0.063$\pm$0.0021 & 0.0646$\pm$0.0068 & \textbf{0.065$\pm$0.00079} & \textbf{0.067$\pm$0.00051} \\
\texttt{nbn-flexcode} & \textbf{0.0314$\pm$0.0033} & 0.027$\pm$0.0054 & 0.0368$\pm$0.013 & timeout & timeout \\
\texttt{nbn-flow} & 0.0392$\pm$0.0056 & 0.0252$\pm$0.00067 & \textbf{0.0225$\pm$0.00058} & timeout & timeout \\
\texttt{nbn-kde} & 0.0313$\pm$0.0024 (1/5) & 0.0241$\pm$0.00052 (2/5) & 0.0262$\pm$0.00022 (2/5) & oom & oom \\
\texttt{nbn-knn} & 0.0376$\pm$0.0024 & 0.0333$\pm$0.00077 & 0.0325$\pm$0.00026 & oom & oom \\
\texttt{nbn-lg} & 0.0644$\pm$0.0047 & 0.0637$\pm$0.0023 & 0.0652$\pm$0.0064 & \textbf{0.065$\pm$0.00068} & \textbf{0.067$\pm$0.00047} \\
\texttt{nbn-mdn} & 0.0356$\pm$0.0088 & \textbf{0.0205$\pm$0.001} & timeout & timeout & timeout \\
\bottomrule
\end{tabular}}
\end{minipage}\par\bigskip

\noindent\begin{minipage}{\textwidth}\centering\footnotesize
\captionof{table}{Parameter learning, non-Gaussian networks: held-out negative log-likelihood per node (lower is better), against network size.}\label{tab:app_param_learning_continuous_nongauss_nll_per_node}\par\smallskip
\adjustbox{max width=\textwidth}{\begin{tabular}{lrrrrr}
\toprule
Method & $n=10$ & $n=50$ & $n=100$ & $n=500$ & $n=1000$ \\
\midrule
\texttt{pgmpy-lg} & 1.13$\pm$0.071 & 1.31$\pm$0.095 & 1.42$\pm$0.14 & \textbf{1.36$\pm$0.015} & \textbf{1.39$\pm$0.016} \\
\texttt{nbn-flexcode} & 1.11$\pm$0.068 & 1.3$\pm$0.027 & 1.41$\pm$0.048 & timeout & timeout \\
\texttt{nbn-flow} & 1.11$\pm$0.076 & 1.26$\pm$0.096 & 1.84$\pm$0.84 & timeout & timeout \\
\texttt{nbn-kde} & 1.08$\pm$0.087 (1/5) & 1.45$\pm$3.4 (2/5) & 1.51$\pm$0.032 (2/5) & oom & oom \\
\texttt{nbn-knn} & 1.11$\pm$0.067 & 1.25$\pm$0.053 & \textbf{1.31$\pm$0.048} & oom & oom \\
\texttt{nbn-lg} & 1.13$\pm$0.071 & 1.31$\pm$0.095 & 1.42$\pm$0.14 & \textbf{1.36$\pm$0.015} & \textbf{1.39$\pm$0.016} \\
\texttt{nbn-mdn} & \textbf{1.09$\pm$0.06} & \textbf{1.19$\pm$0.031} & timeout & timeout & timeout \\
\bottomrule
\end{tabular}}
\end{minipage}\par\bigskip

\noindent\begin{minipage}{\textwidth}\centering\footnotesize
\captionof{table}{Parameter learning, non-Gaussian networks: fit time in seconds, against network size.}\label{tab:app_param_learning_continuous_nongauss_fit_time}\par\smallskip
\adjustbox{max width=\textwidth}{\begin{tabular}{lrrrrr}
\toprule
Method & $n=10$ & $n=50$ & $n=100$ & $n=500$ & $n=1000$ \\
\midrule
\texttt{pgmpy-lg} & 1.84$\pm$0.033 & 1.92$\pm$0.02 & 2.01$\pm$0.016 & 2.8$\pm$0.059 & 3.93$\pm$0.15 \\
\texttt{nbn-flexcode} & 47.3$\pm$14 & 404$\pm$14 & 858$\pm$18 & timeout & timeout \\
\texttt{nbn-flow} & 178$\pm$16 & 553$\pm$41 & 1.08e+03$\pm$10 & timeout & timeout \\
\texttt{nbn-kde} & 1.97$\pm$0.1 (1/5) & 15.4$\pm$0.23 (2/5) & 34.3$\pm$2 (2/5) & oom & oom \\
\texttt{nbn-knn} & 1.1$\pm$0.0083 & 3.75$\pm$0.0084 & 7.57$\pm$0.77 & oom & oom \\
\texttt{nbn-lg} & \textbf{0.356$\pm$0.0075} & \textbf{0.449$\pm$0.013} & \textbf{0.553$\pm$0.01} & \textbf{1.51$\pm$0.031} & \textbf{2.76$\pm$0.028} \\
\texttt{nbn-mdn} & 79.4$\pm$24 & 679$\pm$24 & timeout & timeout & timeout \\
\bottomrule
\end{tabular}}
\end{minipage}\par\bigskip

\clearpage

\subsection{Sample efficiency}\label{app:learning_curves}

\noindent\begin{minipage}{\textwidth}\centering\footnotesize
\includegraphics[width=\textwidth]{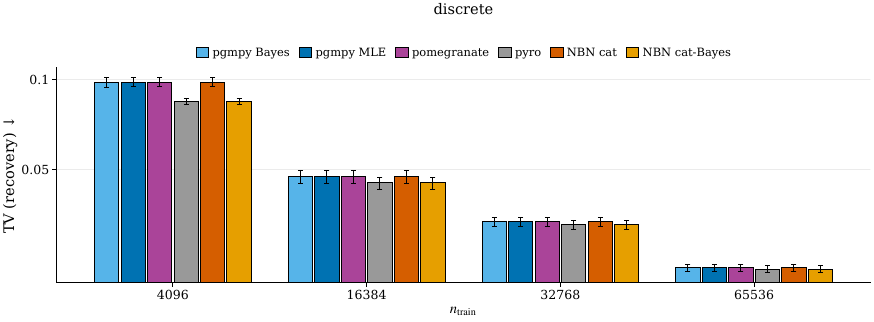}
\captionof{figure}{Sample efficiency, all methods: parameter-recovery TV against training-set size.}\label{fig:app_learning_curves_param_recovery_tv}
\end{minipage}\par\bigskip

\noindent\begin{minipage}{\textwidth}\centering\footnotesize
\includegraphics[width=\textwidth]{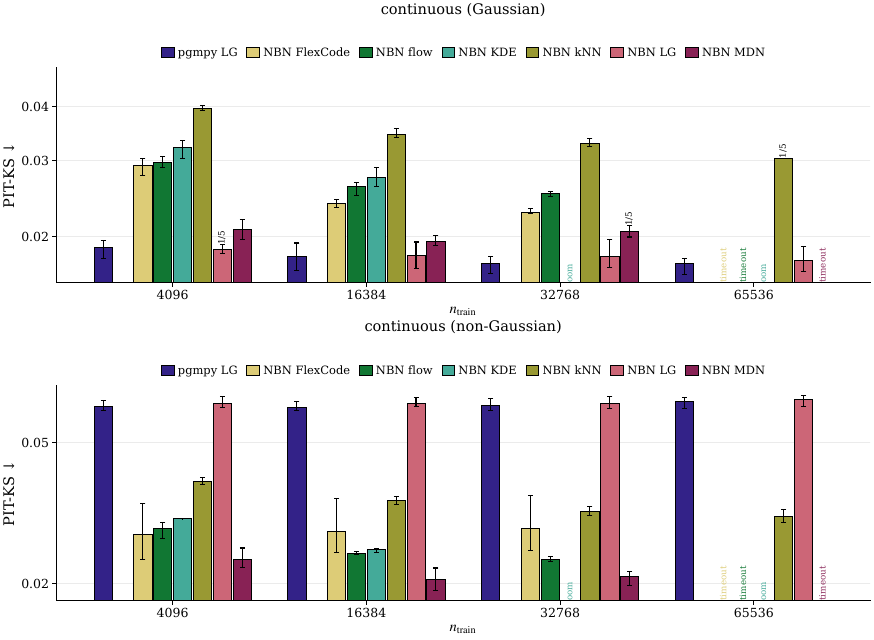}
\captionof{figure}{Sample efficiency, all methods: calibration error (PIT-KS) against training-set size.}\label{fig:app_learning_curves_calibration_pit_ks}
\end{minipage}\par\bigskip

\noindent\begin{minipage}{\textwidth}\centering\footnotesize
\includegraphics[width=\textwidth]{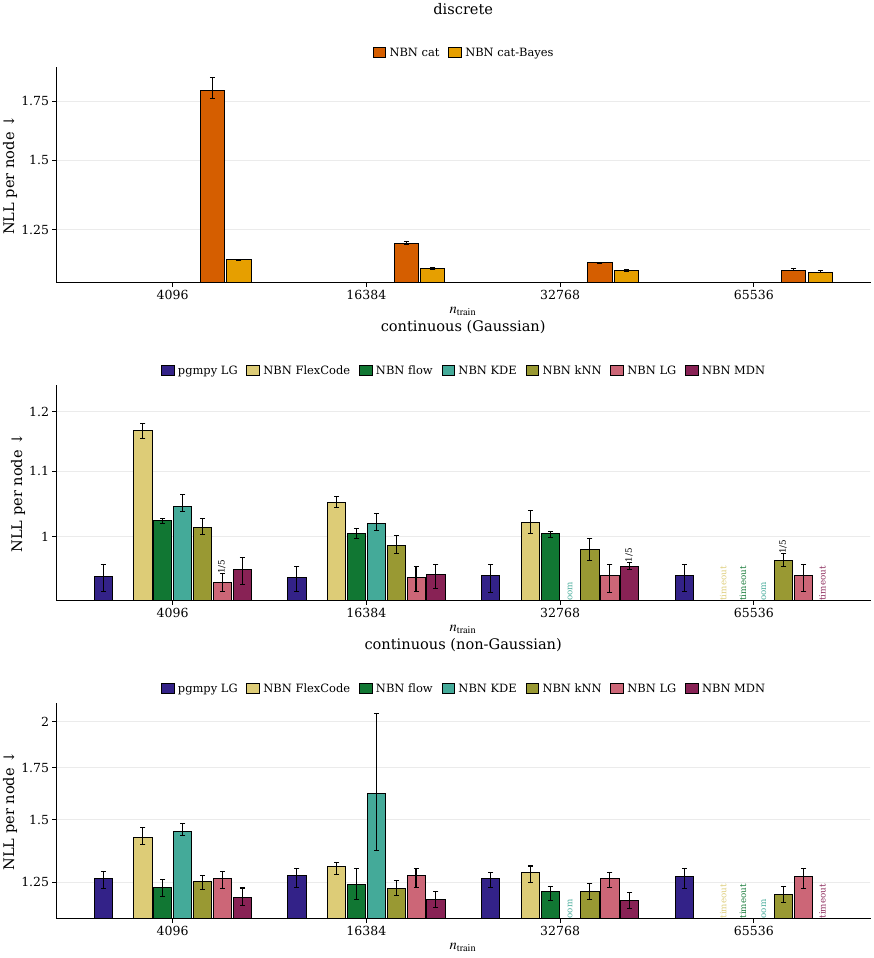}
\captionof{figure}{Sample efficiency, all methods: held-out negative log-likelihood per node against training-set size.}\label{fig:app_learning_curves_nll_per_node}
\end{minipage}\par\bigskip

\noindent\begin{minipage}{\textwidth}\centering\footnotesize
\includegraphics[width=\textwidth]{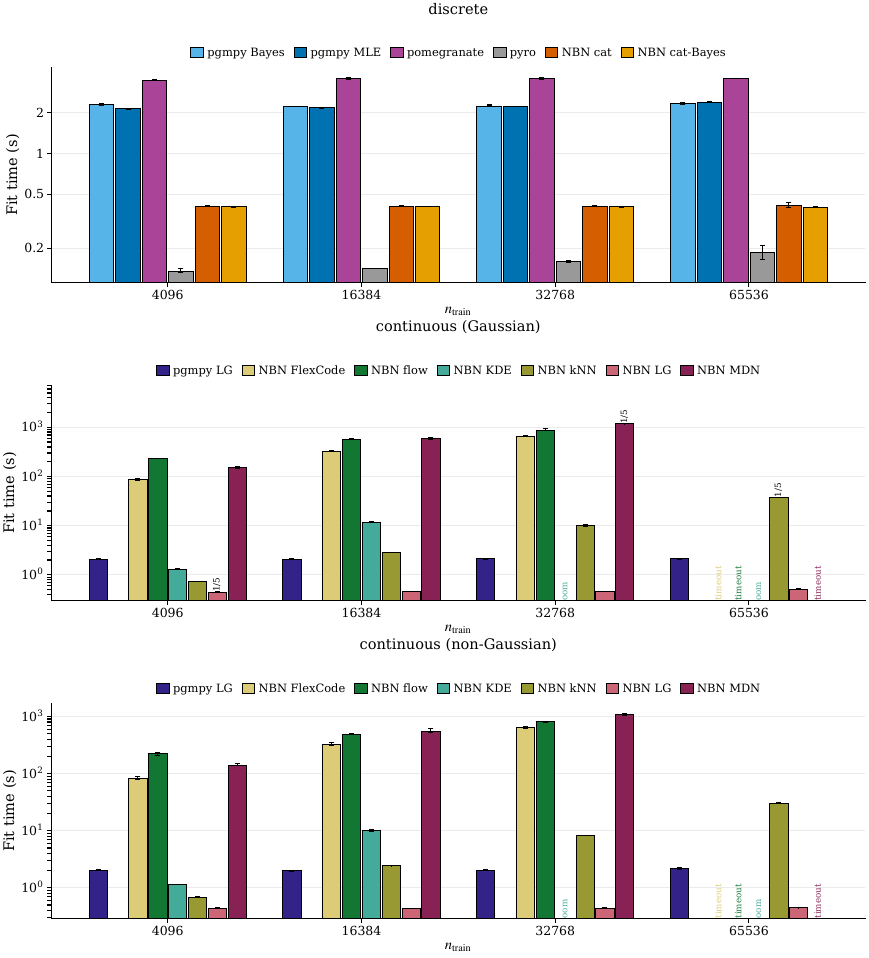}
\captionof{figure}{Sample efficiency, all methods: fit time (s) against training-set size.}\label{fig:app_learning_curves_fit_time}
\end{minipage}\par\bigskip

\noindent\begin{minipage}{\textwidth}\centering\footnotesize
\captionof{table}{Sample efficiency, discrete networks: parameter-recovery TV (lower is better), against training-set size.}\label{tab:app_learning_curves_discrete_param_recovery_tv}\par\smallskip
\adjustbox{max width=\textwidth}{\begin{tabular}{lrrrr}
\toprule
Method & $n_{tr}=4096$ & $n_{tr}=16384$ & $n_{tr}=32768$ & $n_{tr}=65536$ \\
\midrule
\texttt{pgmpy-bayes} & 0.0977$\pm$0.0035 & 0.0474$\pm$0.0023 & 0.0335$\pm$0.0011 & 0.0236$\pm$0.00064 \\
\texttt{pgmpy-mle} & 0.0979$\pm$0.0035 & 0.0474$\pm$0.0023 & 0.0336$\pm$0.0011 & 0.0236$\pm$0.00064 \\
\texttt{pomegranate-discrete-ve} & 0.0979$\pm$0.0035 & 0.0474$\pm$0.0023 & 0.0336$\pm$0.0011 & 0.0236$\pm$0.00064 \\
\texttt{pyro-empirical} & \textbf{0.0847$\pm$0.0019} & \textbf{0.0453$\pm$0.0021} & \textbf{0.0328$\pm$0.001} & \textbf{0.0233$\pm$0.00062} \\
\texttt{nbn-cat} & 0.0979$\pm$0.0035 & 0.0474$\pm$0.0023 & 0.0336$\pm$0.0011 & 0.0236$\pm$0.00064 \\
\texttt{nbn-cat-bayes} & \textbf{0.0847$\pm$0.0019} & \textbf{0.0453$\pm$0.0021} & \textbf{0.0328$\pm$0.001} & \textbf{0.0233$\pm$0.00062} \\
\bottomrule
\end{tabular}}
\end{minipage}\par\bigskip

\noindent\begin{minipage}{\textwidth}\centering\footnotesize
\captionof{table}{Sample efficiency, discrete networks: held-out negative log-likelihood per node (lower is better), against training-set size.}\label{tab:app_learning_curves_discrete_nll_per_node}\par\smallskip
\adjustbox{max width=\textwidth}{\begin{tabular}{lrrrr}
\toprule
Method & $n_{tr}=4096$ & $n_{tr}=16384$ & $n_{tr}=32768$ & $n_{tr}=65536$ \\
\midrule
\texttt{nbn-cat} & 1.8$\pm$0.049 & 1.21$\pm$0.0047 & 1.15$\pm$0.0015 & \textbf{1.12$\pm$0.0029} \\
\texttt{nbn-cat-bayes} & \textbf{1.16$\pm$0.0017} & \textbf{1.13$\pm$0.0031} & \textbf{1.12$\pm$0.0035} & \textbf{1.12$\pm$0.0031} \\
\bottomrule
\end{tabular}}
\end{minipage}\par\bigskip

\noindent\begin{minipage}{\textwidth}\centering\footnotesize
\captionof{table}{Sample efficiency, discrete networks: fit time in seconds, against training-set size.}\label{tab:app_learning_curves_discrete_fit_time}\par\smallskip
\adjustbox{max width=\textwidth}{\begin{tabular}{lrrrr}
\toprule
Method & $n_{tr}=4096$ & $n_{tr}=16384$ & $n_{tr}=32768$ & $n_{tr}=65536$ \\
\midrule
\texttt{pgmpy-bayes} & 2.29$\pm$0.025 & 2.22$\pm$0.0038 & 2.25$\pm$0.035 & 2.34$\pm$0.023 \\
\texttt{pgmpy-mle} & 2.14$\pm$0.0092 & 2.18$\pm$0.02 & 2.23$\pm$0.013 & 2.39$\pm$0.017 \\
\texttt{pomegranate-discrete-ve} & 3.5$\pm$0.0068 & 3.56$\pm$0.049 & 3.57$\pm$0.063 & 3.59$\pm$0.021 \\
\texttt{pyro-empirical} & \textbf{0.135$\pm$0.0045} & \textbf{0.142$\pm$0.0004} & \textbf{0.16$\pm$0.0036} & \textbf{0.187$\pm$0.022} \\
\texttt{nbn-cat} & 0.407$\pm$0.0024 & 0.41$\pm$0.00093 & 0.408$\pm$0.003 & 0.413$\pm$0.019 \\
\texttt{nbn-cat-bayes} & 0.404$\pm$0.0046 & 0.406$\pm$0.0029 & 0.404$\pm$0.002 & 0.401$\pm$0.0052 \\
\bottomrule
\end{tabular}}
\end{minipage}\par\bigskip

\noindent\begin{minipage}{\textwidth}\centering\footnotesize
\captionof{table}{Sample efficiency, linear-Gaussian networks: PIT-KS (lower is better), against training-set size.}\label{tab:app_learning_curves_continuous_lg_calibration_pit_ks}\par\smallskip
\adjustbox{max width=\textwidth}{\begin{tabular}{lrrrr}
\toprule
Method & $n_{tr}=4096$ & $n_{tr}=16384$ & $n_{tr}=32768$ & $n_{tr}=65536$ \\
\midrule
\texttt{pgmpy-lg} & \textbf{0.0189$\pm$0.00091} & \textbf{0.018$\pm$0.0013} & \textbf{0.0173$\pm$0.0008} & \textbf{0.0173$\pm$0.00073} \\
\texttt{nbn-flexcode} & 0.0292$\pm$0.0013 & 0.0238$\pm$0.00052 & 0.0228$\pm$0.00029 & timeout \\
\texttt{nbn-flow} & 0.0297$\pm$0.0009 & 0.0261$\pm$0.00093 & 0.0252$\pm$0.00034 & timeout \\
\texttt{nbn-kde} & 0.0323$\pm$0.0015 & 0.0275$\pm$0.0014 & oom & oom \\
\texttt{nbn-knn} & 0.0397$\pm$0.00051 & 0.0346$\pm$0.00085 & 0.0329$\pm$0.00076 & 0.0304$\pm$0.00049 (1/5) \\
\texttt{nbn-lg} & 0.0187$\pm$0.00084 (1/5) & 0.0181$\pm$0.0013 & 0.018$\pm$0.0014 & 0.0176$\pm$0.0012 \\
\texttt{nbn-mdn} & 0.0208$\pm$0.0011 & 0.0195$\pm$0.00055 & 0.0206$\pm$0.00096 (1/5) & timeout \\
\bottomrule
\end{tabular}}
\end{minipage}\par\bigskip

\noindent\begin{minipage}{\textwidth}\centering\footnotesize
\captionof{table}{Sample efficiency, linear-Gaussian networks: held-out negative log-likelihood per node (lower is better), against training-set size.}\label{tab:app_learning_curves_continuous_lg_nll_per_node}\par\smallskip
\adjustbox{max width=\textwidth}{\begin{tabular}{lrrrr}
\toprule
Method & $n_{tr}=4096$ & $n_{tr}=16384$ & $n_{tr}=32768$ & $n_{tr}=65536$ \\
\midrule
\texttt{pgmpy-lg} & \textbf{0.943$\pm$0.019} & \textbf{0.942$\pm$0.017} & \textbf{0.944$\pm$0.019} & \textbf{0.945$\pm$0.019} \\
\texttt{nbn-flexcode} & 1.17$\pm$0.013 & 1.05$\pm$0.0087 & 1.02$\pm$0.017 & timeout \\
\texttt{nbn-flow} & 1.02$\pm$0.0043 & 1.01$\pm$0.0073 & 1$\pm$0.0043 & timeout \\
\texttt{nbn-kde} & 1.05$\pm$0.013 & 1.02$\pm$0.013 & oom & oom \\
\texttt{nbn-knn} & 1.01$\pm$0.012 & 0.988$\pm$0.013 & 0.981$\pm$0.015 & 0.966$\pm$0.013 (1/5) \\
\texttt{nbn-lg} & 0.935$\pm$0.015 (1/5) & \textbf{0.942$\pm$0.017} & \textbf{0.944$\pm$0.019} & \textbf{0.945$\pm$0.019} \\
\texttt{nbn-mdn} & 0.953$\pm$0.019 & 0.946$\pm$0.017 & 0.958$\pm$0.011 (1/5) & timeout \\
\bottomrule
\end{tabular}}
\end{minipage}\par\bigskip

\noindent\begin{minipage}{\textwidth}\centering\footnotesize
\captionof{table}{Sample efficiency, linear-Gaussian networks: fit time in seconds, against training-set size.}\label{tab:app_learning_curves_continuous_lg_fit_time}\par\smallskip
\adjustbox{max width=\textwidth}{\begin{tabular}{lrrrr}
\toprule
Method & $n_{tr}=4096$ & $n_{tr}=16384$ & $n_{tr}=32768$ & $n_{tr}=65536$ \\
\midrule
\texttt{pgmpy-lg} & 2.06$\pm$0.049 & 2.06$\pm$0.041 & 2.14$\pm$0.043 & 2.12$\pm$0.057 \\
\texttt{nbn-flexcode} & 86$\pm$4.3 & 331$\pm$9.1 & 664$\pm$25 & timeout \\
\texttt{nbn-flow} & 232$\pm$3.9 & 573$\pm$5.5 & 886$\pm$42 & timeout \\
\texttt{nbn-kde} & 1.29$\pm$0.023 & 11.9$\pm$0.049 & oom & oom \\
\texttt{nbn-knn} & \textbf{0.737$\pm$0.0044} & 2.86$\pm$0.034 & 9.99$\pm$0.39 & 37.7$\pm$0.53 (1/5) \\
\texttt{nbn-lg} & 0.451$\pm$0.0072 (1/5) & \textbf{0.462$\pm$0.0039} & \textbf{0.464$\pm$0.00023} & \textbf{0.509$\pm$0.015} \\
\texttt{nbn-mdn} & 153$\pm$7.3 & 589$\pm$20 & 1.18e+03$\pm$12 (1/5) & timeout \\
\bottomrule
\end{tabular}}
\end{minipage}\par\bigskip

\noindent\begin{minipage}{\textwidth}\centering\footnotesize
\captionof{table}{Sample efficiency, non-Gaussian networks: PIT-KS (lower is better), against training-set size.}\label{tab:app_learning_curves_continuous_nongauss_calibration_pit_ks}\par\smallskip
\adjustbox{max width=\textwidth}{\begin{tabular}{lrrrr}
\toprule
Method & $n_{tr}=4096$ & $n_{tr}=16384$ & $n_{tr}=32768$ & $n_{tr}=65536$ \\
\midrule
\texttt{pgmpy-lg} & 0.0632$\pm$0.0021 & 0.063$\pm$0.002 & 0.0635$\pm$0.0024 & 0.0653$\pm$0.0022 \\
\texttt{nbn-flexcode} & 0.0275$\pm$0.0051 & 0.028$\pm$0.0052 & 0.0287$\pm$0.0054 & timeout \\
\texttt{nbn-flow} & 0.0286$\pm$0.0015 & 0.0244$\pm$0.00023 & 0.0234$\pm$0.00035 & timeout \\
\texttt{nbn-kde} & 0.0305$\pm$2.7e-05 & 0.025$\pm$0.00032 & oom & oom \\
\texttt{nbn-knn} & 0.0388$\pm$0.00089 & 0.0345$\pm$0.0009 & 0.032$\pm$0.00092 & \textbf{0.0311$\pm$0.0013} \\
\texttt{nbn-lg} & 0.0646$\pm$0.0022 & 0.0645$\pm$0.002 & 0.0647$\pm$0.0025 & 0.0663$\pm$0.0024 \\
\texttt{nbn-mdn} & \textbf{0.0234$\pm$0.0015} & \textbf{0.0206$\pm$0.0015} & \textbf{0.021$\pm$0.00093} & timeout \\
\bottomrule
\end{tabular}}
\end{minipage}\par\bigskip

\noindent\begin{minipage}{\textwidth}\centering\footnotesize
\captionof{table}{Sample efficiency, non-Gaussian networks: held-out negative log-likelihood per node (lower is better), against training-set size.}\label{tab:app_learning_curves_continuous_nongauss_nll_per_node}\par\smallskip
\adjustbox{max width=\textwidth}{\begin{tabular}{lrrrr}
\toprule
Method & $n_{tr}=4096$ & $n_{tr}=16384$ & $n_{tr}=32768$ & $n_{tr}=65536$ \\
\midrule
\texttt{pgmpy-lg} & 1.26$\pm$0.03 & 1.27$\pm$0.035 & 1.27$\pm$0.029 & 1.27$\pm$0.037 \\
\texttt{nbn-flexcode} & 1.42$\pm$0.037 & 1.31$\pm$0.022 & 1.29$\pm$0.031 & timeout \\
\texttt{nbn-flow} & 1.23$\pm$0.03 & 1.24$\pm$0.055 & 1.22$\pm$0.026 & timeout \\
\texttt{nbn-kde} & 1.45$\pm$0.025 & 1.62$\pm$0.34 & oom & oom \\
\texttt{nbn-knn} & 1.25$\pm$0.026 & 1.23$\pm$0.026 & 1.22$\pm$0.027 & \textbf{1.21$\pm$0.028} \\
\texttt{nbn-lg} & 1.26$\pm$0.03 & 1.27$\pm$0.035 & 1.27$\pm$0.029 & 1.27$\pm$0.037 \\
\texttt{nbn-mdn} & \textbf{1.2$\pm$0.03} & \textbf{1.19$\pm$0.028} & \textbf{1.19$\pm$0.029} & timeout \\
\bottomrule
\end{tabular}}
\end{minipage}\par\bigskip

\noindent\begin{minipage}{\textwidth}\centering\footnotesize
\captionof{table}{Sample efficiency, non-Gaussian networks: fit time in seconds, against training-set size.}\label{tab:app_learning_curves_continuous_nongauss_fit_time}\par\smallskip
\adjustbox{max width=\textwidth}{\begin{tabular}{lrrrr}
\toprule
Method & $n_{tr}=4096$ & $n_{tr}=16384$ & $n_{tr}=32768$ & $n_{tr}=65536$ \\
\midrule
\texttt{pgmpy-lg} & 2.04$\pm$0.044 & 1.99$\pm$0.018 & 2.04$\pm$0.016 & 2.2$\pm$0.096 \\
\texttt{nbn-flexcode} & 83.4$\pm$4.5 & 324$\pm$17 & 633$\pm$19 & timeout \\
\texttt{nbn-flow} & 222$\pm$13 & 488$\pm$4.8 & 816$\pm$20 & timeout \\
\texttt{nbn-kde} & 1.15$\pm$0.02 & 10.2$\pm$0.45 & oom & oom \\
\texttt{nbn-knn} & 0.681$\pm$0.012 & 2.46$\pm$0.0066 & 8.31$\pm$0.054 & 30.2$\pm$0.49 \\
\texttt{nbn-lg} & \textbf{0.442$\pm$0.0072} & \textbf{0.44$\pm$0.0015} & \textbf{0.44$\pm$0.0089} & \textbf{0.452$\pm$0.0021} \\
\texttt{nbn-mdn} & 142$\pm$8.3 & 552$\pm$42 & 1.09e+03$\pm$44 & timeout \\
\bottomrule
\end{tabular}}
\end{minipage}\par\bigskip

\clearpage